\documentclass[preprint,12pt]{elsarticle}
\usepackage{graphicx} 
\usepackage[utf8]{inputenc}
\usepackage{graphicx}
\usepackage{diagbox} 
\usepackage{hyperref}
\usepackage{amsmath}
\usepackage{paralist}
\usepackage{amsfonts}
\usepackage{amssymb} 
\usepackage{amsthm}
\usepackage{algorithm}
\usepackage{algorithmic}
\usepackage{harpoon}
\usepackage{float}
\usepackage{booktabs}
\usepackage{multirow}
\usepackage[final]{changes}
\usepackage{comment}
\usepackage{mathtools,nccmath}
\usepackage{bm}
\newtheorem{theorem}{Theorem}[section]
\newtheorem{example}{Example}[section]
\newtheorem{lemma}{Lemma}[section]
\newtheorem{corollary}{Corollary}[section]

\newtheorem{definition}{Definition}[section]

\def\d{\text{d}}
\def\E{\mathbb{E}}
\begin{document}
\begin{frontmatter}
\title{A new local time-decoupled squared Wasserstein-2 method for training stochastic neural networks to reconstruct uncertain parameters in dynamical systems}

\author[label1,label2]{Mingtao Xia\footnote{corresponding author, email: xiamingtao@nyu.edu}}
\affiliation[label1]{organization={Courant Institute of Mathematical Sciences, New York University},
            addressline={251 Mercer Street}, 
            city={New York},
            postcode={NY 10012}, 
            state={New York},
            country={USA}}
        \affiliation[label2]{organization={Department of Mathematics, University of Houston},
            addressline={Philip Guthrie Hoffman Hall, 3551 Cullen Blvd},
            city={Houston},
            postcode={TX 77004},
            state={Texas},
            country={USA}}
              \author[label3]{Qijing Shen}
\author[label4]{Philip Maini}
\author[label4]{Eamonn Gaffney}
\author[label1]{Alex Mogilner}
\affiliation[label3]{organization={Nuffield Department of Medicine, University of Oxford},
            addressline={Roosevelt Drive}, 
            city={Oxford},
            postcode={OX3 7BN}, 
            country={United Kingdom}}
    \affiliation[label4]{organization={Mathematical Institute, University of Oxford},
            addressline={Radcliffe Observatory, Andrew Wiles Building}, 
            city={Oxford},
            postcode={OX2 6GG}, 
            country={United Kingdom}}

\begin{abstract}
    In this work, we propose and analyze a new local time-decoupled squared Wasserstein-2 method for reconstructing the distribution of unknown parameters in dynamical systems. Specifically, we show that 
    a stochastic neural network model, which can be effectively trained by minimizing our proposed local time-decoupled squared Wasserstein-2 loss function, is an effective model for approximating the distribution of uncertain model parameters in dynamical systems. Through several numerical examples, we showcase the effectiveness of our proposed method in reconstructing the distribution of parameters in different dynamical systems.
\end{abstract}

\begin{highlights}
\item We proposed and analyzed a new local time-decoupled squared Wasserstein-2 method for direct reconstruction of the \textit{distribution} of model parameters in several dynamical systems from time-series data. Our method is easy to implement and outperforms several statistical benchmark methods.
\item We analyzed a stochastic neural network model whose weights are sampled from independent normal distributions. Under nonrestrictive conditions, we proved that this stochastic neural network can approximate any continuous random variable under the Wasserstein-2 distance metric, making it an effective model for reconstructing the distribution of model parameters.
\end{highlights}


\end{frontmatter}

\section{Introduction}

\numberwithin{equation}{section}
The inverse problem of reconstructing a noisy dynamical system from time-series data finds wide applications across different disciplines. Efficient algorithms to solve such inverse-type problems advance different fields including inferring neural circuit dynamics from spiking data \cite{pillow2008spatio} in neuroscience, modeling and predicting complex weather patterns from historical data \cite{Carrassi2018} in climate science, uncovering disease transmission dynamics from infection case counts over time \cite{roda2020difficult} in epidemiology, and deducing reaction rates from experimental concentration-time profiles in reaction kinetics in biochemistry \cite{loskot2019comprehensive}. However, such inverse-type problems pose substantial mathematical and computational challenges, particularly when data are limited and noisy, motivating ongoing research into novel algorithms and theoretical frameworks to improve models' reconstruction accuracy and efficiency.

In this paper, we study the inverse problem of inferring the distribution of model parameters for several dynamical systems including ordinary differential equations (ODEs), partial differential equations (PDEs), and stochastic differential equations (SDEs) from time-series data or spatiotemporal data. Existing methods for such problems can be broadly categorized into traditional statistical approaches and modern data-driven techniques. Traditional statistical methods often involve parameter estimation frameworks. For example, linear and nonlinear regression methods play a role in simpler systems where the functional form of the model is partially known \cite{Chib2002}. Furthermore, maximum likelihood estimation and Bayesian inference methods \cite{Creswell2015, Marin2021} are often adopted. Maximum likelihood estimation optimizes the likelihood of model parameter values in a proposed model from observed data, while Bayesian methods incorporate prior information and compute posterior distributions. These approaches are widely used in applications such as reaction network reconstruction and epidemiological modeling.  On the other hand, data-driven methods, leveraging advances in machine learning, offer a complementary toolkit for inverse problems. For example, neural networks and reservoir computing frameworks have been successful in reconstructing chaotic systems and inferring governing equations directly from data \cite{Pathak2018, Brunton2016}. Sparse identification of nonlinear dynamics (SINDy) has emerged as a powerful tool for discovering interpretable dynamical systems by identifying a parsimonious set of governing equations from time-series data \cite{Brunton2016}. Gaussian process regression has proven effective for nonparametric inference, especially in uncertainty quantification \cite{Raissi2017}. Hybrid approaches, which integrate data-driven and traditional statistical methods such as the physics-informed neural network method \cite{raissi2019physics}, are also gathering increasing attention as they take advantage of both statistical methods and data-driven methods and provide more interpretable machine-learning tools.

The Wasserstein distance, which can effectively measure the discrepancy between two probability distributions \cite{villani2009optimal, panaretos2019statistical, balasubramanian2024gaussian}, was utilized for various uncertainty quantification (UQ) tasks. In \cite{bernton2019parameter}, the Wasserstein distance was proposed to estimate model parameters that govern a probabilistic model. Compared to empirical maximum likelihood estimates, using the Wasserstein distance to estimate model parameters in uncertainty models can be more efficient and accurate \cite{blanchet2021sample}.
However, previous methods mainly focused on point estimates of model parameters (inferring the exact values of the parameters) and could not quantify the uncertainty in the model parameters, \textit{i.e.}, they cannot reconstruct a distribution of the model parameters from data. 
Recently, a time-decoupled Wasserstein-2 ($W_2$) distance method, which compares the distributions of two stochastic processes across different time points, has been revealed to be an efficient loss function for reconstructing intrinsically noisy stochastic processes such as pure-diffusion processes and jump-diffusion processes using parameterized neural networks \cite{xia2024squared, xia2024efficient}. Furthermore, in \cite{xia2024local}, a local squared $W_2$ method, which adopts a ``neighborhood technique" to enlarge the number of available data, was proposed to infer the distribution of $y$ given observed data $x$ in the uncertainty model $y=f(x, \omega)$ when $\omega$ are uncertain latent unobserved variables. However, for the specific problem of inferring unknown parameters in deterministic or stochastic dynamical models, these two methods are not suitable: \added{the time-decoupled $W_2$ distance method assumes that there is no uncertainty in the underlying model (\textit{i.e.}, the ODE or SDE to be reconstructed is fixed), and the local squared $W_2$ method is not directly applicable to parameter inference problems.} 
Two major difficulties arise: first, an efficient model is required to approximate the distribution of parameters in a dynamical system; second, it is necessary to distinguish uncertainty in model parameters from uncertainty in the initial state of the dynamical systems and intrinsic stochasticity in the dynamical system such as the Wiener process in SDEs.

In this work, we propose and analyze a local time-decoupled squared $W_2$ method, which builds upon the time-decoupled squared $W_2$ method for inferring the distribution of model parameters in deterministic or stochastic dynamical systems and the local squared $W_2$ method for handling the uncertainty in the initial state. 
\added{As an illustration, consider the following ODE:
\begin{equation}
    \d\bm{X}(t;\theta) = \bm{f}(\bm{X}(t;\theta), t;\theta)\text{d}t,\,\,t\in[0, T],\,\, \bm{X}(0)\sim\nu_0,\,\,\theta\sim\mu,
    \label{ODE_model}
\end{equation}
where $\theta\in\mathbb{R}^{\ell}$ is a continuous random variable representing \added{uncertain parameters} in the ODE and $\bm{f}:\mathbb{R}^{d+1+\ell}\rightarrow\mathbb{R}^d$. $\nu_0$ is a probability measure defined on the Borel $\sigma$-algebra $\mathcal{B}(\mathbb{R}^d)$. For each realization of the ODE~\eqref{ODE_model}, $\theta$ is sampled independently. Thus, we obtain different trajectories by solving the ODE~\eqref{ODE_model} multiple times due to parameter heterogeneity. We use another ODE model as an approximation to Eq.~\eqref{ODE_model}:
\begin{equation}
    \d\hat{\bm{X}}(t;\hat{\theta}) = \bm{f}(\hat{\bm{X}}(t;\hat{\theta}), t;\hat{\theta})\text{d}t,\,\,t\in[0, T], \,\,\hat{\bm{X}}(0)=\bm{X}_0, \,\, \hat{\theta}\sim\hat{\mu}.
    \label{approximate_ODE}
\end{equation}
Here, $\hat{\theta}\in\mathbb{R}^{\ell}$ is another continuous random variable and denotes uncertain parameters in the approximate ODE.  We aim to construct the probability density function $\hat{\mu}$, \textit{i.e.} the distribution of $\hat{\theta}\in\mathbb{R}^{\ell}$, such that $\hat{\mu}$ can match $\mu$ well in Eq.~\eqref{ODE_model}.}

Specifically, we show that minimizing our proposed loss function can lead to efficient training of a stochastic neural network (SNN) model with weight uncertainty. SNNs are effective in uncertainty quantification and generative modeling \cite{Gal2016, Rezende2014}.
Instead of deterministic outputs, SNNs introduce randomness in the weights and/or biases \cite{blundell2015weight, yu2021simple} or utilize stochastic neurons whose outputs are binary \cite{tang2013learning}. Theoretical results exist for the approximation errors for certain types of SNNs such as the approximation error of a single-layer SNN \cite{gonon2023approximation} as well as the universal approximation ability of dropout neural networks \cite{manita2022universal}. In this work, we prove that the SNN model we use can serve as an effective model to approximate any continuous random variables under moderate assumptions. Compared with traditional Bayesian methods, our proposed approach has the advantage of not requiring any knowledge or prior distributions of model parameters and directly reconstructing model parameter distribution from time-series data. Compared to other data-driven methods, such as Bayesian neural networks \cite{neal2012bayesian}, generative modeling methods \cite{bohm2019uncertainty}, and Wasserstein generative adversarial networks that train a generator and a discriminator \cite{arjovsky2017wasserstein, boukraichi2022uncertainty}, our method is more physics-informed and provides more insights and interpretability of the dynamical system. Our method directly outputs the distribution of unknown parameters governing the dynamical system. Furthermore, our method does not require deep neural networks and we shall show that shallow SNNs with hundreds of neurons are capable of reconstructing the joint distribution of several model parameters when inputting only a few hundred trajectories as training data.



The main contributions of our work are as follows:
\begin{itemize}
    \item We propose and analyze a local time-decoupled squared $W_2$ method for the direct reconstruction of model parameters in specific dynamical systems including ODEs, SDEs, and PDEs from time-series or spatiotemporal data. Our method takes into account both uncertainties in the initial state as well as intrinsic fluctuations, \textit{e.g.} Wiener processes, of the dynamical system when reconstructing model parameters.
    \item We analyze an SNN model whose weights are sampled from independent normal distributions. We prove that this SNN model can approximate any multidimensional continuous random variable under moderate assumptions. Furthermore, this SNN model can be trained by direct minimization of our local time-decoupled squared $W_2$ loss function.
    \item Through numerical experiments, we showcase the effectiveness of our proposed method for reconstructing the distribution of model parameters in several deterministic and stochastic dynamical systems.
\end{itemize}

The structure of this paper is as follows. In Section~\ref{section2}, we analyze a local time-decoupled squared $W_2$ loss function for reconstructing model parameters in dynamical systems. In Section~\ref{section3}, we prove that an SNN model can approximate a continuous random variable in the squared $W_2$ sense, which makes this SNN model an ideal approximate model for reconstructing the distribution of uncertain parameters in dynamical systems.
In Section~\ref{section4}, we carry out numerical experiments and showcase the effectiveness of training the SNN model by minimizing our proposed local function on the reconstruction of uncertain model parameters in different dynamical systems. In Section~\ref{summary}, we summarize our results and propose potential future research directions.

\section{A local time-decoupled squared $W_2$ loss function}
\label{section2}
In this section, we propose and analyze a local time-decoupled squared $W_2$ method for reconstructing the distribution of uncertain parameters in specific dynamical systems. Our local time-decoupled squared $W_2$ method integrates the time-decoupled squared $W_2$-distance method proposed in \cite{xia2024efficient, xia2024squared} and the local squared $W_2$-distance method in \cite{xia2024local}. \added{Therefore, both intrinsic fluctuations in the dynamical systems
and uncertainties in the initial condition can be taken into account.} First, we introduce the $W_2$ distance between distributions associated with two multidimensional random variables.

\begin{definition} 
\rm 
\label{def:W2}
For two $d$-dimensional random variables 
\begin{equation}
    \bm{X}=(X_1, ..., X_d),\,\, \hat{\bm{X}}=(\hat{X}_1,...,\hat{X}_d)\in\mathbb{R}^d
\end{equation}
with associated probability measures $\nu, \hat{\nu}$, respectively, the $W_2$-distance $W_2(\nu, \hat{\nu})$ between their probability measures is defined as
\begin{equation}
W_2(\nu, \hat{\nu}) \coloneqq \inf_{\pi(\nu, \hat \nu)}
\E_{(\bm{X}, \hat{\bm{X}})\sim \pi(\nu, \hat \nu)}\big[\|{\bm{X}}
  - \hat{{\bm{X}}}\|^{2}\big]^{\frac{1}{2}}.
\label{pidef}
\end{equation}
In Eq.~\eqref{pidef} and throughout this paper, the norm $\|\cdot\|$ denotes the $l^2$ norm of a vector: 
$\|\bm{X}\|\coloneqq \Big( \sum_{i=1}^d |X_i|^2\d
t\Big)^{\frac{1}{2}}$. $\pi(\nu, \hat \nu)$ is a coupling probability measure which iterates over all
coupled distributions of $\bm{X}(t), \hat{\bm{X}}(t)$,
defined by the condition

\begin{equation}
\begin{cases}
  {\bm{P}}_{\pi(\nu, \hat \nu)}\left(A \times 
  \mathbb{R}^d\right) ={\bm{P}}_{\nu}(A),\\
         {\bm{P}}_{\pi(\nu, \hat \nu)}\left( \mathbb{R}^d\times A\right)
         = {\bm{P}}_{\hat \nu}(A), 
\end{cases}\forall A\in \mathcal{B}( \mathbb{R}^d),
\label{pi_def}
\end{equation}
where $\mathcal{B}(\mathbb{R}^d)$ denotes the
Borel $\sigma$-algebra associated with the space of $d$-dimensional
functions in $\mathbb{R}^d$.
\end{definition}

Next, we define the local squared $W_2$ distance for the probability measures associated with the trajectories of two dynamical systems, which builds upon the local squared $W_2$ distance introduced in \cite{xia2024local}.
\begin{definition}
\rm
\label{localw2def}
    The \textbf{local squared} $\bm{W_2}$ \textbf{distance} between the probability measures associated with two dynamical systems $\{\bm{X}\}_{t\in[0, T]}, \{\hat{\bm{X}}\}_{t\in[0, T]}$ at a specific time $t$ is defined by:
    \begin{equation}
        W_{2, \delta}^{2, \text{e}}(\bm{X}(t), \hat{\bm{X}}(t))\coloneqq \int_{\mathbb{R}^d}W_2^2\big(\nu^{\text{e}}_{\bm{X}_0, \delta}(t), \hat{\nu}^{\text{e}}_{\bm{X}_0, \delta}(t)\big) \nu_0^{\text{e}}(\d\bm{X}_0).
        \label{local_squared}
    \end{equation}
    Here, $\nu_0^{\text{e}}(\cdot)$ is the empirical distribution of the initial condition for $\bm{X}(t)$ and $\hat{\bm{X}}(t)$. $\nu^{\text{e}}_{\bm{X}_0, \delta}(t)$ and $\hat{\nu}^{\text{e}}_{\bm{X}_0, \delta}(t)$ are the empirical conditional probability distributions of $\bm{X}$ and $\hat{\bm{X}}$ at time $t$ conditioned on $|\bm{X}(0)-\bm{X}_0|\leq\delta$ and $|\hat{\bm{X}}(0)-\bm{X}_0|\leq\delta$, respectively. 
\end{definition}

Now, we define the \textbf{local time-decoupled squared} $\bm{W}_2$ distance between the probability measures associated with trajectories of two dynamical systems.
\begin{definition}
\rm
\label{localtimew2def}
    Let $0=t_0<t_1<...<t_n=T$ be a time discretization mesh in $[0, T]$. The \textbf{local time-decoupled squared} $\bm{W_2}$ \textbf{distance} between the distributions associated with two dynamical systems $\{\bm{X}\}_{t\in[0, T]}, \{\hat{\bm{X}}\}_{t\in[0, T]}$ at time $t$ is defined by:
    \begin{equation}
    \begin{aligned}
                \tilde{W}_{2, \delta}^{2, \text{e}}(\bm{X}, \hat{\bm{X}})&\coloneqq \int_0^TW_{2, \delta}^{2, \text{e}}(\bm{X}(t), \hat{\bm{X}}(t))\d t \\
                &\quad= \lim_{n\rightarrow\infty, \max (t_{i+1}-t_i)\rightarrow0}\sum_{i=0}^{n-1}W_{2, \delta}^{2, \text{e}}(\bm{X}(t_i), \hat{\bm{X}}(t_i))(t_{i+1} - t_i),
    \end{aligned}
        \label{local_define1}
    \end{equation}
    where $W_{2, \delta}^{2, \text{e}}(\bm{X}(t), \hat{\bm{X}}(t))$ is the local squared $W_2$ distance in Definition~\ref{localw2def}. Empirically, we use 
    \begin{equation}
        \tilde{W}_{2, \delta}^{2, \text{e}}(\bm{X}, \hat{\bm{X}})\approx\sum_{i=0}^{n-1}W_{2, \delta}^{2, \text{e}}(\bm{X}(t_i), \hat{\bm{X}}(t_i))(t_{i+1} - t_i),
                \label{local_define}
    \end{equation}
    as an approximation to Eq.~\eqref{local_define1} for numerically calculating the local time-decoupled squared $W_2$ distance loss function. \added{Under some technical conditions, the approximation error of using the time-discretized RHS of Eq.~\eqref{local_define} to approximate the local time-decoupled squared $W_2$ distance loss function in Eq.~\eqref{local_define1} is $O(\max_{i=1}^{n-1} (t_{i+1}-t_i))$ for ODEs and $O\big(\sqrt{\max_{i=0}^{n-1} (t_{i+1}-t_i)}\big)$ for jump-diffusion processes, which will be analyzed in the proof of Theorem~\ref{well_posedness} for ODEs and in the proof of Corollary~\ref{jump_diffusion_thm} for jump-diffusion processes, respectively.}
\end{definition}

The loss function Eq.~\eqref{local_define} was also used to reconstruct the dynamics of an ODE using a parameterized neural network in \cite{xia2024local}. Yet, there is no understanding of why minimizing the loss function Eq.~\eqref{local_define} leads to the successful reconstruction of the \textit{distribution of parameters} underlying a dynamical system. Additionally, intrinsic stochasticity, such as Wiener processes in stochastic dynamical systems, was not considered in \cite{xia2024local}. In this section, we shall show that: i) the local time-decoupled squared $W_2$ distance in Eq.~\eqref{local_define1} is well-defined in several deterministic or stochastic dynamical systems and ii) minimizing the loss function in Eq.~\eqref{local_define} is a necessary condition for the reconstruction of the distribution of parameters in dynamical systems.

\added{First, we prove that the local time-decoupled squared $W_2$ distance is well-defined in some typical dynamical systems including ODEs and certain SDEs. 
We can prove that the local squared $W_2$ distance between the probability measures associated with $\{\bm{X}\}_{t\in[0, T]}$ and $\{\hat{\bm{X}}\}_{t\in[0, T]}$, which are trajectories generated by solving the two ODEs~\eqref{ODE_model} and \eqref{approximate_ODE}, 
is well-defined. In Eqs.~\eqref{ODE_model}, \eqref{approximate_ODE}, and the models we study below, we assume that the parameters are independent of the initial condition, \textit{i.e.}, $\theta$ is independent of $\bm{X}(0)$ and $\hat{\theta}$ is independent of $\hat{\bm{X}}(0)$.}


\begin{theorem}
\rm
\label{well_posedness}
Suppose 
\begin{equation}
\sup_{\bm{X}(0), \theta, t}\|\bm{X}(t)\|^2\leq X,\,\, \sup_{\bm{X}(0), \hat{\theta}, t}\|\hat{\bm{X}}(t)\|^2\leq \hat{X}
    \label{thm2_1_upper}
\end{equation} 
are uniformly bounded, where $\bm{X}(t)$ and $\hat{\bm{X}}(t)$ are solutions to the ODEs~\eqref{ODE_model} and \eqref{approximate_ODE}, respectively. Furthermore, we assume that $\bm{f}$ is continuous and uniformly bounded. Then, the limit
\begin{equation}
\lim_{\max (t_{i+1}-t_i)\rightarrow0}\sum_{i=0}^{n-1}W_{2, \delta}^{2, \text{e}}(\bm{X}(t_i), \hat{\bm{X}}(t_i))(t_{i+1} - t_i)    
\label{squared_local_limit}
\end{equation}
on the RHS of Eq.~\eqref{local_define1} exists.
\end{theorem}
We provide a proof of Theorem~\ref{well_posedness} in ~\ref{appendixA}, which is similar to the proof of Theorem 3.1 in \cite{xia2024efficient}. Theorem~\ref{well_posedness} can be extended to reveal that the local time-decoupled squared $W_2$ distance between probability measures associated with two noisy jump-diffusion processes is also well-defined. We can prove the following corollary.
\begin{corollary}
\rm
\label{jump_diffusion_thm}
Consider the following two $d$-dimensional jump-diffusion processes:
\begin{equation}
\begin{aligned}
        &\d \bm{X}(t) = \bm{f}(\bm{X}(t), t; \theta)\d t + \bm{\sigma}(\bm{X}(t),
    t; \theta)\d \bm{B}_t \\
    &\hspace{2cm}+ \int_U\bm{\beta}(\bm{X}(t), \xi, t; \theta)\tilde{N}( \d
    t, \gamma(\d\xi)),\,\,\bm{X}(0)\sim \nu_0
\end{aligned}
    \label{model_equation}
\end{equation}
and 
\begin{equation}
\begin{aligned}
    &\d \hat{\bm{X}}(t) = \bm{f}(\hat{\bm{X}}(t), t; \hat{\theta})\d t +\bm{\sigma}(\hat{\bm{X}}(t), t; \hat{\theta})\d \hat{\bm{B}}_t \\
    &\hspace{2cm}
+\int_U\bm{\beta}(\hat{\bm{X}}(t), \xi, t; \hat{\theta})\hat{N}\big(\d t,
\gamma(\d\xi)\big),\,\,\hat{\bm{X}}(0)=\bm{X}(0).
\end{aligned}
\label{approximate_equation}
\end{equation}
In Eq.~\eqref{model_equation}, $\bm{f}:\mathbb{R}^{d+1}\rightarrow\mathbb{R}^d$ and $\bm{\sigma}:\mathbb{R}^{d+1}\rightarrow\mathbb{R}^{d\times m}$ denote
the drift and diffusion functions of the SDE, respectively; $\nu_0$ is a probability measure defined on the Borel $\sigma$-algebra $\mathcal{B}(\mathbb{R}^d)$; $\bm{B}(t)$
represents an $m$-dimensional standard Brownian motion; $\tilde{N}(\d t, \gamma(\d\xi))$ is a compensated Poisson process independent of $\bm{B}_t$ defined as follows:
\begin{equation}
  \tilde{N}(\d t, \gamma(\d\xi)) \coloneqq N(\d t, \gamma(\d\xi))
  - \gamma(\d\xi) \d t,
  \label{compensated_poisson}
\end{equation} 
where $N(\d t, \gamma(\d\xi))$ is a Poisson process with intensity
$\gamma(\d\xi) \d t$, and $\gamma(\d\xi)$ is a measure defined on
$U\subseteq\mathbb{R}$, the measure space of the Poisson
process. $\hat{N}(\d
t, \gamma(\d\xi))$ is another compensated Poisson process of intensity $
\gamma(\d\xi) \d t$ and independent of $\bm{B}_t, \tilde{N}_t$ in
Eq.~\eqref{model_equation} and $\hat{\bm{B}}_t$ in Eq.~\eqref{approximate_equation}. $\theta, 
\hat{\theta}\in\mathbb{R}^{\ell}$ are uncertain model parameters. We assume that $\bm{f}, \bm{\sigma}, \bm{\beta}$ are continuous and uniformly bounded. Then 
\begin{equation}
    \tilde{W}_{2, \delta}^{2, \text{e}}(\bm{X}, \hat{\bm{X}})\coloneqq \lim_{\max (t_{i+1}-t_i)\rightarrow0}\sum_{i=0}^{n-1}W_{2, \delta}^{2, \text{e}}(\bm{X}(t_i), \hat{\bm{X}}(t_i))(t_{i+1} - t_i)
\end{equation} 
exists. Furthermore,
\begin{equation}
\begin{aligned}
      &\Big|\sum_{i=0}^{n-1}W_{2, \delta}^{2, \text{e}}(\bm{X}(t_i), \hat{\bm{X}}(t_i))(t_{i+1}-t_i)
- \tilde{W}_{2, \delta}^{2, \text{e}}(\bm{X}, \hat{\bm{X}})\Big| \\
&\quad\quad\leq 2(X + \hat{X}) T\max_i\Big(\sqrt{F_i\Delta t+\Sigma_i+B_i}
  +\!\sqrt{\hat{F}_i\Delta t + \hat{\Sigma}_i + \hat{B}_i}\Big),
\end{aligned}
     \label{convergence_result}
\end{equation}
where $\Delta t\coloneqq \max_{i=0}^{n-1}(t_{i+1}-t_i)$, and $\nu(t_i)$ and $\hat{\nu}(t_i)$ are the probability measures of $\bm{X}(t_i)$ and $\hat{\bm{X}}(t_i)$, respectively. In Eq.~\eqref{convergence_result},
\begin{equation}
  \begin{aligned}
    F_i\coloneqq & \sup_{\bm{X}_0, \theta}\E\bigg[\medint\int_{t_i}^{t_{i+1}}
      \sum_{\ell=1}^d f_{\ell}^2(\bm{X}(t^-),t^-; \theta)\d t\bigg],\\
    &\,\,\hat{F}_i\coloneqq\sup_{\bm{X}_0, \hat{\theta}}\E\bigg[\medint\int_{t_i}^{t_{i+1}} \sum_{\ell=1}^d
f_{\ell}^2(\hat{\bm{X}}(t^-),t^-; \hat{\theta})\d t\bigg], \\
\Sigma_i\coloneqq & \sup_{\bm{X}_0, \theta}\E\bigg[\medint\int_{t_i}^{t_{i+1}} \sum_{\ell=1}^d
  \sum_{j=1}^m\sigma_{\ell, j}^2(\bm{X}(t^-),t^-;\theta)\d t\bigg], \\
  &\,\,
\hat{\Sigma}_i\coloneqq\sup_{\bm{X}_0, \hat{\theta}}\E\bigg[\medint\int_{t_i}^{t_{i+1}}
\sum_{\ell=1}^d\sum_{j=1}^m\hat{\sigma}_{\ell, j}^2(\hat{\bm{X}}(t^-),t^-; \theta)\d t\bigg],\\
B_i\coloneqq & \sup_{\bm{X}_0, \theta}\E\bigg[\medint\int_{t_i}^{t_{i+1}} \sum_{\ell=1}^d
\medint\int_U \beta_{\ell}^2(\bm{X}(t^-),\xi, t^-;\theta)\gamma(\d\xi)\d t\bigg], \\&\,\, 
\hat{B}_i\coloneqq\sup_{\bm{X}_0, \hat{\theta}}\E\bigg[\medint\int_{t_i}^{t_{i+1}}
  \sum_{\ell=1}^d\medint\int_U \hat{\beta}_{\ell}^2(\hat{\bm{X}}(t^-),\xi, t^-;\hat{\theta})
  \gamma(\d\xi)\d t\bigg],\\
    &X\coloneqq \sup_{X_0, \theta, t}\E[\|\bm{X}(t)\|^2]^{\frac{1}{2}}, \,\,\, \hat{X}\coloneqq \sup_{X_0, \theta, t}\E[\|\hat{\bm{X}}(t)\|^2]^{\frac{1}{2}},
    \end{aligned}
\end{equation}
\added{where $f_{\ell}(\bm{X}(t^-),t^-; \theta), \sigma_{\ell, j}(\bm{X}(t^-),t^-; \theta)$, and $\beta_{\ell}(\bm{X}(t^-),\xi,t^-; \theta)$ refer to the left-hand limits: 
\begin{equation}
\begin{aligned}
        &f_{\ell}(\bm{X}(t^-),t^-; \theta) = \lim_{s\rightarrow t, s<t}f_{\ell}(\bm{X}(s),s; \theta),\\
    &\hspace{0.5cm} \sigma_{\ell, j}(\bm{X}(t^-),t^-; \theta) = \lim_{s\rightarrow t, s<t}\sigma_{\ell, j}(\bm{X}(s),s; \theta),\\
    &\hspace{1cm}\beta_{\ell}(\bm{X}(t^-),t^-,\xi; \theta) = \lim_{s\rightarrow t, s<t}\beta_{\ell}(\bm{X}(s),s,\xi; \theta).
\end{aligned}
\end{equation}}
\end{corollary}
The proof of Corollary~\ref{jump_diffusion_thm} is in~\ref{appendixB}. Next, we show that the local time-decoupled squared $W_2$ distance $\tilde{W}_{2, \delta}^{2, \text{e}}(\bm{X},  \hat{\bm{X}})$ in Definition~\eqref{localtimew2def} can be bounded by the squared $W_2$ distance between the two probability measures associated with the two sets of parameters $\theta$ and $\hat{\theta}$ in Eqs.~\eqref{ODE_model} and \eqref{approximate_ODE}, which implies the necessity of minimizing the local time-decoupled squared $W_2$ distance $\tilde{W}_{2, \delta}^{2, \text{e}}(\bm{X},  \hat{\bm{X}})$ if we wish to match the distribution of $\theta$ using the distribution of the reconstructed $\hat{\theta}$.

\begin{theorem}
\rm
\label{theorem1}
Suppose the drift, diffusion, and jump functions in the two jump-diffusion processes Eqs.~\eqref{model_equation} and ~\eqref{approximate_equation} satisfy the following Lipschitz condition:
\begin{equation}
\begin{aligned}
        &\sum_{i=1}^d|f_i(\bm{X}, t;\theta) - f_i(\hat{\bm{X}}, t;\hat{\theta})|\leq C(\|\bm{X} - \hat{\bm{X}}\| + \|\theta-\hat{\theta}\|), \\
        &\sum_{i=1}^d|\sigma_{i, j}(\bm{X}, t;\theta) - \sigma_{i, j}(\hat{\bm{X}}, t;\hat{\theta})|\leq C(\|\bm{X} - \hat{\bm{X}}\| + \|\theta-\hat{\theta}\|),\,\, j=1,...,m,\\ 
        &\sum_{i=1}^d|\beta_i(\bm{X}, \xi, t;\theta) - \beta_i(\hat{\bm{X}},\xi, t;\hat{\theta})|\leq C(\|\bm{X} - \hat{\bm{X}}\| + \|\theta-\hat{\theta}\|), \\
    &\quad\quad
    \forall \bm{X}, \hat{\bm{X}}\in\mathbb{R}^d,\,\,\forall \theta, \hat{\theta}\in\mathbb{R}^{\ell}, \,\, C<\infty.
\end{aligned}
\label{Lipschitz}
\end{equation}
In Eq.~\eqref{Lipschitz}, $f_i$ and $\beta_i$ denote the $i^{\text{th}}$ component of $\bm{f}$ and $\bm{\beta}$, respectively, and $\sigma_{i, j}$ denotes the $(i, j)$ element of the matrix $\bm{\sigma}$ in Eq.~\eqref{model_equation}.
Furthermore, we assume that the assumptions in Corollary~\ref{jump_diffusion_thm} hold and the sixth-order moments
\begin{equation}
    \E[\|\theta\|^6]\leq \Theta_6,\,\, \E[\|\hat{\theta}\|^6]\leq \hat{\Theta}_6
\end{equation}
are uniformly bounded. Then, $\tilde{W}_{2, \delta}^{2, \text{e}}(\bm{X}, \hat{\bm{X}})$ can be bounded by the squared $W_2$ distance $W_2^2(\mu, \hat{\mu})$:
    \begin{equation}
    \begin{aligned}
    &\E[\tilde{W}_{2, \delta}^{2, \text{e}}(\bm{X}, \hat{\bm{X}})]
    \leq 8C_0T\delta^2 \exp(C_0T) + \tfrac{6C_1}{C_0}T\exp(C_0T)\\
    &\quad\quad\times\big(W_2^2(\mu, \hat{\mu}) + 2C_3 \E[h(N^{\#}(\bm{X}_0;\delta), \ell)\cdot (\Theta_6^{\frac{1}{3}}+\hat{\Theta}_6^{\frac{1}{3}})]\big),
       \end{aligned}
        \label{para_dependence_bound}
    \end{equation}
    where $\mu, \hat{\mu}$ are the probability measures associated with $\theta$ and $\hat{\theta}$, respectively. In Eq.~\eqref{para_dependence_bound}, $C_0, C_1, C_2$ are three constants, and
        \begin{equation}
h(N, \ell)\coloneqq\left\{
\begin{aligned}
&2N^{-\frac{1}{2}}(\log(1+N) + 1), \ell\leq4,\\
&2N^{-\frac{2}{\ell}}, \ell> 4.
\end{aligned}
\right.
\label{t_def}
\end{equation}
In Eq.~\eqref{para_dependence_bound}, $N^{\#}(\bm{X}_0; \delta)$ refers to the number of trajectories in the data set such that their initial conditions satisfy $\|\bm{X}(0)-\bm{X}_0\|\leq \delta$.
\end{theorem}

We prove Theorem~\ref{theorem1} in \ref{AppendixC}. Theorem~\ref{theorem1} implies that minimizing the local time-decoupled squared $W_2$ distance $\tilde{W}_{2, \delta}^{2, \text{e}}$ is a necessary condition if we wish to match the distribution of $\theta$ by the distribution of $\hat{\theta}$.
In Eq.~\eqref{para_dependence_bound}, there is a trade-off between the first term and the third term on the RHS: if we increase $\delta$, then the first term on the RHS of Eq.~\eqref{para_dependence_bound} will increase but the factor $h(N^{\#}(\bm{X}_0;\delta);\ell)$ in the last term will decrease. Thus, it is important to choose an appropriate $\delta$ to keep both the first and last terms on the RHS of Eq.~\eqref{para_dependence_bound} small so that the expected local time-decoupled squared $W_2$ distance $\E[\tilde{W}_{2, \delta}^{2, \text{e}}(\bm{X}, \hat{\bm{X}})]$ can be well controlled by $W_2^2(\mu, \hat{\mu})$. \added{In this case, minimizing $\E[\tilde{W}_{2, \delta}^{2, \text{e}}(\bm{X}, \hat{\bm{X}})]$ is necessary to minimize $W_2^2(\mu, \hat{\mu})$, \textit{i.e.}, match the distribution of $\theta$ by the distribution of $\hat{\theta}$.} On the other hand, in Eq.~\eqref{t_def}, \added{the rate of convergence when using the empirical probability measure} is $2N^{-\frac{1}{2}}(\log(1+N) + 1), \ell\leq 4$ or $N^{-\frac{2}{\ell}}, \ell> 4$ as the number of observed trajectories increases, which depends only on the dimensionality of the parameter $\theta$ instead of on the dimensionality of $\bm{X}(t)$. 

  ODEs can be regarded as special jump-diffusion processes whose diffusion function and jump function are both 0. Thus,
 Theorem~\ref{theorem1} can be applied to bound the time-decoupled local squared $W_2$ distance between the distributions of trajectories of the two ODEs Eq.~\eqref{ODE_model} and ~\eqref{approximate_ODE} using the squared $W_2$ distance between the probability distributions of $\theta$ and $\hat{\theta}$. 
Finally, Theorem~\ref{theorem1} can also be generalized to the cases of reconstructing parameters in spatiotemporal stochastic partial differential equations (SPDEs). We provide an example of bounding 
the local time-decoupled squared $W_2$ distance between solutions to a parabolic SPDE associated with two different sets of model parameters in \ref{AppendixD}.

\added{Finally, one can also take into account time discretization errors using a time discretization scheme such as the Runge-Kutta scheme for solving ODEs and the strong It\^o-Taylor approximation \cite{kloeden1992numerical} for numerically solving SDEs. Specifically, for stiff problems, a detailed analysis of numerical implementation is worth carrying out.} 
Additionally, 
it is also possible to extend Theorem~\ref{theorem1} to more complicated spatiotemporal dynamics by replacing $B_t$ with more complicated spatiotemporal cylindrical Brownian noise \cite{liu2015stochastic} or considering spatiotemporal integrodifferential equations \cite{deng2024adaptive}.
Such discussions require a more detailed analysis of SPDEs and are thus beyond the scope of this paper.

\section{An SNN model for approximating the distribution of continuous random variables}
\label{section3}
In this section, we analyze an SNN model used in \cite{xia2024local}.  We apply this SNN model to reconstruct the distribution of unknown parameters for general dynamical systems. Specifically, we will prove that this SNN model can approximate the probability distribution of any continuous multidimensional random variable under certain technical assumptions in the $W_2$ distance sense, and the training of this SNN can be done by direct minimization of the local time-decoupled squared $W_2$ loss function Eq.~\eqref{local_define} analyzed in Section~\ref{section2}. 
We sketch the structure of the SNN with weight uncertainty in Fig.~\ref{fig:nn_model}. 
 All weights in the neural networks $\{w_{i, j, k}\}$ are sampled from independent normal distributions with $w_{i, j, k}\sim\mathcal{N}(a_{i, j, k}, \sigma_{i, j, k}^2)$. The biases $\{b_{i, k}\}$ for all $i, k$ are deterministic. The means and variances of the weights $\{a_{i, j, k}\}, \{\sigma_{i, j, k}^2\}$ and the biases $\{b_{i, k}\}$ are optimized through training. 
    \begin{figure}[h!]
    \centering
\includegraphics[width=0.9\linewidth]{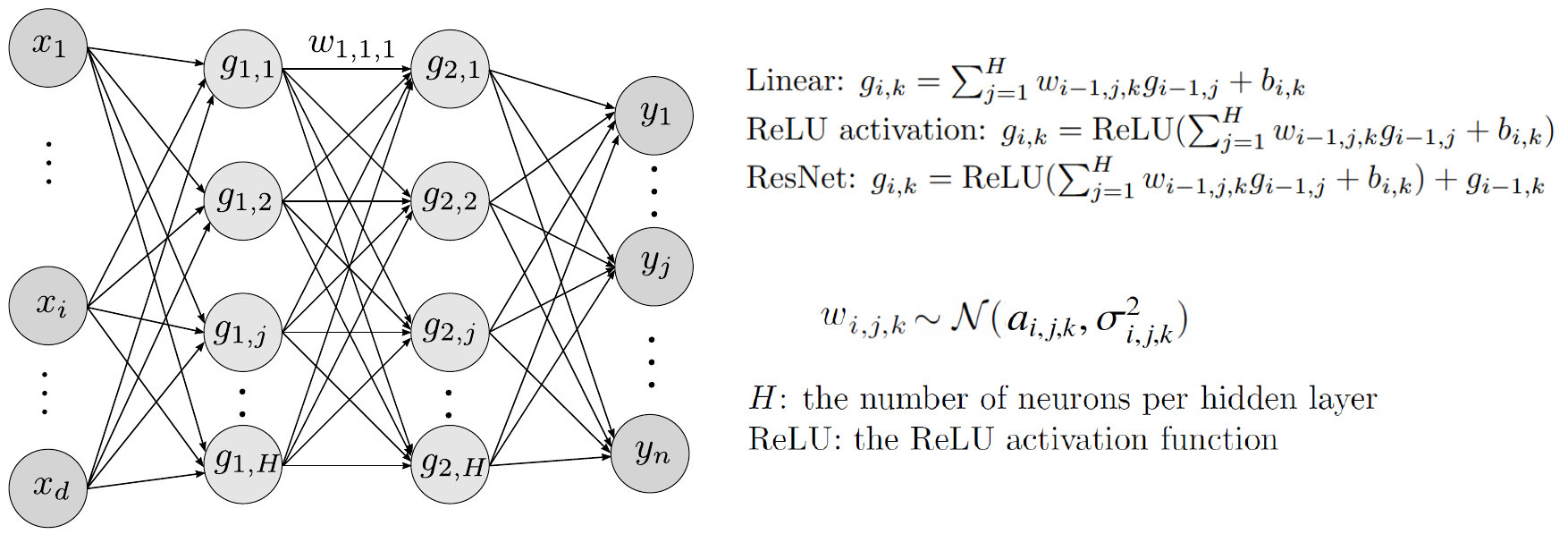}
    \caption{\footnotesize A sketch of the structure of the neural network model with weight uncertainty used in \cite{xia2024local} and in this paper. The weights $w_{i, j, k}\sim\mathcal{N}(a_{i, j, k}, \sigma_{i, j, k}^2)$ are independently sampled, \textit{i.e.}, $w_{i_1, j_1, k_1}$ is independent of $w_{i_2, j_2, k_2}$ when $(i_1, j_1, k_1)\neq (i_2, j_2, k_2)$. When using this neural network model to make predictions, for each input $\bm{x}=(x_1,\ldots,x_d)\in D\subseteq\mathbb{R}^d$, we resample all weights $\{w_{i, j, k}\}$ again. For each neuron in the hidden layer, one of the following three forward propagation methods is considered: the linear operation, the ReLU activation, or the ResNet technique.}
    \label{fig:nn_model}
\end{figure}

First, we prove that the SNN model in Fig.~\ref{fig:nn_model} can approximate any continuous random variable whose probability distribution follows a parameterized multivariate normal distribution in the squared $W_2$ distance sense. In the following, $\mathcal{N}(\bm{y} - \bm{b}, \Sigma)$ denotes the probability density function of a multivariate normal distribution with mean $\bm{b}$ and the covariance matrix $\Sigma$ and takes the form:
\begin{equation}
    \mathcal{N}(\bm{y}-\bm{b}, \Sigma) \coloneqq \frac{1}{\sqrt{2\pi}^{d'}}\cdot\frac{1}{|\Sigma|^{\frac{1}{2}}}\cdot\exp\big(\tfrac{1}{2}(\bm{y} - \bm{b})^T \Sigma (\bm{y} - \bm{b})\big).
\end{equation}


\begin{theorem}
\rm
\label{thm_3_1}
Let $\bm{y}\in\mathbb{R}^{d'}$ be a continuous random variable whose probability density is $f_{\bm{x}}(\bm{y}), \bm{x}\in D\subseteq\mathbb{R}^d$ such that $f_{\bm{x}}(\bm{y}) =\mathcal{N}(\bm{y}-\bm{b}({\bm{x}}), A(\bm{x})^TA({\bm{x}}))$. $D$ is a bounded set and $\bm{x}$ has a probability measure $\gamma(\cdot)$ on $D$. 
We make the following assumptions:
\begin{enumerate}
    \item For any sequence of sets $\{D_i\}_{i=1}^{\infty}$, if $D_i\rightarrow D$ as $i\rightarrow\infty$, then $\lim_{i\rightarrow\infty}\gamma(D_i)=\gamma(D)=1$. Furthermore, for any $\Delta x>0$, we can find a set of equidistance grids $\{\bm{x}_i\}_{i=1}^{K}\subseteq D$ such that the distance between two adjacent grids is $\Delta x$, $D\subseteq \cup_{i=1}^K\otimes_{j=1}^d[x_i^j, x_i^j+\Delta \bm{x})$, and $\otimes_{j=1}^d[x_{i_1}^j, x_{i_1}^j+\Delta \bm{x})\cap \otimes_{j=1}^d[x_{i_2}^j, x_{i_2}^j+\Delta \bm{x})=\emptyset$ if $i_1\neq i_2$.
    \item $f_{\bm{x}}(\bm{y})$ is uniformly continuous in $\bm{x}$ such that for any $\epsilon>0$, there exists $\Delta x>0$ satisfying:
\begin{equation}
    W_2^2(f_{\bm{x}}, f_{\tilde{\bm{x}}})< \epsilon,\,\, \forall \|\bm{x}-\tilde{\bm{x}}\|\leq \Delta x, \,\,\bm{x}, \tilde{\bm{x}}\in D.
\end{equation}
\item $
    Y\coloneqq \sup_{\bm{x}\in D}\|\bm{b}({\bm{x}})\|^2 + \|A({\bm{x}})^TA({\bm{x}})\|_F^2<\infty$, 
where $\|\cdot\|_F$ is the Frobenius norm of a matrix.
\end{enumerate}

Then, for any $\epsilon_0>0$, there exists an SNN model as described in Fig.~\ref{fig:nn_model} which uses the ReLU activation and the linear forward propagation such that 
if we denote the probability density function of the output by $\hat{f}_{\bm{x}}$ given the input $\bm{x}$, the following inequality holds:
        \begin{equation}
            \int_D W_2^2(f_{\bm{x}}, \hat{f}_{\bm{x}})\gamma(\d \bm{x})\leq \epsilon_0.
        \end{equation}
\end{theorem}

We prove Theorem~\ref{thm_3_1} in \ref{AppendixE}. 
 Theorem~\ref{thm_3_1} can be further generalized and we can prove that the probability density distribution of the output of the SNN model in Fig.~\ref{fig:nn_model} can approximate 
a multivariate Gaussian mixture model (defined in Corollary~\ref{col3_1} below) in the squared $W_2$ distance sense. First, we prove the following lemma.

\begin{lemma}
\rm
\label{mixed_gauss}
Let $\bm{y}\in\mathbb{R}^{d'}$ be a continuous random variable that has the following probability density function of a Gaussian mixture model: 
\begin{equation}
    f(\bm{y}) = \sum_{i=1}^s p_i\mathcal{N}(\bm{y}-\bm{b}_i, A_i^TA_i), \,\, p_i>0, \,\,\sum_{i=1}^sp_i=1.
    \label{gauss_model}
\end{equation}
Suppose another continuous random variable $\hat{\bm{y}}\in\mathbb{R}^{d'}$ has the following probability density function:
\begin{equation}
    \hat{f}(\hat{\bm{y}}) = \sum_{i=1}^s \hat{p}_i\mathcal{N}(\hat{\bm{y}} - \bm{b}_i, A_i^TA_i) + p(\hat{\bm{y}}),\,\, \hat{p}_i>0
\end{equation}
where $\hat{p}_i\leq p_i$, $p(\cdot):\mathbb{R}^{d'}\rightarrow\mathbb{R}^+\cup\{0\}$ is nonnegative, and $\int_{\mathbb{R}^{d'}}\|\hat{\bm{y}}\|^2 p(\hat{\bm{y}})\d\hat{\bm{y}}\leq\infty$. Then, the following bound of the squared $W_2$ distance between the probability measures of $\bm{y}$ and $\hat{\bm{y}}$ holds: 
\begin{equation}
    W_2^2(f, \hat{f})\leq 2\big(\sum_{i=1}^s(p_i-\hat{p}_i)(\|\bm{b}_i\|^2 + \|A_i^TA_i\|_F^2) + \int_{\mathbb{R}^{d'}} \|\bm{y}\|^2 p(\bm{y})\d\bm{y}\big).
    \label{mixed_gauss_result}
\end{equation}
\end{lemma}

\begin{proof}
First, if $p_i=\hat{p}_i, i=1,...,s$, then $f(\bm{y})=\hat{f}(\bm{y})$ and $W_2^2(f, \hat{f})=0$, indicating that Eq.~\eqref{mixed_gauss_result} holds. Next, we assume that $\sum_{i=1}^s\hat{p}_i<1$.
    Without loss of generality, we assume that $\bm{y}$ and $\hat{\bm{y}}$ are independent of each other.
    We define a special coupling probability measure of the random variable $(\bm{y}, \hat{\bm{y}})\in\mathbb{R}^{2d'}$:
\begin{equation}
\begin{aligned}
        \pi(f, \hat{f})(\bm{y}, \hat{\bm{y}}) &= \big[\sum_{i=1}^s\hat{p}_i\mathcal{N}(\bm{y}-\bm{b}_i, A_i^TA_i)\big]\delta(\bm{y}-\hat{\bm{y}}) \\
        &\quad+ \frac{1}{p}[\sum_{i=1}^s(p_i-\hat{p}_i)\mathcal{N}(\bm{y}-\bm{b}_i, A_i^TA_i)]\cdot p(\hat{\bm{{y}}}),
\end{aligned}
\end{equation}
where $p\coloneqq \sum_{i=1}^s(p_i-\hat{p}_i)=\int_{\mathbb{R}^{d'}}p(\hat{\bm{y}})\d\hat{\bm{y}}$, and $\delta$ is the Dirac delta measure. We can check that the marginal distributions of $\pi(f, \hat{f})$ coincide with $f(\bm{y})$ and $\hat{f}(\hat{\bm{y}})$, respectively. Furthermore, we have
\begin{equation}
\begin{aligned}
    W_2^2(f, \hat{f})&\leq\E_{(\bm{y}, \hat{\bm{y}})\sim\pi(f, \hat{f})}\big[\|\bm{y}-\hat{\bm{y}}\|^2\big]\\
    &\leq \int_{\mathbb{R}^{2d'}}\frac{1}{p}[\sum_{i=1}^s(p_i-\hat{p}_i)\mathcal{N}(\bm{y}-\bm{b}_i, A_i^TA_i)]\cdot p(\bm{\hat{y}})\|\bm{y}-\hat{\bm{y}}\|^2\d\bm{y}\d\hat{\bm{y}}\\
    &\leq \int_{\mathbb{R}^{2d'}}\frac{1}{p}[\sum_{i=1}^s(p_i-\hat{p}_i)\mathcal{N}(\bm{y}-\bm{b}_i, A_i^TA_i)]\cdot p(\bm{\hat{y}})\cdot 2(\|\bm{y}\|^2+\|\hat{\bm{y}}\|^2)\d\bm{y}\d\hat{\bm{y}}\\
    &\leq 2\sum_{i=1}^s(p_i-\hat{p}_i)(\|\bm{b}_i\|^2 + \|A_i^TA_i\|_F^2) + 2\int_{\mathbb{R}^{d'}} \|\hat{\bm{y}}\|^2 p(\hat{\bm{y}})\d\hat{\bm{y}},
    \end{aligned}
\end{equation}
which proves the inequality~\eqref{mixed_gauss_result}.
\end{proof}

Next, we show that the SNN model in Fig.~\ref{fig:nn_model} can approximate a multivariate Gaussian mixture model in the squared $W_2$ distance sense.
\begin{corollary}
\rm
\label{col3_1}
    Let $\bm{y}\in\mathbb{R}^{d'}$ be a continuous random variable with a probability density function $f_{\bm{x}}$, where $\bm{x}\in D\subseteq\mathbb{R}^d$ is continuous and has a probability density $\gamma(\cdot)$. At each $\bm{x}$, $f_{\bm{x}}$ is the probability density function of a Gaussian mixture model: 
    \begin{equation}
        f_{\bm{x}}(\bm{y}) = \sum_{r=1}^s p_r(\bm{x})\mathcal{N}(\bm{y} - \bm{b}_r(\bm{x}), A_r^T(\bm{x})A_r(\bm{x})), \,\, \sum_{r=1}^s p_r(\bm{x})=1,\,\, p_r(\bm{x})>0.
    \end{equation}
    We make the following three additional assumptions:
    \begin{enumerate}
    \item For any sequence of sets $\{D_i\}_{i=1}^{\infty}$, if $D_i\rightarrow D$ as $i\rightarrow\infty$, then $\lim_{i\rightarrow\infty}\gamma(D_i)=\gamma(D)=1$. Additionally, 
    $D$ is a bounded set in $\mathbb{R}^d$ and for any $\Delta x>0$, we can find a set of equidistance grids $\{\bm{x}_i\}_{i=1}^{K}\subseteq D$ such that $D\subseteq \cup_{i=1}^K\otimes_{j=1}^d[x_i^j, x_i^j+\Delta \bm{x})$ and $\otimes_{j=1}^d[x_{i_1}^j, x_{i_1}^j+\Delta \bm{x})\cap \otimes_{j=1}^d[x_{i_2}^j, x_{i_2}^j+\Delta \bm{x})=\emptyset$ if $i_1\neq i_2$.
        \item $f_{\bm{x}}(\bm{y})$ is uniformly continuous in $\bm{x}$ such that for any $\epsilon>0$, there exists $\delta>0$:
    \begin{equation}
        W_2^2\big(f_{\bm{x}}(\bm{y}), f_{\tilde{\bm{x}}}(\bm{y})\big)<\epsilon, \,\,\forall \|\bm{x} - \tilde{\bm{x}}\|\leq\delta, \,\,\forall \bm{x}, \tilde{\bm{x}}\in D.
    \end{equation}
    \item The quantity
    \begin{equation}
        \max_{1\leq r\leq s}\|A_r^T(\bm{x})A_r^T(\bm{x})\|_F^2 + \|\bm{b}_r(\bm{x})\|^2
        \label{bounded_condition}
    \end{equation} 
    is uniformly bounded for all $\bm{x}\in D$. 
    \end{enumerate}
    Then, for any positive number $c_0>0$, there exists an SNN with weight uncertainty described in Fig.~\ref{fig:nn_model} such that:
    \begin{equation}
        \int_D W_2^2(f_{\bm{x}}, \hat{f}_{\bm{x}})\gamma(\d\bm{x})\leq c_0.
        \label{mixed_result}
    \end{equation}
    Here, $\hat{f}_{\bm{x}}$ is the distribution of the output of the SNN when the input is $\bm{x}$.
\end{corollary}

We prove Corollary~\ref{col3_1} in \ref{AppendixF}.
 Finally, we prove that for each continuous random variable $\bm{y}\in\mathbb{R}^{d'}$ with a probability distribution function $f(\bm{y})$, under certain conditions, we can find a random variable $\hat{\bm{y}}\in\mathbb{R}^{d'}$ obeying a Gaussian mixture distribution whose probability distribution function is denoted by $\hat{f}$ such that $\hat{f}$ can approximate $f$ in the $W_2$ sense.
We have the following result.

\begin{theorem}
\rm
\label{thm3_2}
     Suppose $\bm{y}=(y_1,...,y_{d'})\in\mathbb{R}^{d'}$ is a continuous random variable with a smooth probability density function $f(\bm{y})\in  L^2(\mathbb{R}^{d'})\cap L^{\infty}(\mathbb{R}^{d'})$. Furthermore, we assume the following conditions hold:
     \begin{enumerate}
         \item $f(\bm{y})$ is uniformly continuous in $\mathbb{R}^n$
         \item \begin{equation}
         |f|_{\text{mix}}\coloneqq \sum_{|\bm{n}|_0\leq d'}\|\partial_{\bm{n}}^{|\bm{n}|_0}f\|_{L^2}<\infty,\,\,\,|\sqrt{f}|_{\text{mix}}<\infty
     \end{equation}
     where $|\bm{n}|_0$ is the number of nonzero components in $\bm{n}$, $\bm{n}=(n_1,...,n_j)$ satisfying $1\leq n_1<...<n_j\leq d'$, and  $\partial_{\bm{n}}f\coloneqq\partial_{y_{n_1}}...\partial_{y_{n_j}}f$. 
     \item $|fy_i^2|_{\text{mix}}<\infty$ and $|fy_i^2y_j^2|_{\text{mix}}<\infty$.
     \end{enumerate}
     \added{Then, $\forall \epsilon >0$, there exists a probability density function of a Gaussian mixture model:
    \begin{equation}
        \tilde{f}_{\sigma^2, n_0(\sigma)}(\bm{y}) \coloneqq \sum_{i=1}^{(n_0(\sigma)+1)^{d'}} p_i\mathcal{N}(\bm{y}- \bm{y}_i;\sigma^2I_{d'\times d'}),\,\,\bm{y}_i\in\mathbb{R}^{d'}
        \label{pndef}
    \end{equation}
    such that as $\sigma\rightarrow0^+$ and $n(\sigma)\rightarrow\infty$:
    \begin{equation}
        \tilde{f}_{\sigma^2, n_0(\sigma)}(\bm{y})\rightarrow f(\bm{y})
    \end{equation}
    uniformly in $\mathbb{R}^{d'}$. Furthermore,
    \begin{equation}
    W_2^2(f, \tilde{f}_{\sigma^2, n_0(\sigma)})\leq 24\epsilon.
        \label{thm3_result}
    \end{equation}
    In Eq.~\eqref{pndef}, $I_{d'\times d'}$ is a $d'\times d'$ identity matrix.}
\end{theorem}

%
We prove Theorem~\ref{thm3_2} in \ref{proof_thm3_2}.
For any continuous random variable $\bm{y}\in\mathbb{R}^{d'}$ with a probability density function $f$ satisfying the assumptions in Theorem~\ref{thm3_2}, there exists a random variable $\hat{\bm{y}}$ whose probability density function is the probability density function of a multivariate Gaussian mixture model denoted by $\tilde{f}_{\sigma^2, n_0}$ such that $W_2^2(f, \tilde{f}_{\sigma^2, n_0})\leq \epsilon$. 
\added{From Corollary~\ref{col3_1}, there exists an SNN (Fig.~\ref{fig:nn_model}) which takes the input $x\equiv 1$ such that $W_2^2(\tilde{f}_{\sigma^2, n_0}, \hat{f}_{\bm{x}\equiv 1})\leq \epsilon$, where $\hat{f}_{\bm{x}\equiv 1}$ is the probability density function of the output of the SNN. Thus, we have $W_2^2(f, \hat{f}_{\bm{x}\equiv 1})\leq 4\epsilon$, \textit{i.e.}, the SNN model in Fig.~\ref{fig:nn_model} can approximate the distribution of any continuous random variable in the squared $W_2$ sense under the assumptions specified in Theorem~\ref{thm3_2}. 
From Theorem~\ref{theorem1}, minimizing the local time-decoupled squared $W_2$ loss function Eq.~\eqref{local_define} is necessary to minimize $W_2^2(\mu, \hat{\mu})$, \textit{i.e.}, to match the distribution of model parameters $\theta$ by the distribution of $\hat{\theta}$ in dynamical systems like Eqs.~\eqref{model_equation} and~\eqref{approximate_equation}. Thus, we can train the SNN model in Fig.~\ref{fig:nn_model} by direct minimization of the local time-decoupled squared $W_2$ loss function Eq.~\eqref{local_define}.}

\textbf{Remark:} \added{Consider the more general model $\bm{y}=f(\bm{x}, \theta), \bm{x}\in\mathbb{R}^d, \bm{y}\in\mathbb{R}^{d'}$ as in \cite{xia2024local}, where $\omega$ is a latent random variable in the model (\textit{e.g.} measurement noise). $\bm{y}$ is a continuous random variable whose distribution is determined by the observed continuous variable $\bm{x}\in D\subseteq \mathbb{R}^d$.} $D$ is a bounded set in $\mathbb{R}^d$, and for any sequence of sets $\{D_i\}_{i=1}^{\infty}$, if $D_i\rightarrow D$ as $i\rightarrow\infty$, then $\lim_{i\rightarrow\infty}\gamma(D_i)=\gamma(D)=1$. Additionally, we assume that for any $\Delta x>0$, we can find a set of equidistance grids $\{\bm{x}_i\}_{i=1}^{K}\subseteq D$ such that $D\subseteq \cup_{i=1}^K\otimes_{j=1}^d[x_i^j, x_i^j+\Delta \bm{x})$ and $\otimes_{j=1}^d[x_{i_1}^j, x_{i_1}^j+\Delta \bm{x})\cap \otimes_{j=1}^d[x_{i_2}^j, x_{i_2}^j+\Delta \bm{x})=\emptyset$ if $i_1\neq i_2$. $\omega$ are continuous latent parameters in the model sampled from an unknown distribution. We denote the distribution of $\bm{y}$ given $\bm{x}$ by $f_{\bm{x}}$. Under certain assumptions, for any $\epsilon>0$, there exists an SNN in Fig.~\ref{fig:nn_model} such that:
\begin{equation}
    \int_D W_2^2(f_{\bm{x}}, \hat{f}_{\bm{x}})\gamma(\d\bm{x})\leq \epsilon.
\end{equation}
where $\hat{f}_{\bm{x}}$ is the probability density function of the output of the SNN given the input $\bm{x}$. We give a brief discussion on this ``universal approximation" ability of the SNN model to approximate a family of continuous random variables $\bm{y}=f(\bm{x}, \theta)$ in~\ref{appendix_universal} for all $\bm{x}\in D\subseteq \mathbb{R}^d$, where $\theta$ is latent uncertain model parameters. 
Our SNN
can approximate a family of probability density functions $f(\bm{x}, \theta), \bm{x}\in D$ simultaneously, while in \cite{lu2020universal} only a single probability density function $f(\theta)$ of the unknown variable $\theta$ is to be approximated. Furthermore, our SNN utilizes only seven hidden layers and the number of neurons in each layer scales linearly with the dimensionality of either the input or the output variable. This implies that even with a shallow neural network, we might be able to reconstruct a family of uncertainty models $\bm{y}=f(\bm{x}, \theta)$ characterized by different values of $\bm{x}\in D$.


\section{Numerical results}
\label{section4}
In this section, we conduct numerical experiments to test our proposed local squared $W_2$ method, which involves training the SNN model in Fig.~\ref{fig:nn_model} by minimizing Eq.~\eqref{time_coupling0}, a scaled numerical approximation to the local time-decoupled squared $W_2$ distance loss function Eq.~\eqref{local_define}. The $W_2$ distance between two empirical probability measures is then numerically evaluated using the $\texttt{PoT}$ package of Python in \cite{flamary2021pot}.
A pseudocode of our method is given in Algorithm~\ref{algorithm_1}.

\begin{algorithm}
\footnotesize
\caption{\footnotesize \added{The pseudocode of our local time-decoupled
    squared $W_2$ method for training the SNN (the
    time-decoupled squared $W_2$ loss in the \textbf{while} loop can
    be replaced with other loss functions).}}
\begin{algorithmic}
  \STATE Given $N$ observed time-series data $\{\bm{X}_i(t_j),
    t_j=j\Delta t, j=1,..., N_T\}_{i=1}^N$, the underlying dynamical system with unknown model parameters $\theta$ (such as the jump-diffusion process Eq.~\eqref{model_equation}), the stopping criteria
    $\epsilon>0$, and the maximal epochs $i_{\max}$.
    \STATE Initialize the SNN model in Fig.~\ref{fig:nn_model}.
    \STATE Sample $N$ sets of approximate parameters $\{\hat{\theta}\}_{i=1}^N$ by inputting a scalar $1$ into the SNN and evaluate $N$ times independently.
  \STATE Generate $N$ trajectories $\{\hat{\bm{X}}\}_{i=1}^N$ from the approximate dynamical system (such as
    Eq.~\eqref{approximate_equation}) with the approximate parameters 
    $\{\hat{\theta}\}_{i=1}^N$.
  \WHILE{$\tilde{W}_{2, \delta}^{2, \text{e}}(\bm{X}, \hat{\bm{X}})>\epsilon$ \&\& $i<i_{\max}$ } 
  \STATE Perform gradient descent to minimize the loss function
    $\tilde{W}_{2, \delta}^{2, \text{e}}(\bm{X}, \hat{\bm{X}})$ and update the parameters (biases \& means and variances of weights) in the SNN model.
    \STATE Sample $N$ sets of approximate parameters $\{\hat{\theta}\}_{i=1}^N$ from the updated SNN model by inputting a scalar $1$ into the SNN and evaluate $N$ times independently.
    \STATE Generate $N$
    trajectories from the approximate dynamical system with the approximate parameters
    $\{\hat{\theta}\}_{i=1}^N$.
    \ENDWHILE
  \RETURN the trained SNN model
\end{algorithmic}
\label{algorithm_1}
\end{algorithm}

Specifically, when applying the SNN model in Fig.~\ref{fig:nn_model} for the reconstruction of the distribution for uncertain model parameters, we assume that uncertain model parameters are sampled from the same underlying distribution across all samples and thus always input a scalar $1$ as the input into the SNN. 
Default hyperparameters and training settings are given in \ref{appendix_training}. In the following, errors in the distribution of reconstructed model parameters denote the scaled squared $W_2$ distance:
\begin{equation}
    \text{error}\coloneqq\frac{W_2^2(\mu_{\theta}, \hat{\mu}_{\hat{\theta}})}{\|\theta\|^2},
    \label{relative_error}
\end{equation}
\added{where $\theta$ and $\hat{\theta}$ are unknown ground truth model parameters and reconstructed model parameters, $\textit{e.g.}$ in Eqs.~\eqref{ODE_model} and~\eqref{approximate_ODE}, and $\mu_{\theta}$ and $\hat{\mu}_{\theta}$ are the distributions of $\theta$ and $\hat{\theta}$, respectively.} For numerically solving ODEs in all examples, we use the \texttt{odeint} function with default settings in the \texttt{torchdiffeq} package \cite{chen2018neural}. \added{In Example~\ref{example4}, we use the package developed in \cite{xia2024efficient} for numerically solving a jump-diffusion process, which allows for back-propagation and gradient descent for hyperparameter optimization in the neural networks. The It\^o integral is adopted for evaluating the stochastic integrals.} The numerical experiment in
Example \ref{example1} is conducted using Python 3.11 on a desktop with a 32-core Intel®
i9-13900KF CPU
(when comparing runtimes and RAM usage, we train each model on just one core).
Numerical experiments in Examples \ref{example2}, \ref{example3}, and \ref{example4} are carried out using Python 3.11 on NYU HPC with GPU.

\begin{example}
    \rm
    \label{example1}

First, we consider reconstructing the following 2D ODE characterizing the Lokta-Volterra predator-prey dynamics with one uncertain predation rate parameter:
\begin{equation}
    \begin{aligned}
        &\frac{\d x}{\d t} = 2 x - c xy,\,\,\frac{\d y}{\d t} = \tfrac{1}{4}c xy - 2y,\\
        &\hspace{1cm}(x(0), y(0))=(\xi_1,\xi_2),\,\, c\sim\mathcal{U}(2, 4),\,\,\xi_1, \xi_2\sim\mathcal{U}(1, 2), \,\, t\in[0, 8].
    \end{aligned}
    \label{example1_model}
\end{equation}
In Eq.~\eqref{example1_model}, $c$ is the uncertain predation rate parameter, and $\xi_1, \xi_2$ are two independent random variables. We train the SNN model in Fig.~\ref{fig:nn_model} by minimizing the loss function Eq.~\eqref{time_coupling0} (the neighborhood size $\delta=0.4$) to reconstruct the distribution of the parameter $c$. For comparison, we also minimize other loss functions (definitions given in \ref{appendix_loss}) commonly used in statistical inference tasks to train the SNN model in Fig.~\ref{fig:nn_model}.

\begin{figure}[H]
\centering
\includegraphics[width=0.7\textwidth]{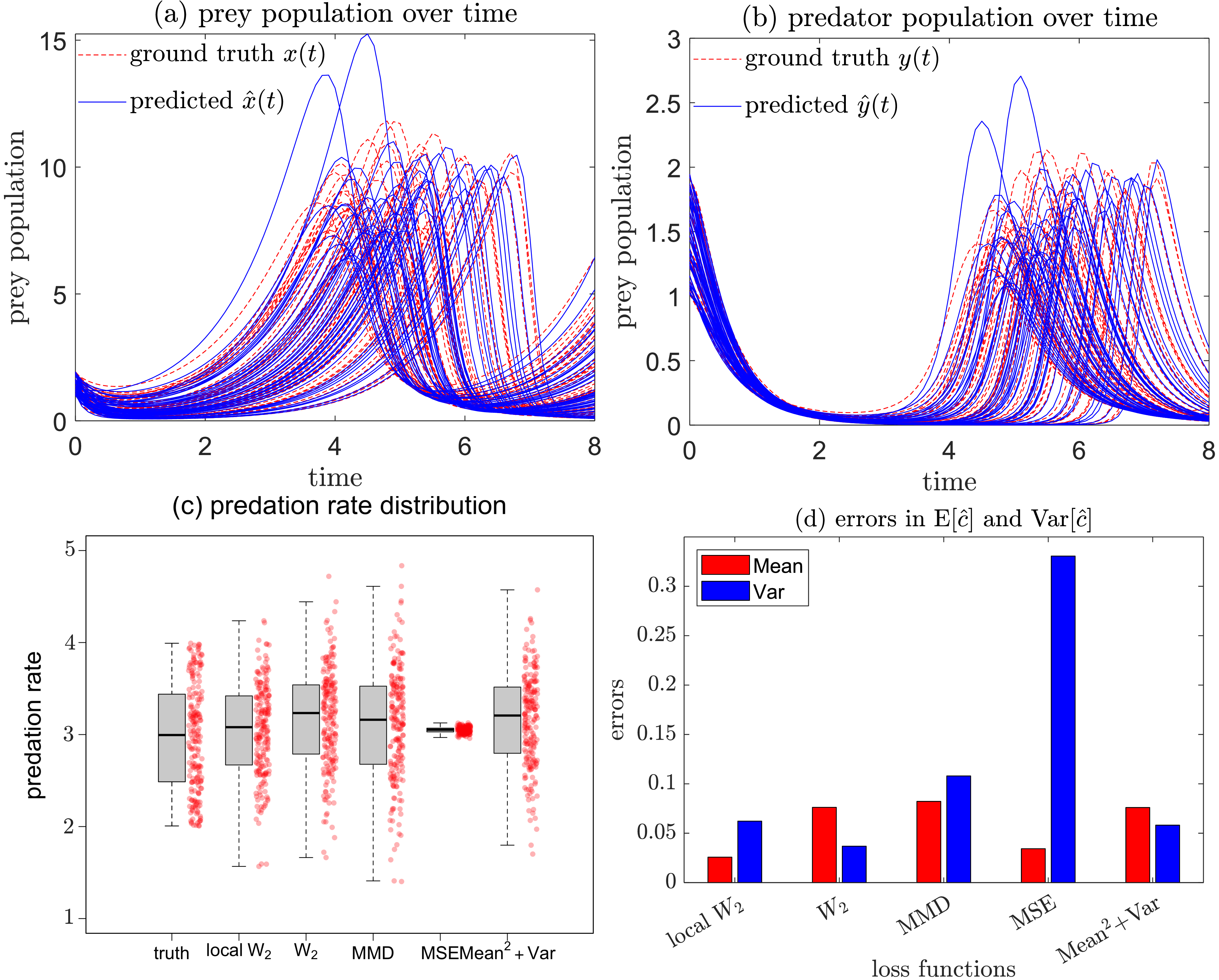}
\caption{\footnotesize(a) Ground truth (red dashed lines) prey population dynamics versus predicted prey population dynamics (blue solid lines) obtained with reconstructed predation rate $\hat{c}$. (b) Ground truth (red) predator population dynamics versus predicted predator population dynamics (blue) obtained with reconstructed predation rate $\hat{c}$. In (a) and (b), for clarity, we plot the first 50 groups of prey and predatory trajectories. \added{Since the predation rate $c$ in Eq.~\eqref{example1_model} is sampled independently for each realization of the model Eq.~\eqref{example1_model}, the ground truth trajectories also form a distribution.} (c) Ground truth $c\sim\mathcal{U}(2, 4)$ versus the distribution of the approximate $\hat{c}$ when minimizing different loss functions. The black horizontal line and the box indicate the median and the interquartile range of the ground truth or predicted predation rate.
(d) Errors in the predicted mean $|\E[\hat{c}] - \E[c]|$ and predicted variance $|\text{Var}[\hat{c}] - \text{Var}[c]|$ when minimizing different loss functions. The errors are their averaged values over 5 independent experiments. In (c) and (d), ``local $W_2$" refers to our local time-decoupled squared $W_2$ loss function Eq.~\eqref{time_coupling0} while ``$W_2$" refers to previous time-decoupled squared $W_2$ loss function in \cite{xia2024squared}.}
\label{fig:example1}
\end{figure}
%
%
%
%
%
{\tiny \begin{table}[h!]
\scriptsize
\centering
\caption{\footnotesize Computational time and memory usage when utilizing different loss functions. Mean and the standard deviation in the runtime and RAM usage over five repeated experiments are recorded.} 
\vspace{0.2in}
\begin{tabular}{llllll}
\toprule {Loss function} & \textbf{Ours (Eq.~\eqref{time_coupling0}}) & time-decoupled $W_2^2$ &
MMD & MSE & Mean+Var$^2$ \\
\midrule
Runtime (hour)& $\bm{1.53\pm0.35}$ & 1.57$\pm$0.36 & 1.91$\pm$0.47 & 1.37$\pm$0.46 & 1.73$\pm$0.57\\
RAM usage (Mb)& $\bm{729.8\pm 56.3}$ & 648.4$\pm$108.7 & 834.3$\pm$19.3 & 616.5$\pm$110.3 & 614.3$\pm$108.2\\
\bottomrule
\end{tabular}
\label{tab:computational_cost}
\end{table}}

%
%
%
%

%

\added{From Fig.\ref{fig:example1} (a) and (b), the distribution of ground truth trajectories of the prey and predator population 
can be matched well by the distribution of trajectories generated by numerically solving Eq.~\eqref{example1_model} with $\hat{c}$ sampled from the reconstructed distribution. Furthermore, the distribution of the reconstructed $\hat{c}$ generated by the SNN model trained using our loss function Eq.~\eqref{time_coupling0} aligns well with the distribution of the ground truth predation rate $c$ (shown in Fig.~\ref{fig:example1} (c)). Using the previous time-decoupled squared $W_2$ loss function in \cite{xia2024squared}, the MMD loss function, or the Mean$^2$+Var loss function leads to an inaccurate calculated value of the mean 
in the reconstructed $\hat{c}$. 
Using the MSE as the loss function to train the SNN yields an almost degenerate distribution of $\hat{c}$ and the calculated value of the variance in $\hat{c}$ is inaccurate, indicating that it not suitable to train the SNN model for reconstructing the distribution of unknown parameters.
Using our proposed local time-decoupled squared $W_2$ loss function yields more accurate values of the mean and variance of the reconstructed $\hat{c}$ compared to other methods (shown in Fig.~\ref{fig:example1} (d)). Finally, from Table~\ref{tab:computational_cost}, the runtime and RAM usage when using our proposed local time-decoupled squared $W_2$ loss function is generally comparable to those when using other benchmark loss functions. Thus, using our proposed local time-decoupled squared $W_2$ method is more efficient than using other benchmark statistical loss functions for training the SNN in Fig.~\ref{fig:nn_model} to reconstruct the distribution of the unknown predation rate in Eq.~\eqref{example1_model}.}

\end{example}

Next, we apply our local time-decoupled squared $W_2$ method for the reconstruction of model parameters in a spatiotemporal PDE. 

\deleted{Examples 1 and 2 are both ``hypothetical" to test our method with synthetic data.}
\begin{example}
\rm
\label{example2}
    We consider reconstructing the distributions of parameters in the following parabolic PDE:
\begin{equation}
\begin{aligned}
    &\partial_t u(x, t;c_1, c_2) = \frac{c_1}{c_2^2}\partial_{xx}u(x, t;c_1, c_2) +\frac{c_1}{c_1t+1}u(x, t;c_1, c_2),\,\,(x, t)\in \mathbb{R}\times[0, 2],\\
    &\quad\quad u(x, 0) = (1+\xi)\tfrac{x}{\sqrt{4}} \cdot \exp\big(-\tfrac{x^2}{ 2  }),\\
    &c_1\sim\mathcal{N}(0.5, \sigma_1^2),\,\, c_2=\tilde{\xi} + \beta(0.5-c_1), \,\,\xi\sim\mathcal{N}(0, \sigma_3^2),\,\, \tilde{\xi}\sim\mathcal{N}(1.5, \sigma_2^2).\\
    \end{aligned}
    \label{example2_model}
\end{equation}
We use another parabolic PDE model to approximate Eq.~\eqref{example2_model}:
\begin{equation}
\begin{aligned}
    &\partial_t \hat{u}(x, t;\hat{c}_1, \hat{c}_2) = \frac{\hat{c}_1}{\hat{c}_2^2}\partial_{xx}\hat{u}(x, t;\hat{c}_1, \hat{c}_2) +\frac{\hat{c}_1}{\hat{c}_1t+1}\hat{u}(x, t;\hat{c}_1, \hat{c}_2),\,\,(x, t)\in \mathbb{R}\times[0, 2],\\
    &\quad\quad \hat{u}(x, 0) = u(x, 0). 
    \end{aligned}
    \label{example2_model_approximate}
\end{equation}
We wish to reconstruct the distribution of $(c_1, c_2)$ in Eq.~\eqref{example2_model} using the distribution of $(\hat{c}_1, \hat{c}_2)$ in the approximate Eq.~\eqref{example2_model_approximate}.
We use a pseudo-spectral method with a spectral expansion in space to solve Eqs.~\eqref{example2_model} and~\eqref{example2_model_approximate} numerically:
\begin{equation}
\begin{aligned}
        u(x, t;\theta)\approx u_{n-1}(x,t;\theta)= \sum_{i=0}^{n-1} u_i(t;\theta)\hat{\mathcal{H}}_i(x),\,\,\theta\coloneqq (c_1, c_2)\\
    \hat{u}(x, t;\hat{\theta})\approx \hat{u}_{n-1}(x,t;\hat{\theta})= \sum_{i=0}^{n-1} \hat{u}_i(t;\hat{\theta})\hat{\mathcal{H}}_i(x),\,\,\hat{\theta}\coloneqq(\hat{c}_1, \hat{c}_2),
\end{aligned}
    \label{spectral_approx}
\end{equation}
where $\hat{\mathcal{H}}_i$ is the generalized Hermite function described in \cite{Shen2010}.

We carry out the following sensitivity tests:
\begin{enumerate}
    \item Vary $\sigma_1, \sigma_2$ which determines the variance in $c_1, c_2$ to investigate how variances in model parameters would affect the reconstruction accuracy of the joint distribution of $(c_1, c_2)$. Other parameters are set as $\beta=1, \sigma_3=0.2, n=12$ and $\delta=0.1$.
    \item Change the value of $\sigma_3$ characterizing the uncertainty in the initial condition and the size of the neighborhood $\delta$ in our loss function Eq.~\eqref{time_coupling0} to investigate how uncertainty in the initial condition and the hyperparameter $\delta$ affect the reconstruction of the distribution $(c_1, c_2)$. Other parameters are $\beta=1, \sigma_1=0.1,\sigma_2=0.2, n=12$. 
    \item Vary $\beta$ which determines the correlation between $c_1$ and $c_2$ as well as the expansion order $N$ in the spectral approximation in Eq.~\eqref{spectral_approx} to explore how the correlation between $c_1$ and $c_2$ and the dimensionality of the discretized ODE affect the reconstruction accuracy of $(c_1, c_2)$. Other parameters are $\sigma_1=0.1,\sigma_2=0.2, \sigma_3 = 0.2, \delta=0.1$.
\end{enumerate}

\begin{figure}[H]
\centering
\includegraphics[width=\textwidth]{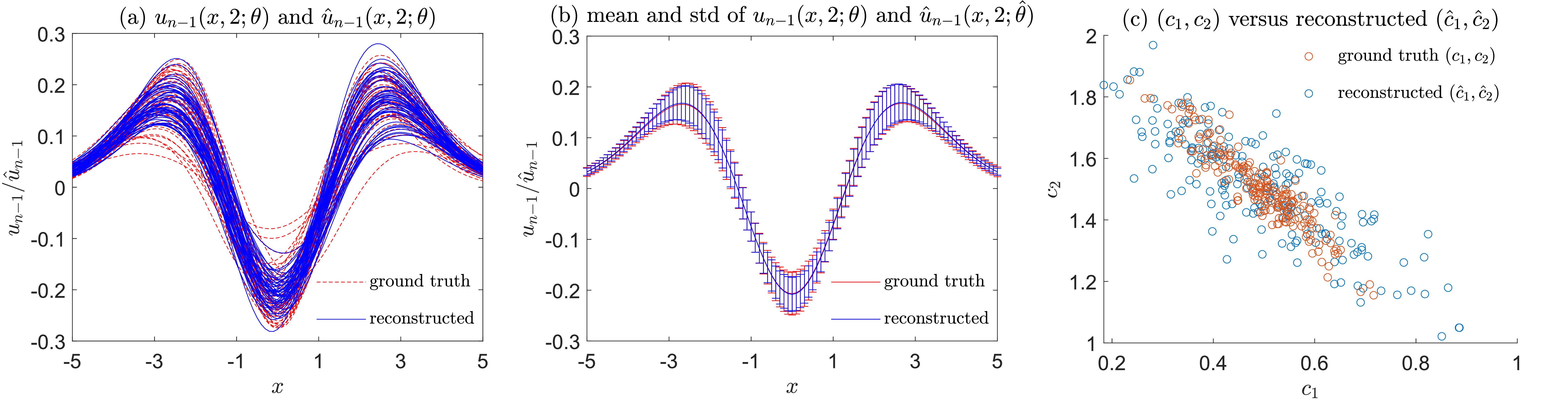}
\includegraphics[width=\textwidth]{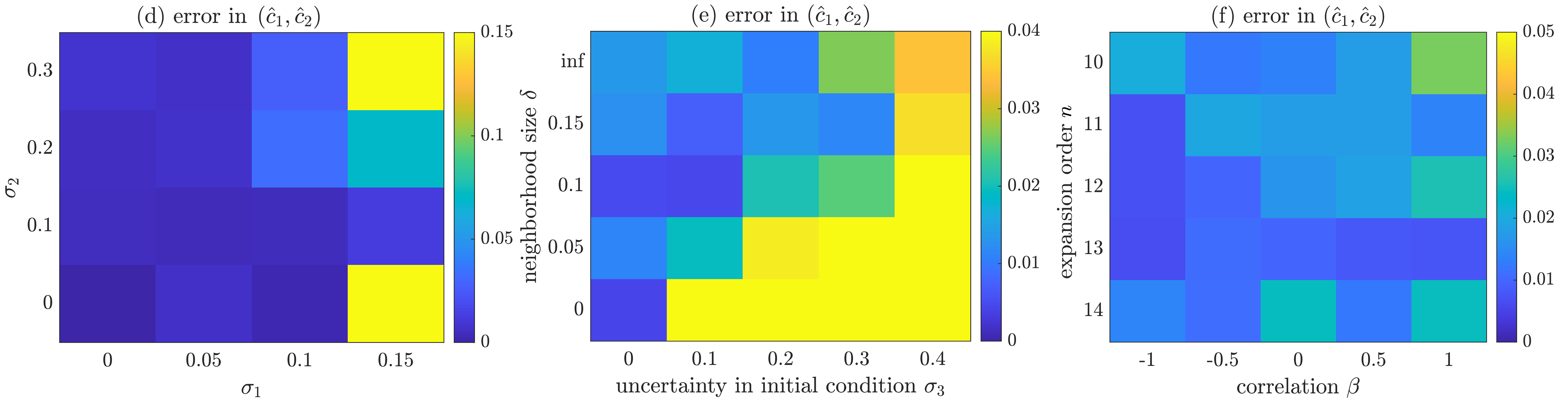}
\caption{\footnotesize(a) Ground truth $u_{n-1}(x, 2;\theta)$ (red dashed lines) versus reconstructed $\hat{u}_{n-1}(x, 2;\hat{\theta})$ (blue solid lines). For clarity, we only plot 50 ground truth $u_{n-1}(x,2;\theta)$ numerical solutions versus 50 approximate numerical solutions $\hat{u}_{n-1}(x,2;\hat{\theta})$ in Eq.~\eqref{spectral_approx}. (b) Mean and standard deviations of the ground truth $u_{N-1}(x, 2)$ versus reconstructed $\hat{u}_{n-1}(x, 2)$. (c) The ground truth $(c_1, c_2)$ versus reconstructed $(\hat{c}_1, \hat{c}_2)$ when $\beta=1, \sigma_1=0.15, \sigma_2=0.1, n=12$. In (a), (b), and (c), the parameters are $n=12$ and $\beta=1, \sigma_1=0.15, \sigma_2=0.1, \sigma_3=0.2, N=12$ and $\delta=0.1$ in the loss function Eq.~\eqref{time_coupling0}. (d) Errors in $(\hat{c}_1, \hat{c}_2)$ w.r.t. different variances $\sigma_1, \sigma_2$ for $(c_1, c_2)$ in Eq.~\eqref{example2_model} (Case 1 on Page 23).  (e) Errors in $(\hat{c}_1, \hat{c}_2)$ w.r.t. different values of the variance $\sigma_3$ in the initial condition $\epsilon$ and different $\delta$. $\delta=\text{inf}$ indicates that we set $\delta=\infty$, which corresponds to the time-decoupled squared $W_2$ loss function in \cite{xia2024efficient} (Case 2 on Page 23).  (f) Errors in $(\hat{c}_1, \hat{c}_2)$ w.r.t. different values $N$ and $c_0$ (Case 3 on Page 23).}
\label{fig:example2}
\end{figure}

We plot the reconstructed numerical solution $\hat{u}_{n-1}(x,t; \hat{\theta})$ versus the ground truth numerical solution $u_{n-1}(x, t;\theta)$ in Eq.~\eqref{spectral_approx} as the numerical solution to Eq.~\eqref{example2_model} in Fig.~\ref{fig:example2} (a), (b). The distribution of ground truth numerical solutions can be matched well by the distribution of the numerical solutions generated with the reconstructed $(\hat{c}_1, \hat{c}_2)$, shown in Fig.~\ref{fig:example2} (c).
From Fig.~\ref{fig:example2} (d), the larger the variance in $c_1$, the larger the errors in the predicted distribution of $(\hat{c}_1, \hat{c}_2)$. When the variance $c_1$ is too large, we might have close-to-zero or even negative $c_1$ in Eq.~\eqref{example2_model}, which makes numerically solving Eq.~\eqref{example2_model} ill-posed and yields a poor reconstruction of the joint distribution of $(c_1, c_2)$. \added{Thus, it is reasonable to reconstruct the distribution of model parameters for well-posed dynamical systems instead of ill-posed dynamical systems.} From Fig.~\ref{fig:example2} (e), when the variance $\sigma_3^2$ in the initial condition is small and the initial conditions are more densely distributed, a smaller $\delta$ leads to a more accurate reconstruction of $(c_1, c_2)$. Thus, it is necessary to consider uncertainties in the initial condition of an ODE or PDE for the accurate reconstruction of unknown model parameters and find the size of the neighborhood $\delta$ that is compatible with uncertainty in the initial state.
Finally, from  Fig.~\ref{fig:example2} (f), the error in the reconstructed distribution of $(\hat{c}_1, \hat{c}_2)$ is \added{independent of the} dimensionality $n$ of the discretized ODE system.
Furthermore, the accuracy of the reconstructed $(\hat{c}_1, \hat{c}_2)$ is insensitive to the correlation between the two uncertain model parameters $(c_1, c_2)$, indicating the robustness of our proposed local temporally decoupled squared $W_2$ method for reconstructing the distribution of $(c_1, c_2)$ in Eq.~\eqref{example2_model}.

\end{example}

\added{As an application of our proposed method for parameter inference problems in biophysics, we reconstruct reaction rates in an 8D ODE system characterizing drug dynamics in an ocular model studied in \cite{crawshaw2024modelling}.}

\begin{example}
\rm
\label{example3}

\added{We consider an ODE system which arises from modeling studies of monoclonal antibodies (MABs) injected into the vitreous gel of the eye in the treatment of wet age-related macular degeneration. The monoclonal antibodies bind to the vascular endothelial growth factor (VEGF) with this binding inhibiting the latter’s stimulation of pathological capillary growth through the retina \cite{chappelow2008neovascular, mitchell2018age}. However, injected MABs are cleared from the eye, by passing into the aqueous
compartment at the front of the eye. This in turn necessitates multiple
injections of MABs into the eye. Hence, analyzing how long MABs are retained within the eye and the prospect of reducing injection frequency has
motivated numerous modeling studies (e.g. \cite{caruso2019ocular, lamirande2024first, crawshaw2024modelling, hutton2016mechanistic, hutton2018theoretical}). However, a
common theme within these studies is the need to estimate the parameters \cite{ mitchell2018age, hutton2016mechanistic, crawshaw2024modelling}.}

\added{Hence, we consider a simple well-mixed model in the literature for the
MAB ranibizumab \cite{crawshaw2024modelling, hutton2016mechanistic}, with its concentration in the
vitreous of the eye denoted by $r_{\text{vit}}$, and the vitreous VEGF concentration
denoted by $v_{\text{vit}}$. A complex of one MAB bound to one VEGF in the vitreous
has concentration $c_{\text{vit}}$ and, noting that VEGF has two binding sites, a
complex of 2 MABs and a VEGF is also considered, with concentration $h_{\text{vit}}$.
The formation and disassociation of these species is represented by the law
of mass action, with the reaction rates labeled by ``$k$", and all species can pass
into the aqueous at the front of the eye, with concentrations in this
compartment then distinguished by the subscript ``$\text{aq}$" ($r_{\text{aq}}, v_{\text{aq}}, c_{\text{aq}}, h_{\text{aq}}$). Once in this latter
compartment, all species are taken to pass into the bloodstream via
Schlemm’s canal, with a clearance rate $CL$. In addition, there is a flux,  $V_{\text{in}}$, of vitreal VEGF to reflect the rate of elevated levels of VEGF leaking into
the vitreous in pathology. This generates the set of eight ODEs:}

\textbf{Vitreous}
\begin{equation}
    \begin{aligned}
        &\frac{\d v_{\text{vit}}}{\d t} = (k_{\text{off}} c_{\text{vit}} - 2k_{\text{on}} v_{\text{vit}} r_{\text{vit}}) - k_v^{\text{el}} v_{\text{vit}} + \frac{V_{\text{in}}}{V_{\text{vit}}},\\
        &\frac{\d r_{\text{vit}}}{\d t} = (k_{\text{off}} c_{\text{vit}} - 2k_{\text{on}} v_{\text{vit}} r_{\text{vit}}) + (2k_{\text{off}} h_{\text{vit}} - k_{\text{on}} r_{\text{vit}} c_{\text{vit}}) - k_r^{\text{el}} r_{\text{vit}},\\
        &\frac{\d c_{\text{vit}}}{\d t} = -(k_{\text{off}} c_{\text{vit}} - 2k_{\text{on}} v_{\text{vit}} r_{\text{vit}}) + (2k_{\text{off}} h_{\text{vit}} - k_{\text{on}} r_{\text{vit}} c_{\text{vit}}) - k_c^{\text{el}} c_{\text{vit}},\\
        &\frac{\d h_{\text{vit}}}{\d t} = -(2k_{\text{off}} h_{\text{vit}} - k_{\text{on}} r_{\text{vit}} c_{\text{vit}}) - k_h^{\text{el}} h_{\text{vit}}.
    \end{aligned}
    \label{vitreous}
\end{equation}

\textbf{Aqueous}
\begin{equation}
    \begin{aligned}
        &\frac{\d v_{\text{aq}}}{\d t} = (k_{\text{off}} c_{\text{aq}} - 2k_{\text{on}} v_{\text{aq}} r_{\text{aq}}) + \frac{V_{\text{vit}}}{V_{\text{aq}}} k_v^{\text{el}} v_{\text{vit}} - \frac{CL}{V_{\text{aq}}} v_{\text{aq}},\\
        &\frac{\d r_{\text{aq}}}{\d t} = (k_{\text{off}} c_{\text{aq}} - 2k_{\text{on}} v_{\text{aq}} r_{\text{aq}}) + (2k_{\text{off}} h_{\text{aq}} - k_{\text{on}} r_{\text{aq}} c_{\text{aq}}) + \frac{V_{\text{vit}}}{V_{\text{aq}}} k_r^{\text{el}} r_{\text{vit}} - \frac{CL}{V_{\text{aq}}} r_{\text{aq}},\\
        &\frac{\d c_{\text{aq}}}{\d t} = -(k_{\text{off}} c_{\text{aq}} - 2k_{\text{on}} v_{\text{aq}} r_{\text{aq}}) + (2k_{\text{off}} h_{\text{aq}} - k_{\text{on}} r_{\text{aq}} c_{\text{aq}}) + \frac{V_{\text{vit}}}{V_{\text{aq}}} k_c^{\text{el}} c_{\text{vit}} - \frac{CL}{V_{\text{aq}}} c_{\text{aq}},\\
        &\frac{\d h_{\text{aq}}}{\d t} = -(2k_{\text{off}} h_{\text{aq}} - k_{\text{on}} r_{\text{aq}} c_{\text{aq}}) + \frac{V_{\text{vit}}}{V_{\text{aq}}} k_h^{\text{el}} h_{\text{vit}} - \frac{CL}{V_{\text{aq}}} h_{\text{aq}}, \,\, t\in[0, 2].
    \end{aligned}
    \label{aqueous}
\end{equation}

\deleted{In Eqs.~\eqref{vitreous} and~\eqref{aqueous}, $v_{\text{vit,aq}}$, $r_{\text{vit,aq}}$, $c_{\text{vit,aq}}$, and $h_{\text{vit,aq}}$ denote the concentrations of free VEGF, unbound ranibizumab, the VEGF-ranibizumab complex, and the ranibizumab-VEGF-ranibizumab complex in the vitreous (vit) and the aqueous (aq) humor, respectively.} 

From \cite{crawshaw2024modelling}, we set the vitreous and aqueous humor volumes  $V_{\text{vit}}=2.05\text{mL}$
and $V_{\text{aq}}=0.105\text{mL}$, respectively, and   $V_{\text{in}}=5.408\text{pmol}\cdot\text{day}^{-1}$. \added{We aim to reconstruct the distribution of the seven reaction rate parameters: $k_{\text{off}}, k_{\text{on}}, k_v^{\text{el}}, k_r^{\text{el}}, k_c^{\text{el}}$,  $k_h^{\text{el}}$ and $CL$ in Eqs.~\eqref{vitreous} and~\eqref{aqueous}, which are subject to uncertainties in the drug properties \cite{mitchell2018age, crawshaw2024modelling, hutton2016mechanistic}.} 

\added{We generate a synthetic data set of model parameters by sampling the seven kinetic parameters: $\bm{k}\coloneqq(k_{\text{off}}, k_{\text{on}}, k_v^{\text{el}}, k_r^{\text{el}}, k_c^{\text{el}}, k_h^{\text{el}}, CL)$ from the following model:}
\begin{equation}
    \bm{k} = \bm{k}_0 + c\bm{k}_0 * A\tilde{\bm{k}},
    \label{kinetic_model}
\end{equation}
where 
\begin{equation}
\begin{aligned}
&    \bm{k}_0 = (1.669\text{day}^{-1}, 0.00114\text{pM}^{-1}\cdot\text{day}^{-1}, 0.575\text{day}^{-1}, 0.293\text{day}^{-1},\\
&\hspace{2cm} 0.259\text{day}^{-1}, 0.176\text{day}^{-1}, 2.505 \text{mL}\cdot\text{day}^{-1})
    \end{aligned}
\end{equation} 
is the vector of mean values of those kinetic parameters used in \cite{crawshaw2024modelling}. In Eq.~\eqref{kinetic_model}, 
$*$ is the Hadamard componentwise product. $A\in\mathbb{R}^{7\times7}$ is a randomly generated matrix whose components are sampled independently from the distribution $\mathcal{U}(-\frac{1}{2}, \frac{1}{2})$. In Eq.~\eqref{kinetic_model}, $\tilde{\bm{k}}\coloneqq (k_1,k_2, ..., k_7)$ is independently generated for each trajectory in the training dataset. Omitting the units for simplicity, we sample $k_1, k_2\sim\mathcal{U}(0, 1)$, $k_3, k_4\sim\mathcal{N}(0, 0.5^2)$, $k_5\sim \text{Exp}(2)$, $k_6 \sim \text{B}(2, 5)$ (the Beta distribution with shape parameters $\alpha=2, \beta=5$), and $k_7\sim \Gamma(2, 2)$ (the Gamma distribution with a shape parameter $\alpha=2$ and scale parameter $\lambda=2$). The initial condition of each trajectory is independently sampled from $\mathcal{N}(\bm{I}_7, 0.05^2 I_{7\times7})$, where $\bm{I}_7\in\mathbb{R}^7$ refers to a constant vector whose components are 1. 


\begin{figure}[h!]
\centering
\includegraphics[width=\linewidth]{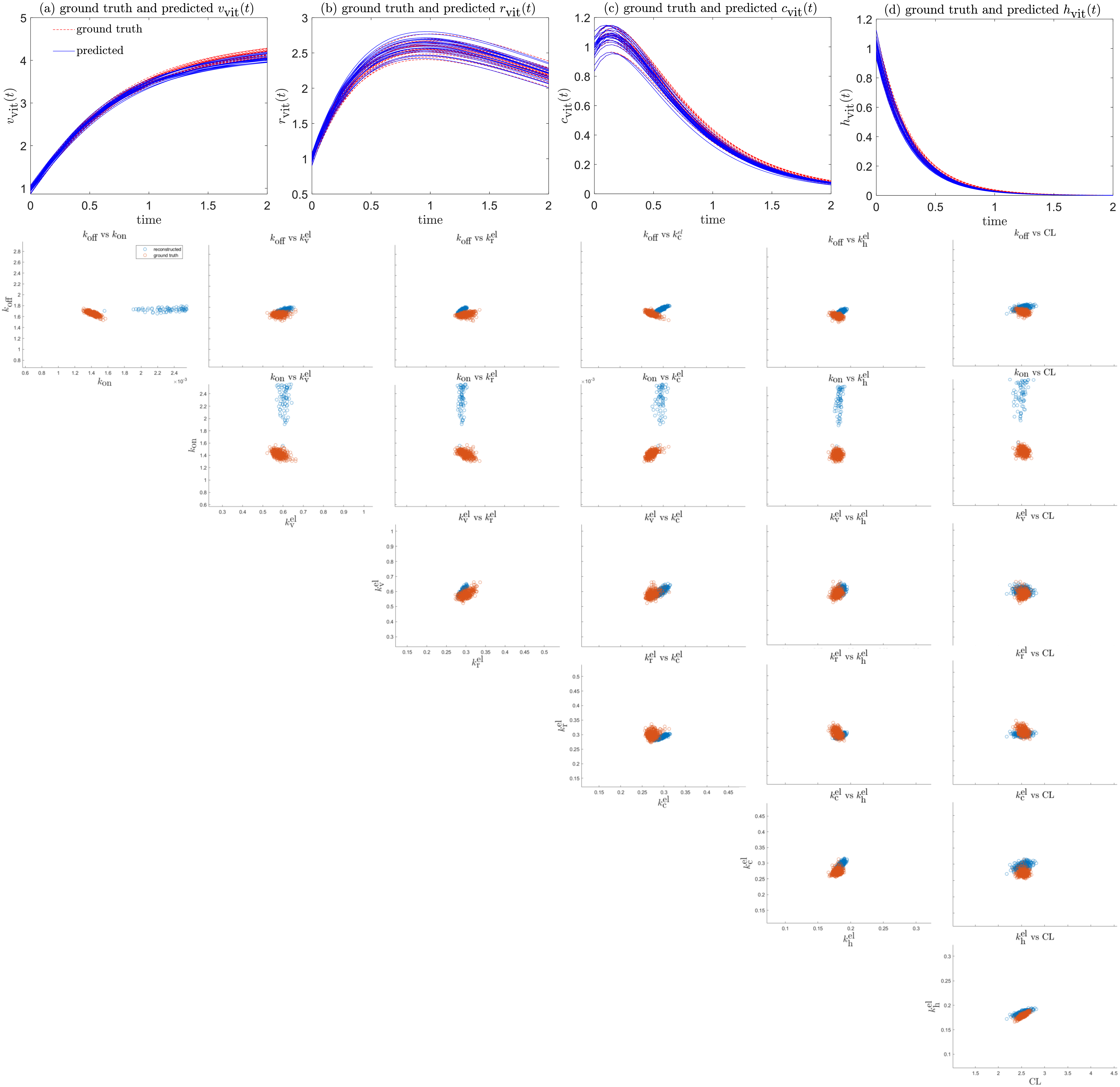}
\vspace{-0.2in}
\caption{\footnotesize(a)-(d) The first 50 out of 400 trajectories of the four quantities $v_{\text{vit}}(t), r_{\text{vit}}(t), c_{\text{vit}}(t), h_{\text{vit}}(t)$ obtained with the parameter vector $\bm{k}$ sampled from the ground truth distribution Eq.~\eqref{kinetic_model} (red dashed lines) versus trajectories obtained by using the parameter vector $\bm{k}$ sampled from the reconstructed distribution generated by the trained SNN (blue solid lines). The SNN has 3 hidden layers and 10 neurons in each layer. ResNet is used for forward propagation, and the nodes and weights are initialized by independently sampling from $\mathcal{N}(0, 0.03^2)$.
(e) The reconstructed joint distribution of any two kinetic parameters in Eq.~\eqref{kinetic_model}. In all subplots, the red dots are sampled from the ground truth joint distribution while the blue dots are sampled from the reconstructed distribution.}
\label{fig:ocular}
\end{figure}

The distribution of trajectories of the four quantities $v_{\text{vit}}(t), r_{\text{vit}}(t), c_{\text{vit}}(t)$ and $h_{\text{vit}}(t)$ obtained with the parameter vector $\bm{k}$ sampled from the ground truth distribution Eq.~\eqref{kinetic_model} can be matched well by the distribution of trajectories obtained by using the parameter vector $\bm{k}$ sampled from the reconstructed distribution generated by the trained SNN in Fig.~\ref{fig:ocular} (a)-(d). Additionally, we plot the empirical joint distribution of any two parameters in $\bm{k}$ versus the empirical joint reconstructed distribution of the corresponding two variables in Fig.~\ref{fig:ocular} (e), and most pairwise joint distributions \added{of any two components in $\bm{k}$} can be matched well by their reconstructed counterparts. However, the reconstruction of $k_{\text{on}}$ is not accurate. A possible reason could be that the magnitude of $k_{\text{on}}$ ($O(10^{-3})$) is much smaller than that of other parameters ($O(1)$), making it harder to reconstruct the distribution of $k_{\text{on}}$. Another possibility is a lack of practical
identifiability for these parameters. \added{Since only observed trajectories are available, 
the reconstruction of the distribution of model parameters using our method signifies a ``worst scenario" with no prior information on the magnitudes of model parameters.} 
\deleted{Thus, it would be worth further investigation to explore if inputting prior knowledge on the model parameters could help more accurately reconstruct the joint distribution of model parameters that span across different magnitudes.}

We also investigate how the number of neurons in each layer, the number of hidden layers in the SNN model (Fig.~\ref{fig:nn_model}), the initialization for the distribution of weights in the SNN model, as well as whether adopting the ResNet technique \cite{he2016deep} for forward propagation would affect the accuracy of the reconstructed distribution of the kinetic parameters. From Table.~\ref{tab:nn_structure}, SNNs with more than one hidden layers and 10 neurons in each layer can all reconstruct the distribution of $\bm{k}$ in Eq.~\eqref{kinetic_model} well. Thus, we do not need a wide or deep SNN model in Fig.~\ref{fig:nn_model} to reconstruct the distribution of $\bm{k}$.
Also, applying the ResNet technique can moderately increase the reconstruction accuracy when the number of hidden layers increases. Finally, initializing the weights and biases to small nonzero values is important, as the reconstruction error is huge if we set all weights and biases to zero. It is worth further investigation on how to optimally design the structure of the SNN model in Fig.~\ref{fig:nn_model}, \added{which is beyond the scope of this paper}. 

\begin{table}[h!]
\scriptsize
\centering
  \caption{Errors in the reconstructed distribution of kinetic parameters in the ocular pharmacokinetic model. Here, ReLU refers to using the ReLU activation function for forward propagation and ResNet refers to using the ResNet technique (described in Fig.~\ref{fig:nn_model}).}
\begin{tabular}{ccccc}
\toprule  width & \# of layers & forward propagation & initialization for weights \& biases & error   \\ 
\midrule 
  10 & 1 & ReLU & $\mathcal{N}(0, 0.03^2)$&  $8.73\times10^{-3}$
 \\ 
  10 & 2 & ReLU & $\mathcal{N}(0, 0.03^2)$& $1.46\times10^{-3}$
  \\ 
10 & 3 & ReLU & $\mathcal{N}(0, 0.03^2)$& $1.13\times10^{-3}$  \\ 
10 & 4 & ReLU & $\mathcal{N}(0, 0.03^2)$& $1.27\times10^{-3}$  \\ 
5 & 3 & ReLU & $\mathcal{N}(0, 0.03^2)$& $1.05\times10^{-3}$ \\ 
15 & 3 & ReLU & $\mathcal{N}(0, 0.03^2)$& $8.50\times10^{-4}$  \\ 
 20 & 3 & ReLU & $\mathcal{N}(0, 0.03^2)$ & $1.15\times10^{-3}$  \\
   10 & 2 & ResNet & $\mathcal{N}(0, 0.03^2)$& $7.68\times10^{-3}$  \\ 
10 & 3 & ResNet & $\mathcal{N}(0, 0.03^2)$& $9.23\times10^{-4}$  \\ 
10 & 4 & ResNet & $\mathcal{N}(0, 0.03^2)$& $9.49\times10^{-4}$ \\ 
 10 & 3 & ReLU & $\mathcal{N}(0, 0)$& 0.609
 \\ 
  10 & 3 & ReLU & $\mathcal{N}(0, 0.01^2)$& $7.89\times10^{-4}$
 \\ 
  10 & 3 & ReLU & $\mathcal{N}(0, 0.02^2)$& $7.54\times10^{-4}$
 \\ 
\bottomrule
\end{tabular}
\label{tab:nn_structure}
\end{table}

\deleted{ResNet is generally used in deep neural networks, and we are not using a deep NN.}
\end{example}

Finally, we consider reconstructing the distribution of parameters of a jump-diffusion process, in which intrinsic stochasticity from the Wiener process and Poisson process as well as uncertainty in the initial condition of the jump-diffusion process are coupled with uncertainty in the parameters governing the jump-diffusion processes.

\begin{example}
\label{example4}
\rm
 We reconstruct the distribution of uncertain parameters governing a jump-diffusion process that can be used to describe the
posited stock returns \cite{xia2024efficient, merton1976option}. Instead of considering a deterministic jump magnitude function as in \cite[Example 2]{xia2024efficient}, we consider the following jump-diffusion process:
\begin{equation}
\begin{aligned}
  &\d X_t = 0.05 \d t + s\sqrt{|X_t|} \d B_t
  + \int_{U} \xi X_t\d \tilde{N}(\gamma(\d\xi)\d{t}),\,\,\,  t\in[0, 2],\\
  &\quad\quad s\sim \sigma_0\mathcal{N}(1, 1),\,\, \xi\sim\mathcal{N}(\beta_0, \sigma_1^2),\,\, X_0\sim \mathcal{N}(2, \sigma_2^2).
  \end{aligned}
 \label{example4_model}
\end{equation}
\added{In Eq.~\eqref{example4_model}, the drift term represents the risk-free interest rate, the diffusion term stands for the fluctuation in the stock price, and the jump term represents events of paying dividends.}
In Eq.~\eqref{example4_model}, $\tilde{N}$ is a compensated Poisson process defined in Eq.~\eqref{compensated_poisson}. We aim at reconstructing the distributions of $s$ as well as $\xi$ using two separate SNNs using the distribution of $\hat{s}_0$ and $\hat{\xi}$ in the following approximate jump-diffusion process:
\begin{equation}
\begin{aligned}
  &\d \hat{X}_t = 0.05 \d t + \hat{s}\sqrt{|\hat{X}_t|} \d B_t
  + \int_{U} \hat{\xi} \hat{X}_t\d \hat{N}(\gamma(\d\xi)\d{t}),\,\, \hat{X}_0=X_0, \,\,  t\in[0, 2].
  \end{aligned}
 \label{example4_model_approximate}
\end{equation}
In Eq.~\eqref{example4_model_approximate}, $\hat{N}$ is another compensated Poisson process that is independent of $\tilde{N}$.
\added{It has been revealed in \cite{xia2024efficient} that when $\xi\equiv1$, larger values of $s$ and $\beta_0$ make it more difficult to reconstruct the jump-diffusion process \textit{i.e.}, relative errors in the learned diffusion and jump functions get larger} because the trajectories are more sparsely distributed.
Here, we carry out further experiments on the following cases:
\begin{enumerate}
    \item Change the values of $\beta_0$ as well as the value of $\sigma_1$ in Eq.~\eqref{example4_model} to explore how the mean and variance of the jump magnitude function affect the reconstruction of distributions of $|s|$ and $\xi$ (here we reconstruct the distribution of $|s|$ because $\hat{s}\sqrt{|\hat{X}_t|} \d B_t$ and $|\hat{s}|\sqrt{|\hat{X}_t|} \d B_t$ are identically distributed). Other parameters are $\sigma_0=0.3, \sigma_2=0.1$, and $\delta=0.1$ (the neighborhood size in the loss function Eq.~\eqref{time_coupling0}).
    \item Vary the value of $\sigma_0$ as well as the value of $\sigma_1$ to investigate how the uncertainty in the diffusion function and uncertainty in the jump magnitude affect the reconstruction of distributions of $|s|$ and $\xi$. Other parameters are $\beta_0=0.3, \sigma_2=0.1$, and $\delta=0.1$. 
    \item Vary the value of $\sigma_1$ and $\sigma_2$ to determine how uncertainty in the initial condition and jump magnitude affects the reconstruction of $|s|$ and $\xi$. Other parameters are $\sigma_0=0.3, \beta_0=0.3, \sigma_2=0.1$, and $\delta=0.1$.
\end{enumerate}

\begin{figure}[H]
\centering
\includegraphics[width=\textwidth]{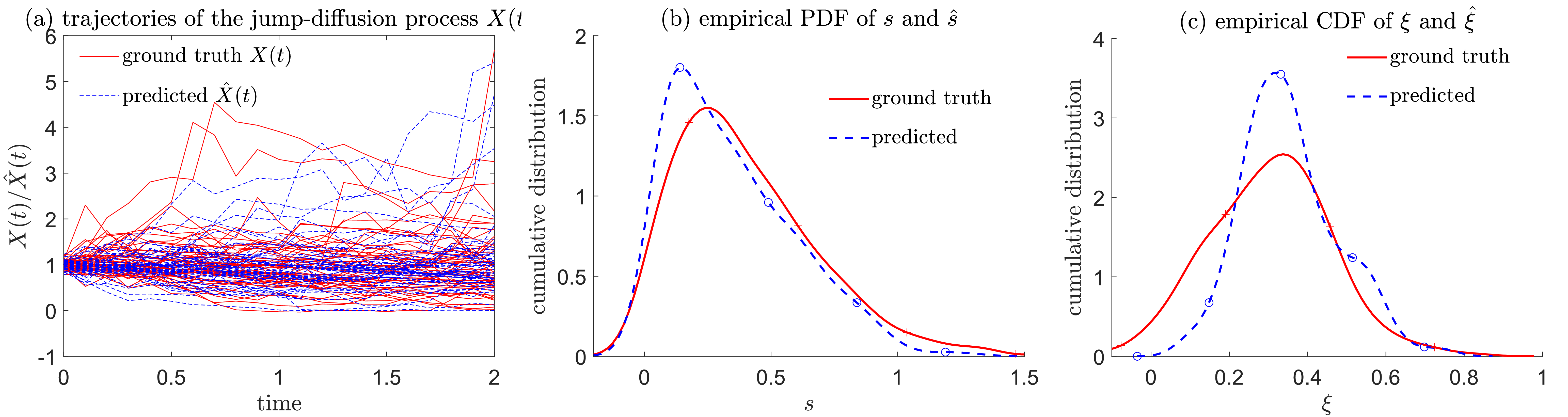}
\includegraphics[width=\textwidth]{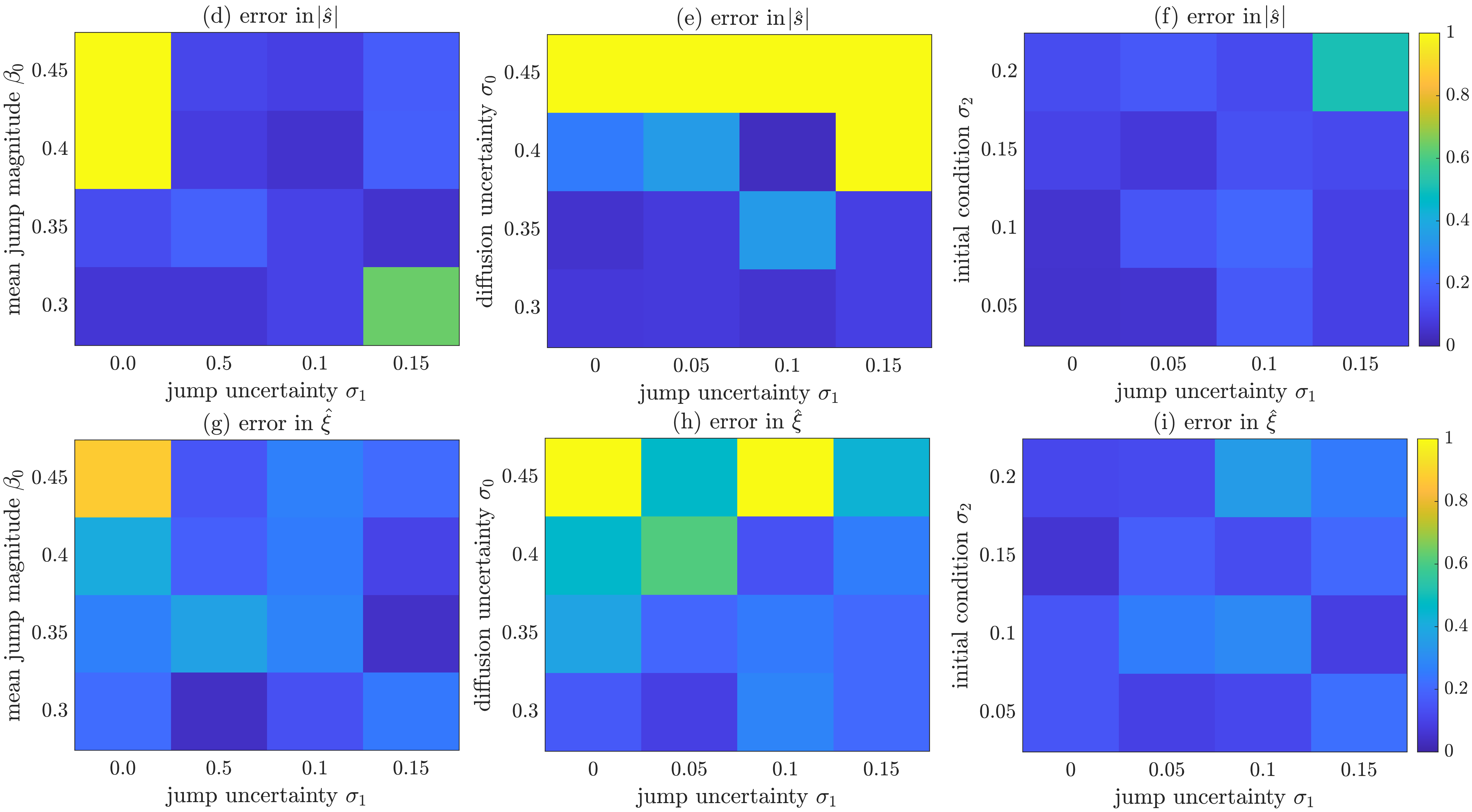}
\caption{\footnotesize(a) Ground truth trajectories generated from Eq.~\eqref{example4_model} versus the reconstructed trajectories generated from the approximate Eq.~\eqref{example4_model_approximate}. For clarity, we plot 50 ground truth trajectories and 50 reconstructed trajectories. (b) The empirical probability density function ground truth $|s|$ versus the empirical probability density function of the reconstructed $|\hat{s}|$ in Eqs.~\eqref{example4_model} and ~\eqref{example4_model_approximate}. (c) The empirical probability density function ground truth $\xi$ versus the empirical probability density function of the reconstructed $|\hat{\xi}|$ in Eqs.~\eqref{example4_model} and ~\eqref{example4_model_approximate}. In (a)-(c), $\sigma_0=0.3, \beta_0 = 0.35, \sigma_1=0.15, \sigma_2=0.1$, and $\delta=0.1$ in the loss function Eq.~\eqref{time_coupling0}. (d) and (g) The errors in the reconstructed distribution of $\hat{\sigma}$ and $\xi$ for case 1 on Page 30, respectively. 
(e) and (h) The errors in the reconstructed distribution of $\hat{\sigma}$ and $\xi$ for case 2 on Page 30, respectively. (f) and (i) The errors in the reconstructed distribution of $\hat{\sigma}$ and $\xi$ for case 3 on Page 30, respectively. }
\label{fig:example4}
\end{figure}

\added{From Fig.~\ref{fig:example4} (a), the distribution of reconstructed trajectories of the approximate jump-diffusion process Eq.~\eqref{example4_model_approximate} match well with the distribution of ground truth jump-diffusion trajectories at each time $t$. Also, from  Fig.~\ref{fig:example4} (b), (c), the probability density functions of $|s|$ and $\xi$ can be matched well by the probability density functions of $|\hat{s}|$ and $\hat{\xi}$ in Eq.~\eqref{example4_model_approximate}, respectively.} From Fig.~\ref{fig:example4} (d), (g), when the mean and variance of $\xi$ in the jump function of Eq.~\eqref{example4_model} become large, the reconstruction of the distributions of $s$ and $\xi$ is also less accurate. The increase in the mean $\beta_0$ of the jump magnitude impacts the reconstruction accuracy more than the increase in the variance $\sigma_1$ does. From Fig.~\ref{fig:example4} (e), (h), a larger $\sigma_0$ characterizing greater uncertainty in the diffusion function of the jump-diffusion process Eq.~\eqref{example4_model} also leads to less accurate reconstructions of both $s$ and $\xi$. Finally, from 
Fig.~\ref{fig:example4} (f), (i), choosing a neighborhood size $\delta=0.1$ works well for $\sigma_2\in[0.05, 0.2]$ characterizing different levels of uncertainty in the initial condition and errors in both $|\hat{s}|$ and $\hat{\xi}$ are well controlled. 


\added{When the form of the diffusion function in the jump-diffusion process in Eq.~\eqref{example4_model} is unknown but the diffusion function itself is deterministic, we can use a deterministic parameterized neural network to reconstruct the diffusion function as was done in \cite{xia2024efficient} while simultaneously using an SNN to reconstruct the distribution of $\xi$ in the jump function of Eq.~\eqref{example4_model}. When the diffusion function is deterministic, the reconstruction accuracy of both the diffusion function and the distribution of $\xi$ in Eq.~\eqref{example4_model} is good. The results are shown in \ref{example4_appendix}.} 

\end{example}

 \section{Summary and conclusions}
 \label{summary}

In this paper, we proposed and analyzed a local time-decoupled squared $W_2$ distance method for reconstructing the distributions of parameters in specific dynamical systems from time-series data using an SNN model (Fig.~\ref{fig:nn_model}). We also analyzed the SNN model and proved that it could approximate any continuous random variable as long as moderate assumptions are satisfied, making it a suitable model for reconstructing the distribution of parameters in a dynamical system. 
Our method took advantage of a previous local squared $W_2$ method and a previous time-decoupled squared $W_2$ method so that both uncertainty in the initial state and intrinsic fluctuations such as the Wiener process in the dynamics could be considered.
We showcased the effectiveness of our approach for reconstructing the distribution of model parameters in several dynamical systems such as ODEs, PDEs, and SDEs. Our method outperformed several other benchmark statistical methods.

\added{One limitation of our proposed method is that when the model parameters span across different magnitudes, the reconstruction of the distribution of model parameters whose magnitudes are smaller is less accurate.
In practice, prior information on the distribution of model parameters might be available either from biophysical estimates or experimental assays.
Thus, how to incorporate prior information of the model parameter distribution, such as confining the range of model parameters,
is a potential future research direction. Specifically, it is illuminating to make comparisons with Bayesian approaches or Monte-Carlo simulation approaches, both of which would require a prior distribution of model parameters, in terms of both accuracy and computational complexity.} Also, it is helpful to consider applying our method to quantify uncertainties in model selection from time-series data or spatiotemporal data \cite{nardini2020learning}. Taking into account measurement errors might also be necessary when such errors are not negligible \cite{nardini2019influence}.
Additionally, it will be promising to apply our proposed local time-decoupled squared $W_2$ method for reconstructing the distribution of uncertain parameters in more complicated dynamical systems such as multidimensional SPDEs. It will be useful to analyze other stochastic neural networks' ability to approximate the distribution of uncertain model parameters, especially model parameters that take discrete values. \added{On the other hand, more theoretical analysis on the Wasserstein-distance-based loss function, such as proving that minimizing a Wasserstein-distance-type loss function is sufficient for reconstructing the distribution of parameters in a dynamical system, can also be interesting, which would require analysis on identifiability of model parameters, \textit{i.e.}, whether two different sets of model parameters would lead to different trajectories.} Finally, more theoretical and empirical studies on how to design the optimal architecture, \textit{e.g.} the depth and width, of the SNN model in Fig.~\ref{fig:nn_model} will be informative.

\section*{CRediT authorship contribution statement}
\textbf{Mingtao Xia}: Writing – review \& editing, Writing – original draft, Visualization, Validation, Supervision, Software, Project administration, Methodology, Investigation, Formal analysis, Conceptualization.
 \textbf{Qijing Shen}: Writing – review \& editing, Visualization, Software, Investigation, Methodology.
\textbf{Philip Maini}: Writing – review \& editing, Resources, Investigation.
\textbf{Eamonn Gaffney}: Writing – review \& editing, Resources, Investigation.
\textbf{Alex Mogilner}:  Writing – review \& editing, Resources, Investigation.

\section*{Declaration of competing interest}
The authors declare that they have no known competing financial interests or personal relationships that could influence the work in this paper.

\section*{Acknowledgement}
This work was supported in part through the NYU IT High Performance Computing resources, services, and staff expertise. The authors also acknowledge Dr. Jessica R. Crawshaw for helpful discussions.

\section*{Code and data availability}
All codes will be published on GitHub and the link to the repository will be provided upon acceptance of this manuscript. In this study, no data was created or used. 

\newpage 
 \appendix

 \section{The proof of Theorem~\ref{well_posedness}}

\label{appendixA}
Here, we prove Theorem~\ref{well_posedness}. Notice that 
    \begin{equation}
        W_{2, \delta}^{2, \text{e}}(\bm{X}(t_i), \hat{\bm{X}}(t_i))  = \frac{1}{N} \sum_{j=1}^N W_{2}^{2}(\nu_{\bm{X}_{0, j}, \delta}^{\text{e}}(t_i), \hat{\nu}_{\bm{X}_{0, j}, \delta}^{\text{e}}(t_i)), 
        \label{definition_j}
    \end{equation} 
    where $N$ is the number of observed trajectories and $\bm{X}_{0, j}$ denotes the initial condition of the $j^{\text{th}}$ trajectory. Suppose $0=t_0^1<t_1^1<...<t_{n_1}^1=T$;
$0=t_0^2<t_1^2<...<t_{n_2}^2=T$ are sets of grids on $[0,
  T]$. We define a third set of grids $0=t_0^3<...<t_{n_3}^3=T$ such
that $\{t_0^1,...,t_{n_1}^1\}\cup \{t_0^2,...,t_{n_2}^2\} =
\{t_0^3,...,t_{n_3}^3\}$. Let $\Delta
t\coloneqq\max\{\max_i(t_{i+1}^1-t_i^1), \max_j(t_{j+1}^2-t_j^2),
\max_k(t_{k+1}^3-t_k^3)\}$. We denote $\nu^{\text{e}}_{\bm{X}_{0, j}, \delta}(t)$ and $\hat{\nu}^{\text{e}}_{\bm{X}_{0, j}, \delta}(t)$ to be the empirical conditional probability distributions of $\bm{X}(t)$ and $\hat{\bm{X}}(t)$ at time $t$ conditioned on $|\bm{X}(0)-\bm{X}_{0, j}|\leq\delta$ and $|\hat{\bm{X}}(0)-\bm{X}_{0, j}|\leq\delta$, respectively. 
We shall show
\begin{equation}
\bigg|\sum_{i=0}^{n_1-1}W_{2, \delta}^{2, \text{e}}\big(\bm{X}(t_i^1), \hat{\bm{X}}(t_i^1)\big)(t_{i+1}^1-t_i^1) -
\sum_{i=0}^{n_3-1}W_{2, \delta}^{2, \text{e}}\big(\bm{X}(t_i^3), \hat{\bm{X}}(t_i^3)\big)(t_{i+1}^3-t_i^3)\bigg| \rightarrow 0,
\label{limit_exist}
\end{equation}
as $\Delta t\rightarrow 0$.

First, suppose in the interval $(t_i^1, t_{i+1}^1)$, we have
$t_i^1=t_{\ell}^3<t_{\ell+1}^3<...<t_{\ell+s}^3=t_{i+1}^1, s\geq 1$.
For $s>1$, since
$t_{i+1}^1-t_i^1=\sum_{k=\ell}^{\ell+s-1}(t_{k+1}^3-t_k^3)$, we have

\begin{equation}
  \begin{aligned}
    &\bigg|W_{2, \delta}^{2, \text{e}}\big(\bm{X}(t_i^1), \hat{\bm{X}}(t_i^1)\big) (t_{i+1}^1-t_i^1)  -
    \sum_{k=\ell}^{\ell+s-1} W_{2, \delta}^{2, \text{e}}\big(\bm{X}(t_k^3), \hat{\bm{X}}(t_k^3)\big)(t_{k+1}^3-t_k^3)\bigg| \\
    \: & \leq\frac{1}{N}\sum_{j=1}^N\sum_{k=\ell+1}^{\ell+s-1}\Big(W_{2}\big(\nu^{\text{e}}_{\bm{X}_{0, j}, \delta}(t_i^1),
    \hat{\nu}^{\text{e}}_{\bm{X}_{0, j}, \delta}(t_i^1)\big) + W_{2}\big(\nu^{\text{e}}_{\bm{X}_{0, j}, \delta}(t_i^3), \hat{\nu}^{\text{e}}_{\bm{X}_{0, j}, \delta}(t_k^3)\big|\Big)  \\
    \: & \qquad\quad \times \Big|W_{2}\big(\nu^{\text{e}}_{\bm{X}_{0, j}, \delta}(t_i^1), \hat{\nu}^{\text{e}}_{\bm{X}_{0, j}, \delta}(t_i^1)\big)
    - W_{2}\big(\nu^{\text{e}}_{\bm{X}_{0, j}, \delta}(t_k^3), \hat{\nu}^{\text{e}}_{\bm{X}_{0, j}, \delta}(t_k^3)\big)\Big| (t_{k+1}^3 - t_k^3).
    \end{aligned}
\label{triang}
\end{equation} 
Since $\|\bm{X}\|$ and $\|\hat{\bm{X}}\|$ are uniformly bounded, we have:
\begin{equation}
    W_{2}\big(\nu^{\text{e}}_{\bm{X}_{0, j}, \delta}(t_i^1), \hat{\nu}^{\text{e}}_{\bm{X}_{0, j}, \delta}(t_i^1)\big)\leq \sup_{\bm{X}_0} \E[\|\bm{X}(t)\|^2]^{\frac{1}{2}} + \E[\|\hat{\bm{X}}(t)\|^2]^{\frac{1}{2}} \leq X + \hat{X}
    \label{Mcondition}
\end{equation}
and 
\begin{equation}
     W_{2}\big(\nu^{\text{e}}_{\bm{X}_{0, j}, \delta}(t_k^3), \hat{\nu}^{\text{e}}_{\bm{X}_{0, j}, \delta}(t_k^3)\big)\leq X + \hat{X},
\end{equation}
where $X, \hat{X}$ are the upper bounds in Eq.~\eqref{thm2_1_upper}.
We can take a specific coupling measure $\pi^*_{\delta, \bm{X}_{0, j}}\big(\bm{X}(t_i^1), \bm{X}(t_k^3)\big)$ 
such that if we regard
\begin{equation}
    (\bm{X}(t_i^1), \bm{X}(t_k^3)):\mathbb{R}^{2d}\rightarrow\mathbb{R}^{2d}
\end{equation}
as a mapping from initial state space $(\bm{X}_0, \tilde{\bm{X}}_0)\in\mathbb{R}^{2d}$ to the solutions $(\bm{X}(t_i^1), \tilde{\bm{X}}(t_k^3))\in \mathbb{R}^{2d}$ at times $(t_i^1, t_k^3)$ satisfying $\bm{X}(t_i^1)=\bm{X}_0, \tilde{\bm{X}}(t_k^3)=\tilde{\bm{X}}_0$, $\pi^*_{\delta, \bm{X}_{0, j}}$ is the pushforward measure of $\nu^{\text{e}}_{\bm{X}_{0, j}, \delta}\cdot\delta_{\tilde{\bm{X}}_0 = \bm{X}_0}$, where $\nu^{\text{e}}_{\bm{X}_{0, j}, \delta}$ is the empirical conditional distribution of the initial condition $\bm{X}_0$ conditioned on $|\bm{X}_{0, j}-\bm{X}_0|\leq \delta$. Then, for any $\bm{X}_{0, j}$, we have
\begin{equation}
  \begin{aligned}
    W_{2}^{2}(\nu_{\bm{X}_{0, j}}^{\text{e}}(t_i^1), \nu_{\bm{X}_{0, j}}^{\text{e}}(t_k^3)) \leq &
    \sup_{\bm{X}_{0, j}}\E_{\big(\bm{X}(t_i^1), \bm{X}(t_k^3)\big)\sim\pi^*_{\delta, \bm{X}_{0, j}}}\big[\|\bm{X}(t_k^3) - \bm{X}(t_i^1)\|_2^2\big] \\
    \leq & \sup_{\bm{X}_{0, j}}\E\bigg[\medint\int_{t_i^1}^{t_{i+1}^1}\!
      \sum_{i=1}^d f_i^2(\bm{X}(t),t;\theta) \d t \bigg] (t_{i+1}^1-t_i^1).
\end{aligned}
\label{bound1}
\end{equation}

Similarly, we have
\begin{equation}
  \begin{aligned}
    W_{2}^{2}\big(\hat{\nu}_{\bm{X}_{0, j}}^{\text{e}}(t_i^1), \hat{\nu}_{\bm{X}_{0, j}}^{\text{e}}(t_k^3)\big) \leq &
    \,\sup_{\bm{X}_{0, j}}\E\bigg[\medint \int_{t_i}^{t_{i+1}} \sum_{\ell=1}^d
      f_{\ell}^2(\bm{X}(t),t; \hat{\theta}) \d t \bigg](t_{i+1}^1-t_i^1).
\end{aligned}
    \label{bound2}
  \end{equation}

Using the triangular inequality of the Wasserstein distance \cite{clement2008elementary}, we have
\begin{equation}
\begin{aligned}
 &   \Big|W_{2}^{}\big(\nu^{\text{e}}_{\bm{X}_{0, j}, \delta}(t_i^1), \hat{\nu}^{\text{e}}_{\bm{X}_{0, j}, \delta}(t_i^1)\big)
    - W_{2}\big(\nu^{\text{e}}_{\bm{X}_{0, j}, \delta}(t_k^3), \hat{\nu}^{\text{e}}_{\bm{X}_{0, j}, \delta}(t_k^3)\big)\Big| 
\\&\quad\quad\leq \Big|W_2\big(\nu^{\text{e}}_{\bm{X}_{0, j}, \delta}(t_i^1), \hat{\nu}^{\text{e}}_{\bm{X}_{0, j}, \delta}(t_i^1)\big)- W_2\big(\nu^{\text{e}}_{\bm{X}_{0, j}, \delta}(t_k^3),
\hat{\nu}^{\text{e}}_{\bm{X}_{0, j}, \delta}(t_k^1)\big)\Big| \\
&\quad\quad\quad\quad+ \Big|W_2\big(\nu^{\text{e}}_{\bm{X}_{0, j}, \delta}(t_i^3), \hat{\nu}^{\text{e}}_{\bm{X}_{0, j}, \delta}(t_i^1)\big)- W_2\big(\nu^{\text{e}}_{\bm{X}_{0, j}, \delta}(t_k^3),
\hat{\nu}^{\text{e}}_{\bm{X}_{0, j}, \delta}(t_k^3)\big)\Big|\\
&\quad\quad\leq W_2\big(\nu^{\text{e}}_{\bm{X}_{0, j}, \delta}(t_i^1), \nu^{\text{e}}_{\bm{X}_{0, j}, \delta}(t_k^3)\big) +
W_2\big(\hat{\nu}^{\text{e}}_{\bm{X}_{0, j}, \delta}(t_i^1), \hat{\nu}^{\text{e}}_{\bm{X}_{0, j}, \delta}(t_k^3)\big).
\end{aligned}
\label{traing_property}
\end{equation}
Substituting Eqs.~\eqref{Mcondition},~\eqref{bound1}, \eqref{bound2}, and
\eqref{traing_property} into Eq.~\eqref{triang}, we conclude that

\begin{equation}
  \begin{aligned}
    \bigg|W_2^2\big(\nu_{\bm{X}_{0, j}, \delta}^{\text{e}}(t_i^1), \hat{\nu}_{\bm{X}_{0, j}, \delta}^{\text{e}}(t_i^1)\big)& (t_{i+1}^1-t_i^1) 
    \\&\hspace{-2cm}- \sum_{k=\ell}^{\ell+s-1} W_2^2\big((\nu_{\bm{X}_{0, j}, \delta}^{\text{e}}(t_k^3),
    \hat{\nu}_{\bm{X}_{0, j}, \delta}^{\text{e}}(t_k^3)\big)(t_{k+1}^3-t_k^3)\bigg| \\
    \: & \quad \leq 2(X+\hat{X})(t_{i+1}^1-t_i^1)\big(\sqrt{F_i\Delta t}
    + \sqrt{\hat{F}_i \Delta t}\big),
    \end{aligned}
\label{intermediate}
\end{equation}
where  
\begin{equation}
    F_i\coloneqq \sup_{\bm{X}_{0, j}}\E\bigg[\medint \int_{t_i^1}^{t_{i+1}^1} \sum_{\ell=1}^d
      f_{\ell}^2(\bm{X}(t),t; \theta) \d t \bigg],\,\, \hat{F}_i\coloneqq \sup_{\bm{X}_{0, j}}\E\bigg[\medint \int_{t_i^1}^{t_{i+1}^1} \sum_{\ell=1}^d
      f_{\ell}^2(\bm{X}(t),t; \hat{\theta}) \d t \bigg].
\end{equation}
Summing over $i$ and $j$ for the inequality~\eqref{intermediate}, we have:
\begin{equation}
  \begin{aligned}
    &\bigg|\sum_{i=0}^{n_1-1}W_{2, \delta}^{2, \text{e}}\big(\bm{X}(t_i^1), \hat{\bm{X}}(t_i^1)\big) (t_{i+1}^1-t_i^1)  -
    \sum_{i=0}^{n_3-1} W_{2, \delta}^{2, \text{e}}\big(\bm{X}(t_k^3), \hat{\bm{X}}(t_k^3)\big)(t_{k+1}^3-t_k^3)\bigg|\\ 
    &\quad\quad\leq  2(X + \hat{X}) T\max_i\big(\sqrt{F_i\Delta t}
    + \sqrt{\hat{F}_i\Delta t}\big).
  \end{aligned}
  \end{equation}
Similarly, 

\begin{equation}
  \begin{aligned}
    &\bigg|\sum_{i=0}^{n_2-1}W_{2, \delta}^{2, \text{e}}\big(\bm{X}(t_i^2), \hat{\bm{X}}(t_i^2)\big) (t_{i+1}^2-t_i^2)  -
    \sum_{i=0}^{n_2-1} W_{2, \delta}^{2, \text{e}}\big(\bm{X}(t_k^3), \hat{\bm{X}}(t_k^3)\big)(t_{k+1}^3-t_k^3)\bigg| \\
    &\quad\quad\leq  2(X+ \hat{X}) T\max_i\Big(\sqrt{F'_i\Delta t}
    +\!\sqrt{\hat{F}'_i\Delta t}\,\Big),
  \end{aligned}
  \end{equation}
  where 
  \begin{equation}
    F_i'\coloneqq \sup_{\bm{X}_{0, j}}\E\bigg[\medint \int_{t_i^2}^{t_{i+1}^2} \sum_{\ell=1}^d
      f_{\ell}^2(\bm{X}(t),t; \hat{\theta}) \d t \bigg],\,\, \hat{F}_i'\coloneqq \sup_{\bm{X}_{0, j}}\E\bigg[\medint \int_{t_i^2}^{t_{i+1}^2} \sum_{\ell=1}^d
      f_{\ell}^2(\bm{X}(t),t; \hat{\theta}) \d t \bigg].
\end{equation}
Thus, as $\Delta t\rightarrow 0$,

\begin{equation}
  \bigg|\sum_{i=0}^{n_1-1}W_{2, \delta}^{2, \text{e}}\big(\nu(t_i^1), \hat{\nu}(t_i^1)\big)(t_{i+1}^1-t_i^1)
  - \sum_{i=0}^{n_2-1}W_{2, \delta}^{2, \text{e}}\big(\nu(t_i^2),
  \hat{\nu}(t_i^2)\big)(t_{i+1}^2-t_i^2)\bigg|\rightarrow 0,
\label{convergence}
\end{equation}
which implies the limit
\begin{equation}
  \lim\limits_{\max(t_{i+1}^1-t_{i}^1)\rightarrow 0}\sum_{i=0}^{N-1}W_{2, \delta}^{2, \text{e}}\big(\nu(t_i^1),
  \hat{\nu}(t_i^1)\big)(t_i^1-t_{i-1}^1)
  \label{limit_exists}
\end{equation}
exists. Therefore, the local time-decoupled squared $W_2$ distance in Eq.~\eqref{local_define}:
\begin{equation}
    \tilde{W}_{2, \delta}^{2, \text{e}}(\bm{X}, \hat{\bm{X}})\coloneqq \int_0^T {W}_{2, \delta}^{2, \text{e}}(\bm{X}(t), \hat{\bm{X}}(t))\d t
\end{equation}
is well-defined.

 
\section{Proof of Corollary~\ref{jump_diffusion_thm}}
\label{appendixB}
The proof of Corollary~\ref{jump_diffusion_thm} is similar to the proof of Theorem~\ref{well_posedness} and the proof of Theorem 3.1 in \cite{xia2024efficient}. Suppose $0=t_0^1<t_1^1<...<t_{n_1}^1=T$;
$0=t_0^2<t_1^2<...<t_{n_2}^2=T$ are two sets of grids on $[0,
  T]$. We define a third set of grids $0=t_0^3<...<t_{n_3}^3=T$ such
that $\{t_0^1,...,t_{n_1}^1\}\cup \{t_0^2,...,t_{N_2}^2\} =
\{t_0^3,...,t_{n_3}^3\}$. Let $\Delta
t\coloneqq\max\{\max_i(t_{i+1}^1-t_i^1), \max_j(t_{j+1}^2-t_j^2),
\max_k(t_{k+1}^3-t_k^3)\}$. We denote $\nu^{\text{e}}_{\bm{X}_{0, j}, \delta}(t)$ and $\hat{\nu}^{\text{e}}_{\bm{X}_{0, j}, \delta}(t)$ to be the empirical conditional probability distributions of $\bm{X}(t)$ and $\hat{\bm{X}}(t)$ at time $t$ conditioned on $|\bm{X}(0)-\bm{X}_{0, j}|\leq\delta$ and $|\hat{\bm{X}}(0)-\bm{X}_{0, j}|\leq\delta$, respectively. We need to show that
    \begin{equation}
\bigg|\sum_{i=0}^{n_1-1}W_{2, \delta}^{2, \text{e}}\big(\bm{X}(t_i^1),
\hat{\bm{X}}(t_i^1)\big)(t_{i+1}^1-t_i^1) -
\sum_{i=0}^{n_3-1}W_{2, \delta}^{2, \text{e}}\big(\bm{X}(t_i^3),
\hat{\bm{X}}(t_i^3)\big)(t_{i+1}^3-t_i^3)\bigg| \rightarrow 0.
\end{equation}

For $j=1, 2$, we define:
  \begin{equation}
  \begin{aligned}
    F_i^j\coloneqq & \sup_{\bm{X}_0, \theta}\E\bigg[\medint\int_{t_i^j}^{t_{i+1}^j}
      \sum_{\ell=1}^d f_{\ell}^2(\bm{X}(t^-),t^-; \theta)\d t\bigg], \\&\,\,
    \hat{F}_i^j\coloneqq\sup_{\bm{X}_0, \hat{\theta}}\E\bigg[\medint\int_{t_i^j}^{t_{i+1}^j} \sum_{\ell=1}^d
f_{\ell}^2(\hat{\bm{X}}(t^-),t^-; \hat{\theta})\d t\bigg], \\
\Sigma_i^j\coloneqq & \sup_{\bm{X}_0, \theta}\E\bigg[\medint\int_{t_i^j}^{t_{i+1}^j} \sum_{\ell=1}^d
  \sum_{j=1}^m\sigma_{\ell, j}^2(\bm{X}(t^-),t^-)\d t\bigg], \\\,\,&
\hat{\Sigma}_i^j\coloneqq\sup_{\bm{X}_0, \hat{\theta}}\E\bigg[\medint\int_{t_i^j}^{t_{i+1}^j}
\sum_{\ell=1}^d\sum_{j=1}^m\hat{\sigma}_{\ell, j}^2(\hat{\bm{X}}(t^-),t^-; \theta)\d t\bigg],\\
B_i^j\coloneqq & \sup_{\bm{X}_0, \theta}\E\bigg[\medint\int_{t_i^j}^{t_{i+1}^j} \sum_{\ell=1}^d
\medint\int_U \beta_{\ell}^2(\bm{X}(t^-),\xi, t^-;\theta)\nu(\d\xi)\d t\bigg], \\&\,\, 
\hat{B}_i^j\coloneqq\sup_{\bm{X}_0, \hat{\theta}}\E\bigg[\medint\int_{t_i^j}^{t_{i+1}^j}
  \sum_{\ell=1}^d\medint\int_U \hat{\beta}_{\ell}^2(\hat{\bm{X}}(t^-),\xi, t^-;\hat{\theta})
  \nu(\d\xi)\d t\bigg].
    \end{aligned}
\end{equation}

Similar to the proof of Theorem~\ref{well_posedness}, we find that:
\begin{equation}
  \begin{aligned}
    \bigg|W_2^2\big(\nu_{\bm{X}_{0, j}, \delta}^{\text{e}}(t_i^1), \hat{\nu}_{\bm{X}_{0, j}, \delta}^{\text{e}}(t_i^1)\big)& (t_{i+1}^1-t_i^1) 
    - \sum_{k=\ell}^{\ell+s-1} W_2^2\big(\nu_{\bm{X}_{0, j}, \delta}^{\text{e}}(t_k^3),
    \hat{\nu}_{\bm{X}_{0, j}, \delta}^{\text{e}}(t_k^3)\big)(t_{k+1}^3-t_k^3)\bigg| \\
    \: &\hspace{-2cm} \hspace{-0.4in} \leq 2(X+\hat{X})(t_{i+1}^1-t_i^1)\big(\sqrt{F_i^1\Delta t + \Sigma_i^1 + B_i^1}
    + \sqrt{\hat{F}_i^1 \Delta t + \hat{\Sigma}_i^1 +\hat{B}_i^1}\big).
    \end{aligned}
\label{intermediate_1}
\end{equation}

Summing over $i=0,...,n_1-1$ and $j=1,...,N$, we can obtain
\begin{equation}
  \begin{aligned}
    \bigg|\sum_{i=0}^{n_1-1}W_{2, \delta}^{2, \text{e}}\big(\bm{X}(t_i^1), \hat{\bm{X}}(t_i^1)) & (t_{i+1}^1-t_i^1) -
    \sum_{k=0}^{n_3-1}W_{2, \delta}^{2, \text{e}}\big(\bm{X}(t_k^3), \hat{\bm{X}}(t_k^3))(t_{k+1}^3-t_k^3)\bigg| \\
    &\hspace{-2cm}\leq  2(X+\hat{X}) T\max_i\big(\sqrt{F_i^1\Delta t+ \Sigma_i^1 +B_i^1}
    + \sqrt{\hat{F}_i^1\Delta t+ \hat{\Sigma}_i^1 +\hat{B}_i^1}\big).
  \end{aligned}
  \end{equation}
Similarly, we have:
\begin{equation}
  \begin{aligned}
    \bigg|\sum_{i=0}^{n_2-1}W_{2, \delta}^{2, \text{e}}\big(\bm{X}(t_i^2), \hat{\bm{X}}(t_i^2)) & (t_{i+1}^2-t_i^2) -
    \sum_{k=0}^{n_3-1}W_{2, \delta}^{2, \text{e}}\big(\bm{X}(t_k^3), \hat{\bm{X}}(t_k^3))(t_{k+1}^3-t_k^3)\bigg| \\
    &\hspace{-2cm}\leq  2(X+\hat{X}) T\max_i\big(\sqrt{F_i^2\Delta t+ \Sigma^2_i +B_i^2}
    + \sqrt{\hat{F}^2_i\Delta t+ \hat{\Sigma}^2_i +\hat{B}^2_i}\big).
  \end{aligned}
  \label{intermediate_2}
  \end{equation}

  Thus, 
  \begin{equation}
  \begin{aligned}
    \bigg|\sum_{i=0}^{n_1-1}W_{2, \delta}^{2, \text{e}}\big(\bm{X}(t_i^1), \hat{\bm{X}}(t_i^1)) & (t_{i+1}^1-t_i^1) -
    \sum_{i=0}^{n_2-1}W_{2, \delta}^{2, \text{e}}\big(\bm{X}(t_i^3), \hat{\bm{X}}(t_i^2))(t_{i+1}^3-t_i^3)\bigg| \\
    &\hspace{-2cm}\leq  2(X+\hat{X}) T\max_i\big(\sqrt{F_i^1\Delta t+ \Sigma_i^1 +B_i^1}
    + \sqrt{\hat{F}_i^1\Delta t+ \hat{\Sigma}_i^1 +\hat{B}_i^1}\big)\\
    &\hspace{-2cm}+2(X+\hat{X}) T\max_i\big(\sqrt{F_i^2\Delta t+ \Sigma^2_i +B_i^2}
    + \sqrt{\hat{F}^2_i\Delta t+ \hat{\Sigma}^2_i +\hat{B}^2_i}\big).
  \end{aligned}
  \end{equation}
  Since $\bm{f}, \bm{\sigma}, \bm{\beta}$ are uniformly bounded, $F_i^j, \Sigma_i^j, B_i^j, \hat{F}_i^j, \hat{\Sigma}_i^j, \hat{B}_i^j\rightarrow 0$ as $\Delta t\rightarrow 0$ uniformly for $j=1, 2$. Thus, the limit 
  \begin{equation}
      \tilde{W}_{2, \delta}^{2, \text{e}}(\bm{X}, \hat{\bm{X}})\coloneqq \int_0^T W_{2, \delta}^{2, \text{e}}(\bm{X}(t), \hat{\bm{X}}(t))\d t
  \end{equation}
  exists. 

  \section{The proof of Theorem~\ref{theorem1}}
  \label{AppendixC}
  In this section, we prove Theorem~\ref{theorem1}.
  First, consider $\tilde{\bm{X}}$ that solves the following jump-diffusion process:
\begin{equation}
\begin{aligned}
            &\d \tilde{\bm{X}}(t) = \bm{f}(\tilde{\bm{X}}(t), t; \hat{\theta})\d t + \bm{\sigma}(\tilde{\bm{X}}(t),
    t; \hat{\theta})\d \bm{B}_t + \int_U\bm{\beta}(\tilde{\bm{X}}(t), \xi, t; \hat{\theta})\tilde{N}( \d
    t, \nu(\d\xi)), \\
    &\hspace{1cm}\tilde{\bm{X}}(0) = \bm{X}(0).
\end{aligned}
    \label{tilde_X}
\end{equation}
Using Theorem 2.1 in \cite{xia2024efficient}, we have:
\begin{equation}
            \E\Big[\big\|\bm{X}(t) -
              \tilde{\bm{X}}(t)\big\|^{2}\Big]\leq
            \E\big[H(t)|\bm{X}(0)\big]
            \exp\Big(\big(2C+1+(2C +1)m
            +(2C+1)\gamma(U)\big)td\Big),
            \label{pleq2}
\end{equation}
where $\bm{X}(0)$ is the initial condition and $H(t)$ is defined as
\begin{equation}
\begin{aligned}
  H(t) \coloneqq & \E\bigg[\sum_{i=1}^d\int_0^t\big(f_{i}(\bm{X}(s^-), s^-; \theta)
    - f_{i}(\bm{X}(s^-), s^-;\hat{\theta})\big)^2\d s\bigg]\\
\: & \quad + \E\bigg[\sum_{i=1}^d\int_0^t\sum_{j=1}^m
  \big(\sigma_{i, j}(\bm{X}(s^-), s^-;\theta)
  - \sigma_{i, j}(\bm{X}(s^-), s^-;\hat{\theta})\big)^2\d s\bigg]\\
\: & \qquad + \E\bigg[\sum_{i=1}^d\int_0^t\int_U\big(\beta_{i}(\bm{X}(s^-), \xi, s^-;\theta)
  - \beta_{i}(\bm{X}(s^-), \xi, s^-;\hat{\theta})\big)^2\gamma(\d\xi)\d s\bigg]\\
  &\leq Ct(1 + m + \gamma(U)) \|\theta - \hat{\theta}\|^2.
\end{aligned}
\label{h_define}
\end{equation}
Since the probability distribution of $\hat{\bm{X}}(t)\in\mathbb{R}^d$ is the same as the probability distribution of $\tilde{\bm{X}}(t)$ for any $t\in[0, T]$, we conclude that
\begin{equation}
    W_2^2(\nu_{\bm{X}_0}(t), \hat{\nu}_{\bm{X}_0}(t)) = W_2^2(\nu_{\bm{X}_0}(t), \tilde{\nu}_{\bm{X}_0}(t)),
\end{equation}    
where $\nu_{\bm{X}_0}(t), \hat{\nu}_{\bm{X}_0}(t)$, and $\tilde{\nu}_{\bm{X}_0}(t)$ are the probability distributions of $\bm{X}(t), \hat{\bm{X}}(t),$ and $\tilde{\bm{X}}(t)$ given the same initial condition $\bm{X}_0$, respectively. Given any coupled distribution of $\theta, \hat{\theta}$ denoted by $\pi(\mu, \hat{\mu})$ such that its marginal distributions coincide with the probability distributions of $\theta$ and $\hat{\theta}$, we 
denote $\pi^*(\bm{X}(t), \bm{\tilde{X}}(t))$ to be its pushforward probability measure for $(\bm{X}(t), \tilde{\bm{X}}(t))\in\mathbb{R}^{2d}$. We can easily check that the marginal distributions of $\pi^*$ coincide with $\nu_{\bm{X}_0}(t)$ and $\tilde{\nu}_{\bm{X}_0}(t)$, respectively.
From Eqs.~\eqref{pleq2} and \eqref{h_define}, we have
\begin{equation}
\begin{aligned}
        \E_{(\bm{X}(t), \tilde{\bm{X}}(t))\sim\pi^*}\Big[\big\|\bm{X}(t) -
              \tilde{\bm{X}}(t)\big\|^{2}\Big]&\leq Ct(1+m + \gamma(U))\E_{(\theta, \hat{\theta})\sim \pi(\mu, \hat{\mu})}\big[\|\theta - \hat{\theta}\|^2\big]\\
              &\hspace{-2cm}\times\exp\Big(\big(2C+1+(2C +1)m
            +(2C+1)\gamma(U)\big)td\Big).
\end{aligned}
\end{equation}
Taking the infimum of $\pi$ over all coupled distributions of $(\theta, \hat{\theta})$ whose marginal distributions are $\mu$ and $\hat{\mu}$, we conclude that
\begin{equation}
    W_2^2(\nu_{\bm{X}_0}(t), \hat{\nu}_{\bm{X}_0}(t))\
    \leq C_1t\exp(C_0t)W_2^2(\mu, \hat{\mu}),
    \label{upper_bound}
\end{equation}
where 
\begin{equation}
    C_0\coloneqq \Big(2C+1+(2C +1)m
            +(2C+1)\gamma(U)\Big)d,\,\, C_1 \coloneqq C(1 + m + \gamma(U))
\end{equation}
are two constants. 
            
Using the triangular inequality of the $W_2$ distance \cite{clement2008elementary}, we have
            \begin{equation}
            \begin{aligned}
                W_2(\nu^{\text{e}}_{\bm{X}_0, \delta}(t), \hat{\nu}^{\text{e}}_{\bm{X}_0, \delta}(t))
                &\leq W_2(\nu_{\bm{X}_0}^{\text{e}}(t), \hat{\nu}^{\text{e}}_{\bm{X}_0}(t)) + W_2(\nu_{\bm{X}_0}^{\text{e}}(t), \nu_{\bm{X}_0, \delta}^{\text{e}}(t)) \\
                &\quad\quad+ W_2( \hat{\nu}^{\text{e}}_{\bm{X}_0}(t), \hat{\nu}^{\text{e}}_{\bm{X}_0, \delta}(t))\\
                &\hspace{-2cm}\leq W_2(\nu_{\bm{X}_0}(t), \hat{\nu}_{\bm{X}_0}(t)) +W_2(\hat{\nu}_{\bm{X}_0}(t), \nu^{\text{e}}_{\bm{X}_0}(t)) + W_2(\nu_{\bm{X}_0}(t), \hat{\nu}^{\text{e}}_{\bm{X}_0}(t))\\
                &\hspace{-1cm} + W_2(\nu_{\bm{X}_0}^{\text{e}}(t), \nu_{\bm{X}_0, \delta}^{\text{e}}(t)) + W_2( \hat{\nu}^{\text{e}}_{\bm{X}_0}(t), \hat{\nu}^{\text{e}}_{\bm{X}_0, \delta}(t)),
            \end{aligned}
            \label{triang1}
            \end{equation}
where $\nu^{\text{e}}_{\bm{X}_{0}, \delta}(t)$ and $\hat{\nu}^{\text{e}}_{\bm{X}_{0}, \delta}(t)$ are the empirical conditional probability distributions of $\bm{X}(t)$ and $\hat{\bm{X}}(t)$ at time $t$ conditioned on $|\bm{X}(0)-\bm{X}_{0}|\leq\delta$ and $|\hat{\bm{X}}(0)-\bm{X}_{0}|\leq\delta$, respectively.
For any $\theta\in\mathbb{R}^{\ell}$, consider 
\begin{equation}
\begin{aligned}
            &\d \bm{X}(t; \bm{X}_0) = \bm{f}({\bm{X}}(t; \bm{X}_0), t; \theta)\d t + \bm{\sigma}(\bm{X}(t; \bm{X}_0),
    t; \hat{\theta})\d \bm{B}_t \\
    &\hspace{2cm}+ \int_U\bm{\beta}(\bm{X}(t; \bm{X}_0), \xi, t; \theta)\tilde{N}( \d
    t, \gamma(\d\xi)),\\
    &\d \bm{X}(t; \bm{X}_0') = \bm{f}({\bm{X}}(t; \bm{X}_0'), t; \theta)\d t + \bm{\sigma}(\bm{X}(t; \bm{X}_0'),
    t; \hat{\theta})\d \bm{B}_t \\
    &\hspace{2cm}+ \int_U\bm{\beta}(\bm{X}(t; \bm{X}_0'), \xi, t; \theta)\tilde{N}( \d
    t, \gamma(\d\xi)),\\
    &\quad\quad\bm{X}(0; \bm{X}_0) = \bm{X}_0,\,\,\bm{X}(0; \bm{X}_0') = \bm{X}_0', \,\, \|\bm{X}_0-\bm{X}_0'\|\leq\delta.
\end{aligned}
    \label{couple_ic}
\end{equation}
Using the stochastic Gronwall lemma \cite[Theorem 2.2]{mehri2019stochastic} and \cite[Theorem 2.1]{xia2024efficient}, we conclude that
\begin{equation}
    \E\Big[\big\|\bm{X}(t; \bm{X}_0) -
              \bm{X}(t; \bm{X}_0')\big\|_2^{2}\Big]\leq\exp\Big(\big(2dC + dCm + dC\gamma(U)+1)t\Big)\E[\|\bm{X}_0-\bm{X}_0'\|^2].
\end{equation}
Therefore, we have:
\begin{equation}
    W_2(\nu_{\bm{X}_0}^{\text{e}}(t), \nu_{\bm{X}_0, \delta}^{\text{e}}(t))\leq \delta \exp(\tfrac{C_0t}{2}).
    \label{w21}
\end{equation}
Similarly, we conclude that:
\begin{equation}
    W_2(\hat{\nu}_{\bm{X}_0}^{\text{e}}(t), \hat{\nu}_{\bm{X}_0, \delta}^{\text{e}}(t))\leq \delta \exp(\tfrac{C_0t}{2}).
    \label{w22}
\end{equation}

Eq.~\eqref{upper_bound} also holds if we replace 
$\hat{\nu}_{\bm{X}_0}(t)$ on the LHS of Eq.~\eqref{upper_bound} with the empirical distribution $\nu^{\text{e}}_{\bm{X}_0}(t)$ and then replace $\hat{\mu}$ with $\mu_{\bm{X}_0}^{\text{e}}$
on the RHS, \textit{i.e.},
\begin{equation}
    W_2(\nu_{\bm{X}_0}(t), \nu_{\bm{X}_0}^{\text{e}}(t))\leq \sqrt{C_1t}\exp(\tfrac{C_0t}{2})W_2(\mu, \mu_{\bm{X}_0}^{\text{e}}).
    \label{empirical1}
\end{equation}
Similarly,
\begin{equation}
    W_2^2(\hat{\nu}_{\bm{X}_0}(t), \hat{\nu}_{\bm{X}_0}^{\text{e}}(t))\leq \sqrt{C_1t}\exp(\tfrac{C_0t}{2})W_2(\hat{\mu}, \hat{\mu}_{\bm{X}_0}^{\text{e}}).
    \label{empirical2}
\end{equation}
In Eqs.~\eqref{empirical1} and ~\eqref{empirical2}, $\mu_{\bm{X}_0}^{\text{e}}$ and $\hat{\mu}_{\bm{X}_0}^{\text{e}}$ denote the empirical distributions of $\theta$ and $\hat{\theta}$.
Finally, by plugging Eqs.~\eqref{w21},~\eqref{w22},~\eqref{empirical1}, ~\eqref{empirical2}, ~\eqref{upper_bound} into Eq.~\eqref{triang1}, we conclude:
\begin{equation}
\begin{aligned}
     &W_2(\nu_{\bm{X}_0, \delta}^{\text{e}}(t), \hat{\nu}^{\text{e}}_{\bm{X}_0, \delta}(t))\leq 2\delta \exp(\tfrac{C_0t}{2}) \\
     &\quad\quad+ \sqrt{C_1t\exp(C_0t)}\big(W_2(\mu, \mu_{\bm{X}_0}^{\text{e}}) 
     + W_2(\hat{\mu}, \hat{\mu}_{\bm{X}_0}^{\text{e}})+W_2(\mu, \hat{\mu})\big).
\end{aligned}
\label{w2_bound1}
\end{equation}

Squaring both sides of Eq.~\eqref{w2_bound1}, integrating over time, and taking the expectation w.r.t. the empirical probability measure of the initial condition $\bm{X}_0$, we conclude that
\begin{equation}
\begin{aligned}
        &\E[\tilde{W}_{2, \delta}^{2, \text{e}}(\bm{X}, \hat{\bm{X}})]\leq 8T\delta^2 \exp(C_0T)\\
        &\quad\quad+\sum_{j=1}^N\frac{1}{N} \frac{6C_1}{C_0}T\exp(C_0T)\big(W_2^2(\mu, \hat{\mu}) +  \E\big[W_2^2(\mu_{\bm{X}_{0, j}}^{\text{e}}, \mu)\big]+  \E\big[W_2^2(\hat{\mu}_{\bm{X}_{0, j}}^{\text{e}}, \hat{\mu})\big] \big).
        \end{aligned}
        \label{thm2_result0}
        \end{equation}
        Specifically, $C_0$ grows linearly with the dimensionality $d$.
        Finally, taking the expectation of Eq.~\eqref{thm2_result0} and applying the empirical error bound of the squared $W_2$ distance in \cite{fournier2015rate}, we conclude:
        \begin{equation}
            \begin{aligned}
            &\sum_{j=1}^N\frac{1}{N}\E\big[W_2^2(\mu_{\bm{X}_{0, j}}^{\text{e}}, \mu)\big]\leq C_2\E[h(N^{\#}(\bm{X}_0;\delta), \ell)]\Theta_6^{\frac{1}{3}},\\
    &\hspace{1cm}\,\,\sum_{j=1}^N\frac{1}{N}\E\big[W_2^2(\hat{\mu}_{\bm{X}_{0, j}}^{\text{e}}, \hat{\mu})\big]\leq C_2\E[h(N^{\#}(\bm{X}_0;\delta), \ell)]\hat{\Theta}_6^{\frac{1}{3}},
\end{aligned}
    \label{theorem2_result}
\end{equation}
where $C_2$ is a constant used in \cite[Theorem 1]{fournier2015rate} and $h$ is defined in Eq.~\eqref{t_def}. This
 completes the proof of Theorem~\ref{theorem1}.

\section{Reconstructing the distribution of parameters in an SPDE}
\label{AppendixD}
We consider the following parabolic SPDE with a Dirichlet boundary condition studied in \cite{grecksch1996time}:
\begin{equation}
\begin{aligned}
    \d U(\bm{x}, t; \theta) &= A(\theta) U(\bm{x}, t; \theta) + f(U(\bm{x}, t; \theta); \theta)\d t + g(U(\bm{x}, t; \theta); \theta)\d B_t,\\
    &\quad\,\,\bm{x}\in D, t\in[0, T],\\
    &\hspace{-1cm}U(\bm{x}, 0; \theta) = U(\bm{x}, 0) \in H^{1,2}_0(\Omega), U(\bm{x}, 0)\sim \nu_0, \,\,U(\bm{x}, t;\theta) = 0, x\in\partial D.
\end{aligned}
    \label{spde}
\end{equation}
In Eq.~\eqref{spde}, $B_t$ is a standard scalar Wiener process and $D$ is a bounded domain in $\mathbb{R}^d$ with sufficiently smooth boundary $\partial D$. $H_0^{1,2}(\Omega)$ is the space of functions $U: \Omega \to \mathbb{R}$ that vanish on $\partial \Omega$ such that $U$ and its first-order generalized derivatives belong 
to $L^2(\Omega)$ equipped with the norm $\|U\|_{L^2}\coloneqq \int_D U^2(\bm{x}, t)\d\bm{x}$. $\nu_0$ is a probability measure defined on the Sobolev space $\mathcal{B}(H^{1,2}_0(\Omega))$. $A$ is a linear operator densely defined in 
$L^2(\Omega)$ such that for $U \in H^{1,2}(\Omega)$,  $AU\in L^2(\Omega)\}$. $A$ is strongly monotone, \textit{i.e.}, 
there is a constant $\alpha>0$ such that
\begin{equation}
    (-AU, U) \geq \alpha \|U\|^2_{H^{1, 2}}, \quad\forall U \in H_0^{1,2}(\Omega),
\label{spde_coer}
\end{equation}
where $(\cdot, \cdot)$ is the inner product. In addition, $f$ and $g$, which map either $L^2(\Omega)$ or $H_0^{1, 2}(\Omega)$ into itself, are formed from real-valued functions of a real variable with uniformly bounded derivatives of appropriate order.

    We reconstruct the distribution of the model parameter $\theta$ in Eq.~\eqref{spde} using an approximate model
    \begin{equation}
    \begin{aligned}
                &\d \hat{U}(\bm{x}, t; \hat{\theta}) = A(\hat{\theta}) \hat{U}(\bm{x}, t; \hat{\theta}) + f(\hat{U}(\bm{x}, t; \hat{\theta}); \hat{\theta})\d t + g(\hat{U}(\bm{x}, t; \hat{\theta}); \hat{\theta})\d \hat{B}_t,\\&\quad\quad\quad \bm{x}\in D, t\in[0, T],\\
                &\quad\quad \hat{U}(\bm{x}, 0) = U(\bm{x}, 0;\theta)=U(\bm{x}, 0), \,\,U(\bm{x}, t;\hat{\theta}) = 0, x\in\partial D.
    \end{aligned}
        \label{approx_sde}
    \end{equation}
    In Eq.~\eqref{approx_sde}, $\hat{B}_t$ is another standard scalar Wiener process independent of $B_t$ in Eq.~\eqref{spde}.
    We use the It\^o-Galerkin scheme for spatial discretization of Eqs.~\eqref{spde} and ~\eqref{approx_sde}:
\begin{equation}
\begin{aligned}
        &\d U_n(t;\theta) = \big(A_n(U_n(t;\theta); \theta) + f_n(U_n(t;\theta);\theta)\big)  \d t + g_n(U_n(t;\theta);\theta) \d B_t,\\
    &\d \hat{U}_n(t;\hat{\theta}) = \big(A_n(\hat{U}_n(t;\hat{\theta}); \hat{\theta}) + f_n(\hat{U}_n(t;\hat{\theta}); \hat{\theta})\big) \d t + g_n(\hat{U}_n(t;\hat{\theta});\hat{\theta}) \d \hat{B}_t.
\end{aligned}
    \label{SPDE_discretize}
\end{equation}
In Eq.~\eqref{SPDE_discretize},
\begin{equation}
        U_n(\bm{x},t;\theta) \coloneqq \sum_{j=1}^n u_j(t;\theta) \varphi_j(\bm{x}) \in X_n,\,\,
        \hat{U}_n(\bm{x},t;\hat{\theta}) \coloneqq \sum_{j=1}^n \hat{u}_j(t;\hat{\theta}) \varphi_j(\bm{x})\in X_n
    \label{Undef}    
\end{equation}
refer to the spectral approximations of $U(\bm{x}, t;\theta)$ and $\hat{U}(\bm{x},t;\hat{\theta})$ in Eqs.~\eqref{spde} and~\eqref{approx_sde}, respectively.
$X_n$ is the $n$-dimensional subspace of $H_0^{1}(\Omega)$ spanned by the basis functions $\{\varphi_1, \ldots, \varphi_n\}$. In Eq.~\eqref{SPDE_discretize},
\[
A_n(U; \theta) \coloneqq P_n \big(A(\theta)U\big) , \quad f_n(U; \theta) \coloneqq P_n \big(f(\theta)U\big) , \quad g_n(U; \theta) \coloneqq P_n \big(g(\theta)U\big),
\]
where $P_n$ denotes the projection of $L^2(\Omega)$ or $H_0^{1, 2}(\Omega)$ onto $X_n$
In the spatiotemporal SPDEs~\eqref{approx_sde} and~\eqref{approx_sde}, we assume that for all $\theta$, the operator $-A$ has the same set of corresponding eigenfunctions 
\[
-A(\theta) \varphi_j = \lambda_j(\theta) \varphi_j, \quad j = 1, 2, \ldots, n,\,\, \lambda_j\leq \lambda_{j+1}.
\]
$\varphi_i\in H_0^{1, 2}(D), i=1,2,...$ forms an orthonormal basis in $L^2(D)$ with $\|\varphi_j\|_{L^2}=1$ and $\lambda_j(\theta) \rightarrow\infty$ uniformly in $\theta$ as $j \rightarrow\infty$. We can prove the following result.

\begin{corollary}
\rm
\label{spde_result}

 We assume that $A(\theta)$, interpreted as mappings of  $H_0^{1}(\Omega)$ into $L^2(\Omega)$ and $f(\theta)$ and $g(\theta)$, interpreted as mappings of $H_0^{1}(\Omega)$ into itself, are Lipschitz continuous:
\begin{equation}
    \begin{aligned}
        &\|A(\theta)(U, \theta) - A(\hat{\theta})(\hat{U}, \hat{\theta})\|_{L^2}\leq L(\|U-\hat{U}\|_{H^{1,2}} + \|\theta - \hat{\theta}\|),\\
        &\|f(U; \theta) - f(\hat{U}; \hat{\theta})\|_{{H^{1,2}}}\leq L(\|U-\hat{U}\|_{H^{1,2}} + \|\theta - \hat{\theta}\|),\\
        &\|g(U; \theta) - g(\hat{U}; \hat{\theta})\|_{H^{1,2}}\leq L(\|U-\hat{U}_{H^{1,2}}\|_{{H^{1,2}}} + \|\theta - \hat{\theta}\|), L\leq \infty.\\
    \end{aligned}
\end{equation}
Furthermore, we assume that
\begin{equation}
\E[\|\theta\|^6]\leq\Theta_6,\,\,\E[\|\hat{\theta}\|^6]\leq\hat{\Theta}_6.
\end{equation}

Then, we have the following bound for the local time-decoupled squared distance between the probability measures of $U$ and $\hat{U}$:
\begin{equation}
\begin{aligned}
         \tilde{W}_{2, \delta}^{2, \text{e}}(U, \hat{U}) &\leq 3\cdot \Big(8C_0(\beta_n; n)\delta^2T \exp(C_0(\beta_n;n )T) +\frac{6C_1(\beta_n)T}{C_0(\beta_n;n)}\exp(C_0(\beta_n; n)T)\\
         &\hspace{0.3cm}\times(W_2^2(\mu, \hat{\mu}) + (\Theta_6^{\frac{1}{3}} + \hat{\Theta}_6^{\frac{1}{3}})2T \E[h(N^{\#}(\bm{U}_n(0;\theta); \delta); \ell)])\Big) \\
         &\hspace{0.6cm}+ 3T\sup_{\theta, U(\bm{x}, 0;\theta)} K_{T, U(\cdot, 0), \theta}\lambda_{N+1}^{-1}(\theta)  + 3T\sup_{\hat{\theta}, \hat{U}(\bm{x}, 0;\hat{\theta})} K_{T, U(\cdot, 0), \hat{\theta}}\lambda_{N+1}^{-1}(\hat{\theta}).
\end{aligned}
     \label{spde_ineq_result}
\end{equation}
In Eq.~\eqref{spde_ineq_result}, $K_{T, U(\cdot, 0), \theta}$ is a constant that depends on $T, U(\bm{x}, 0)$ and $\theta$, while $C_i(\beta_n)$ are constants in Theorem~\ref{theorem1} that grow at most linearly with $n$. $\beta_n$ is another constant depending on the eigenvalues $\{\lambda_i\}_{i=1}^n$. The vector $\bm{U}_n(0;\theta)\coloneqq (u_1(0;\theta),...,u_n(0;\theta))$ refers to the spectral expansion of the initial condition. 
 $\tilde{W}_{2, \delta}^{2, \text{e}}(U, \hat{U})$ is the local time-decoupled squared distance between the probability measures of $U$ and $\hat{U}$:
\begin{equation}
    \tilde{W}_{2, \delta}^{2, \text{e}}(U, \hat{U})\coloneqq\int_0^T W_{2, \delta}^{2, \text{e}}\big(U(\bm{x}, t;\theta), \hat{U}(\bm{x}, t;\hat{\theta})\big)\d t,
\end{equation} 
and 
\begin{equation}
    W_{2, \delta}^{2, \text{e}}\big(U(\bm{x}, t;\theta), \hat{U}(\bm{x}, t;\hat{\theta})\big) \coloneqq \int W_2^2(\nu^{\text{e}}_{U_0, \delta}(t), \hat{\nu}^{\text{e}}_{U_0, \delta}(t))\nu_0^{\text{e}}(\d U_0),
    \label{spde_w2}
\end{equation}
where $\nu_0^{\text{e}}(\d U_0)$ is the empirical distribution of the initial condition $U(\cdot, 0)$, and $\nu^{\text{e}}_{U_0, \delta}(t)$ and $\hat{\nu}^{\text{e}}_{U_0, \delta}(t)$ are the empirical conditional distributions of $U(\bm{x},t;\theta)$ and $\hat{\nu}^{\text{e}}_{U_0,\delta}(t)$ at time $t$ conditioned on $\|U(\bm{x},0) - U_0\|_{L^2}\leq\delta$ and $\|\hat{U}(\bm{x},0) - U_0\|_{L^2}\leq {\delta}$, respectively. In Eq.~\eqref{spde_w2}, the $W_2$ distance between $\nu^{\text{e}}_{U_0, \delta}(t)$ and $\hat{\nu}^{\text{e}}_{U_0, \delta}(t)$ is defined as
\begin{equation}
W_2(\nu^{\text{e}}_{U_0, \delta}(t), \hat{\nu}^{\text{e}}_{U_0, \delta}(t)) \coloneqq \inf_{\pi(\nu, \hat \nu)}
\E_{(U, \hat{U})\sim \pi(\nu^{\text{e}}_{U_0, \delta}(t), \hat{\nu}^{\text{e}}_{U_0, \delta}(t))}\big[\|{U}
  - \hat{U}\|_{L^2}^{2}\big]^{\frac{1}{2}}.
\end{equation}

\end{corollary}

\begin{proof}

First, we show that $A_n$ is also Lipschitz continuous:
\begin{equation}
\begin{aligned}
        \|A_n(U_n;\theta) - A_n(\hat{U}_n;\hat{\theta})\|_{L^2} &=  \|P_n(A(U_n;\theta) - A(\hat{U}_n;\hat{\theta}))\|_{L^2}\\
        &\hspace{-2cm}\leq \|A(U_n;\theta) - A(\hat{U}_n;\hat{\theta})\|_{L^2}\leq L(\|U_n - \hat{U}_n\|_{H^{1, 2}} + \|\theta - \hat{\theta}\|).
\end{aligned}
\end{equation}
Because $X_n$ is a finite dimensional space, there exists a constant $\beta_n$ depending on $\varphi_1,...,\varphi_n$ such that $\forall U_n\in X_n$, $\|U\|_{H^{1,2}}\leq \beta_n\|U\|_{L^2}$. Thus, 
\begin{equation}
\begin{aligned}
    \|A_n(U_n; \theta) - A_n(\hat{U}_n; \hat{\theta})\|_{H^{1,2}}&\leq\beta_n\|A_n(U_n; \hat{\theta}) - A_n(\hat{U}_n;\hat{\theta})\|_{L^{2}}\\
&\quad\quad\leq \beta_n L (\|U_n - \hat{U}_n\|_{L^{2}} + \|\theta - \hat{\theta}\|).
\end{aligned}
\label{L_condition_a}
\end{equation}
Similarly, $f_n$ and $g_n$ can also be shown to be Lipschitz continuous in $U_n$ and $\theta$:
\begin{equation}
\begin{aligned}
    \|f_n(U_n; \theta) - f_n(\hat{U}_n; \hat{\theta})\|_{L^{2}}&\leq\|f_n(U_n; \theta) - f_n(\hat{U}_n; \hat{\theta})\|_{H^{1,2}}\\
    &\quad\quad\leq \beta_n L (\|U_n - \hat{U}_n\|_{L^{2}} + \|\theta - \hat{\theta}\|),\\
    \|g_n(U_n; \theta) - g_n(\hat{U}_n; \hat{\theta})\|_{L^{2}}&\leq \|g_n(U_n; \theta) - g_n(\hat{\theta}; \hat{U}_n)\|_{H^{1,2}}\\&\quad\quad\leq \beta_n L (\|U_n - \hat{U}_n\|_{L^{2}} + \|\theta - \hat{\theta}\|).
    \end{aligned}
    \label{L_condition_f}
\end{equation}

From \cite[Section 3]{grecksch1996time}, for every $\theta$, we have
\begin{equation}
    \E [\|U(\bm{x}, k\Delta t; \theta) - U_n(\bm{x}, k\Delta t;\theta) \|^2]\leq K_{k\Delta t, U(\cdot, 0), \theta}\lambda_{N+1}^{-1}(\theta),
    \label{discretize_bound}
\end{equation}
where $K_{k\Delta t, U(\cdot, 0), \theta}$ is a constant depending on time $k\Delta t$ and the initial condition $U(\cdot, 0)$, and $\theta$. Without loss of generality, we assume that $K_{k\Delta t, U(\cdot, 0), \theta}$ is non-decreasing in $k$ (otherwise we can replace $K_{k\Delta t, U(\cdot, 0), \theta}$ with $\tilde{K}_{k\Delta t, U(\cdot, 0), \theta}\coloneqq \max_{1\leq i\leq k}K_{i\Delta t, U(\cdot, 0), \theta}$.
Given the initial condition $U(\bm{x}, 0)$ and $P_n U(\bm{x}, 0)$, we denote the probability measures of $U(\bm{x}, k\Delta t; \theta)$ and $U_n(\bm{x}, T;\theta)$ by $\nu_{U(\cdot, 0)}(k\Delta t)$ and $\nu_{n, U_n(\cdot, 0)}(k\Delta t)$, respectively. Moreover, the joint probability measure of $\big(U(\bm{x}, k\Delta t; \theta), U_n(\bm{x}, k\Delta t;\theta)\big)$ has marginal distributions $\nu_{U(\cdot, 0)}(k\Delta t)$ and $\nu_{n, U_n(\cdot, 0)}(k\Delta t)$, respectively. From Eq.~\eqref{discretize_bound}, we can deduce that:
\begin{equation}
\begin{aligned}
        W_2^2( \nu_{U(\cdot, 0)}(k\Delta t),  \nu_{n, U_n(\cdot, 0)}(k\Delta t)) &\leq \E [\|U(\bm{x}, k\Delta t; \theta) - U_n(\bm{x}, k\Delta t;\theta) \|^2] \\
        &\leq  \sup_{\theta, U(\bm{x}, 0)} K_{T, U(\cdot, 0), \theta}\lambda_{N+1}^{-1}(\theta).
\end{aligned}
\end{equation}
Furthermore, using the definition of the local squared $W_2$ distance in Eq.~\eqref{local_squared}, we have:
\begin{equation}
\begin{aligned}
        W_{2, \delta}^{2, \text{e}}\big(U(\cdot, k\Delta t;\theta), U_n(\cdot, k\Delta t;\theta)\big) &\leq \sup_{\theta, U(\bm{x}, 0)}\E [\|U(\bm{x}, k\Delta t; \theta) - U_n(\bm{x}, k\Delta t;\theta) \|^2] \\
        &\leq  \sup_{\theta, U(\bm{x}, 0)} K_{T, U(\cdot, 0), \theta}\lambda_{N+1}^{-1}(\theta).
\end{aligned}
\end{equation}
Similarly, we can conclude that:
\begin{equation}
\begin{aligned}
        W_{2, \delta}^{2, \text{e}}\big(\hat{U}(\cdot, k\Delta t;\hat{\theta}), \hat{U}_n(\cdot, k\Delta t;\hat{\theta} )\big) 
        &\leq  \sup_{\hat{\theta}, U(\bm{x}, 0)} K_{T, U(\cdot, 0), \theta}\lambda_{N+1}^{-1}(\hat{\theta}).
\end{aligned}
\end{equation}

Given the same initial condition $U(\bm{x}, 0)=\hat{U}(\bm{x}, 0)$, 
using the triangle inequality of the Wasserstein distance \cite{clement2008elementary}, we have
\begin{equation}
    \begin{aligned}
        W_{2, \delta}^{2, \text{e}}(U(\bm{x}, t;\theta), \hat{U}(\bm{x}, t;\hat{\theta}))&\leq 3W_{2, \delta}^{2, \text{e}}\big(U(\cdot, k\Delta t;\theta), U_n(\cdot, k\Delta t;\theta)\big) \\
        &\hspace{-3cm}+ 3W_{2, \delta}^{2, \text{e}}\big(\hat{U}(\cdot, k\Delta t;\hat{\theta}), \hat{U}_n(\cdot, k\Delta t;\hat{\theta})\big) + 3W_{2, \delta}^{2, \text{e}}(U_n(\bm{x}, t;{\theta}), \hat{U}_n(\bm{x}, t;\hat{\theta}))\big)\\
        &\hspace{-2cm}\leq 3W_{2, \delta}^{2, \text{e}}(U_n(\bm{x}, t;\theta), \hat{U}_n(\bm{x}, t;\hat{\theta})) + 3\sup_{\theta, U(\bm{x}, 0)} K_{T, U(\cdot, 0), \theta}\lambda_{N+1}^{-1}(\theta) \\
        &\quad\quad+ 3\sup_{\hat{\theta}, U(\bm{x}, 0)} K_{T, U(\cdot, 0), \hat{\theta}}\lambda_{N+1}^{-1}(\hat{\theta}).
    \end{aligned}
    \label{theorem_spde_ineq0}
\end{equation}
Integrating both sides of the ineqeuality~\eqref{theorem_spde_ineq0} over time, we have:
\begin{equation}
\begin{aligned}
        \tilde{W}_{2, \delta}^{2, \text{e}}\big(U, \hat{U}\big)&\leq 3 \tilde{W}_{2, \delta}^{2, \text{e}}(U_n, \hat{U}_n) + 3T\sup_{\theta, U(\bm{x}, 0)} K_{T, U(\cdot, 0), \theta}\lambda_{N+1}^{-1}(\theta) \\
        &\quad\quad+ 3T\sup_{\hat{\theta}, U(\bm{x}, 0)} K_{T, U(\cdot, 0), \hat{\theta}}\lambda_{N+1}^{-1}(\hat{\theta}).
\end{aligned}
\label{theorem_spde_ineq}
\end{equation}
Let 
\begin{equation}
\begin{aligned}
    \bm{U}_n(t;\theta)\coloneqq (u_1(t;\theta),...,u_n(t;\theta)),\,\,\hat{\bm{U}}_n(t,\hat{\theta})\coloneqq(\hat{u}_1(t,\hat{\theta}),...,\hat{u}_n(t,\hat{\theta}))
    \end{aligned}
\end{equation}
be two vectors of the spectral expansion coefficients of $U_n(\bm{x},t;\theta)$ and $\hat{U}_n(\bm{x},t;\hat{\theta})$ in Eq.~\eqref{Undef}.
We have $\|U_n(\bm{x}, t;\theta)\|_{L^2}=\|\bm{U}_n(t)\|$ and $\|\hat{U}_n(\bm{x}, t)\|_{L^2}=\|\hat{\bm{U}}_n(t)\|$ because $\|\varphi_i\|_{L^2}=1, i=1,...,n$. 
Furthermore, $\bm{U}_n$ and $\hat{\bm{U}}_n$ are solutions to the two SDEs:
\begin{equation}
\begin{aligned}
    \d \bm{U}_n = \big(\bm{A}_n(\bm{U}_n,t;\theta) +\bm{F}_n(\bm{U}_n,t;\theta)\big)\d t + \bm{G}_n(\bm{U}_n,t;\theta)\d B_t,\\
    \d \hat{\bm{U}}_n = \big(\bm{A}_n(\hat{\bm{U}}_n,t;\hat{\theta}) +\bm{F}_n(\hat{\bm{U}}_n,t;\hat{\theta})\big)\d t + \bm{G}_n(\hat{\bm{U}}_n,t;\hat{\theta})\d B_t
    \end{aligned}
\end{equation}
where $\bm{A}_n(\bm{U}_n,t;\theta), \bm{F}_n(\bm{U}_n,t;\theta)$ and $\bm{F}_n(\bm{G}_n,t;\theta)$ are the $n$-dimensional vector of the coefficients in the spectral expansions of $A_n(U_n; \theta), f_n(U_n; \theta)$, and $g_n(U_n; \theta)$ in Eqs.~\eqref{L_condition_a} and~\eqref{L_condition_f}, respectively. $\bm{A}_n, \bm{F}_N$, and $\bm{G}_n$ are also Lipschitz continuous in $U_n$ and $\theta$ from Eqs.~\eqref{L_condition_a} and \eqref{L_condition_f}.
Applying Eq.~\eqref{para_dependence_bound} in Theorem~\ref{theorem1}, we obtain:
\begin{equation}
\begin{aligned}
        &\E[\tilde{W}_{2, \delta}^{2, \text{e}}(U_n, \hat{U}_n)] = \E[\tilde{W}_{2, \delta}^{2, \text{e}}(\bm{U}_n, \hat{\bm{U}}_n)]\leq 8C_0(\beta_n, n)T\delta^2 \exp(C_0(\beta_n;n)T) \\
    & + \frac{6C_1(\beta_n)}{C_0(\beta_n;n)}T\exp(C_0(\beta_n; n)T)\big(W_2^2(\mu, \hat{\mu}) + 2C_2 \E[h(N^{\#}(\bm{U}_n(\bm{x}, 0);\delta), \ell)]\cdot (\Theta_6^{\frac{1}{3}}+\hat{\Theta}_6^{\frac{1}{3}})\big)
\end{aligned}
    \label{identity_ineq}
\end{equation}
where $C_i, i=0,...,2$ are constants defined in Theorem~\eqref{theorem1}.
Taking the expectation of Eq.~\eqref{theorem_spde_ineq} and plugging in the inequality ~\eqref{identity_ineq} for $\E[\tilde{W}_{2, \delta}^{2, \text{e}}(U_n, \hat{U}_n)]$, the inequality~\eqref{spde_ineq_result} is proved.
\end{proof}

\section{Proof of Theorem~\ref{thm_3_1}}
\label{AppendixE}
Given a parameterized multivariate normal distribution $f_{\bm{x}}$, we design an SNN described in Fig.~\ref{fig:nn_model} with the ReLU activation and the linear forward propagation. The probability density function of the output of this SNN, denoted by $\hat{f}_{\bm{x}}$, can approximate $f_{\bm{x}}$ in the $W_2$ distance sense. Given a real number $0<c<\epsilon_0$, we choose $\Delta x>0$ such that: 
\begin{equation}
    W_2^2(f_{\bm{x}}, f_{\tilde{\bm{x}}})<c, \,\,\forall\bm{x}, \tilde{\bm{x}}\in D,\,\, \|\bm{x} - \tilde{\bm{x}}\|< \sqrt{d}\Delta x.
\end{equation}
We consider a uniform equidistance grid set $X\coloneqq\{\bm{x}_i\}_{i=1}^K, \bm{x}_i=(x_i^1,...,x_i^d)$ such that the distance between two adjacent points is $\Delta x$, $D\subseteq \cup_{i=1}^K\otimes_{j=1}^d[x_i^j, x_i^j+\Delta \bm{x})$, and 
$\otimes_{j=1}^d[x_{i_1}^j, x_{i_1}^j+\Delta \bm{x})\cap \otimes_{j=1}^d[x_{i_2}^j, x_{i_2}^j+\Delta \bm{x})=\emptyset$ if $i_1\neq i_2$. Therefore,
$\forall \bm{x}=(x^1,...,x^d)\in D$, there exists $\bm{x}_i\in X$ such that $|\bm{x}-\bm{x}_i|<\sqrt{d}\Delta x$. 
Let $0<\epsilon<\tfrac{1}{2}$ be a small positive number. We set $4dK$ neurons in the first layer, grouped into $dK$ groups. When inputting $\bm{x}$ into the SNN, the outputs of the four neurons in the $(i, j), i=1,.., K, j=1,...,d$ group are: 
\begin{equation}
\begin{aligned}
        &n_{i, j, 1}^1 =\text{ReLU}(\epsilon^{-1}(x^j - x_i^j-\Delta x)),\,\,  n_{i, j, 2}^1 =\text{ReLU}(\epsilon^{-1}(x^j - x_i^j-\Delta x+\epsilon)), \\  &n_{i, j, 3}^1 =\text{ReLU}(\epsilon^{-1}(x^j - x_i^j-\epsilon)),\,\, n_{i, j, 4}^1 = \text{ReLU}(\epsilon^{-1}(x^j - x_i^j)).
\end{aligned}
\end{equation}
The second hidden layer contain $dK$ neurons labeled by $(i, j), i=1,..,K, j=1,...,d$. We set weights between the first layer and the second layer such that the input of the $(i, j)$ neuron in the second layer is: 
    \begin{equation}
    \begin{aligned}
           n_{i, j}^{2, \text{in}}\coloneqq n_{i, j, 2}^1-n_{i, j, 1}^1 - (n_{i, j, 4}^1 - n_{i, j, 3}^1).
    \end{aligned}
    \end{equation}

    It is easy to verify that $n_{i, j}^{2, \text{in}}\in[0, 1]$. Furthermore,
    when $x^j \in [x_{i}^j+\epsilon, x_{i}^j+\Delta x]$, then $n_{i, j}^{2, \text{in}}=1, j=1,...,d$. If $x^j\leq x^j_{i}$ or $x^j\geq x_{i}^j+\Delta x$, then $n_{i, j}^{2, \text{in}}=0$.
    The output of the $(i, j)$ neuron in the second hidden layer is designed as
    \begin{equation}
        n^2_{i, j} = \text{ReLU}( \epsilon^{-1}(n^{2, \text{in}}_{i, j} - 1+\epsilon)),\,\, i=1,...,K,\,\, j=1,...d.
    \end{equation}

    The third layer contains $K$ neurons, and
    the output of each neuron in the third hidden layer is:
    \begin{equation}
        n_i^3 = \text{ReLU}( \sum_{j=1}^d\epsilon^{-1}(n^2_{i, j} - 1+\tfrac{\epsilon}{d})), i=1,...,K. 
    \end{equation}
    Thus, $n_i^3\in[0, 1]$ and $n_i^3 = 1$ when $\bm{x}\in \otimes_j [x_{j, i}+\epsilon, x_{j, i}+\Delta x-\epsilon]$. If there is a $j=1, ..., d$ such that $x^j< x_{i}^{j}$ or $x^j> x_{i}^{j}+\Delta x$, then $n_i^3=0$. For $\bm{x}\in D$, there exists at most one $i$ such that $n_i^3\neq 0$.
    We denote $D(\epsilon) = \{x\in D: \exists! i, n_i^3=1\}$. It is easy to check that as $\epsilon\rightarrow 0$, $D(\epsilon)\rightarrow D$. 
    
    We set $d' K$ neurons in the fourth layer. Each neuron is labeled with $(i, j), i=1,..,K, k=1,...,d'$. The input and output of the $(i, j)$ neuron in the fourth layer are identical (\textit{i.e.}, each neuron in the fourth layer outputs its input):
    \begin{equation}
        n_i^3 (\omega_{i, k}^4 + (A_i^{-1}\bm{b}_i)_k), \,\,i=1,..,K, \,\, k=1,...,d'.
    \end{equation}
    Here, $\omega_{i, k}\sim\mathcal{N}(0, 1)$ are independent random variables, while $\bm{b}_i$ and $A_i$ are the mean and covariance matrix of $f_{\bm{x}_i}(\bm{y})=\mathcal{N}(\bm{b}_i, A_i)$, respectively. $(A_i^{-1}\bm{b}_i)_k$ refers to the $k^{\text{th}}$ component of the vector $A_i^{-1}\bm{b}_i$. 
    The output of the fifth layer is $d'$ dimensional:
\begin{equation}
    \sum_{i=1}^K n_i^3(A_i \bm{\omega}_i^4 + \bm{b}_i),
\end{equation}
where $\bm{\omega}_i \coloneqq (\omega_{i,1},..., \omega_{i, d'})$. Because for any $\bm{x}\in D$, there exists at most one $i$ such that $n_{i}^3\neq 0$, for any $\bm{x}\in D$, we have:
\begin{equation}
    \begin{aligned}
            \sup_{\hat{\bm{y}}\sim f_{\bm{x}}}\E[\|y\|^2]&\leq \sup_i\E\big[\|A_i\bm{\omega}_i^4 + \bm{b}_i\|^2\big]\\
    &\leq \sup_i(\|A^T_iA_i\|_F^2 + \|\bm{b}_i\|_2^2.
    \end{aligned}
\end{equation}

For each $\bm{x}\in D(\epsilon)$, $\hat{f}(\bm{x})$ obeys a multivariate normal distribution which is identical to $f(\bm{x}_i)$ if $n_i^3=1$, \textit{i.e.}, $\|\bm{x}-\bm{x}_i\|<\Delta x$. Therefore, we have
\begin{equation}
\begin{aligned}
    &\int_D W_2^2(f_{\bm{x}}, \hat{f}_{\bm{x}})\gamma(\text{d}\bm{x})< \int_{D(\epsilon)} W_2^2(f_{\bm{x}}, \hat{f}_{\bm{x}})\gamma(\text{d}\bm{x}) \\
    &\hspace{0.5cm}+ 2 \gamma(D-D(\epsilon))\cdot \Big(\sup_{\bm{y}\sim f_{\bm{x}}}\E[\|\bm{y}\|^2] + \sup_{\hat{\bm{y}}\sim f_{\bm{x}}}\E[\|\hat{\bm{y}}\|^2]\Big)\\
    &\hspace{1cm}\leq c\gamma(D) + \gamma(D-D(\epsilon))2(\sup_{\hat{\bm{y}}\sim f_{\bm{x}}}\E[\|y\|^2] + \sup_i(\|A^T_iA_i\|_F^2 + \|\bm{b}_i\|_2^2))\\
    &\hspace{1.5cm}\leq c + 4\gamma(D-D(\epsilon))Y.
    \end{aligned}
    \label{snn_approx}
\end{equation}

Choosing $c<\epsilon_0$ and a small $\epsilon$ such that $\gamma(D - D(\epsilon))\leq \frac{\epsilon_0-c}{4Y}$, we have proved Theorem~\ref{thm_3_1}.

\section{Proof of Corollary~\ref{col3_1}}
\label{AppendixF}
    The proof of Corollary~\ref{col3_1} is based on the proof of Theorem~\ref{thm_3_1}. Given any positive number $c>0$, we can find a $\Delta x>0$:
    \begin{equation}
        W_2^2(f_{\bm{x}}, f_{\tilde{\bm{x}}})<c,\,\, \forall \|\bm{x} - \tilde{\bm{x}}\|<\sqrt{d}\Delta x,\,\, \bm{x}, \tilde{\bm{x}}\in D.
    \end{equation}
    We establish an equidistance grid point set $X\coloneqq\{\bm{x}_i\}_{i=1}^K$ on $D$ such that the distance between two adjacent points is $\Delta x$, $D\subseteq \cup_{i=1}^K\otimes_{j=1}^d[x_i^j, x_i^j+\Delta \bm{x})$ and $\otimes_{j=1}^d[x_{i_1}^j, x_{i_1}^j+\Delta \bm{x})\cap \otimes_{j=1}^d[x_{i_2}^j, x_{i_2}^j+\Delta \bm{x})=\emptyset$ if $i_1\neq i_2$.
    Thus, $\forall \bm{x}\in D$, there exists $\bm{x}_i\in X$ such that $|\bm{x}-\bm{x}_i|\leq\sqrt{d}\Delta x$. 
    Denote $\Phi(x)$ to be the cumulative distribution function of a standard normal random variable. 
    Suppose $-M=h_{i, 0}<h_{i, 1}<...<h_{i, s}=M$ such that:
    \begin{equation}
        \Phi(w_{i, r+1}) - \Phi(w_{i, r})=p_{i, r+1}, r=0,...,s-2,\,\, \Phi(h_{i, 1}) = p_{i, 1},\,\, \Phi(h_{i, s-1}) = 1 - p_{i, s},
    \end{equation}
    where $p_{i, r}\coloneqq p_r(\bm{x}_i)$.
    We design an SNN of the form in Fig.~\ref{fig:nn_model} that satisfies Eq.~\eqref{mixed_result}. We set the first three hidden layers of this SNN to be the same as the first three layers of the SNN designed in the proof of Theorem~\ref{thm_3_1} in ~\ref{AppendixE}. We also refer to the outputs of the third hidden layer as $n_i^3, i=1,...,K$. For $\bm{x}_i\in X$, we denote $\bm{x}_i = (x_{i}^1,...,x_{i}^d)$. From ~\ref{AppendixE}, $n_i^3\in[0, 1]$, $n_i^3 = 1$ when $\bm{x}=(x^1,...,x^d)\in \otimes_{j=1}^d [x_{i}^j+\epsilon, x_{i}^j+\Delta x-\epsilon]$, and $n_i^3 = 0$ when $\bm{x}\in D- \otimes_{j=1}^d [x_{i}^j, x_{i}^j+\Delta x]$. $\epsilon>0$ is a small number to be determined.
    If there is a $j=1, ..., d$ such that $x^j< x_{i}^j$ or $x^j> x_{i}^j+\Delta x$, then $n_i^3=0$. For $\bm{x}\in D$, there is at most one $i$ such that $n_i^3=1$.
    
    Let $\epsilon_0>0$ be a small number to be determined. The fourth layer contains $K(s+1)$ groups of neurons with each group containing 2 neurons. The outputs of the $(i, r), i=1,...,K, r=0,...,s$ group are:
    \begin{equation}
        n_{i, r}^4 = \text{ReLU}(\tilde{w}_i^3n_i^3 -  h_{i, r}-M_0), \tilde{n}_{i, r}^4 = \text{ReLU}(\tilde{w}_i^3n_i^3 -  h_{i, r} - \epsilon_0-M_0).
    \end{equation}
     $\tilde{w}_i^3\sim\mathcal{N}(M_0, 1)$ are independent random variables. $M_0>|h_{i, r}|, \forall i, r$ is a large number and $\epsilon_0>0$ is a small number. Both $M_0$ and $\epsilon_0$ are to be determined. 

    The fifth layer contains $sK$ neurons and the output of the $(i, r)$ neuron is:
    \begin{equation}
        n_{i, r+1}^5 = \text{ReLU}(\epsilon_0^{-1}(n_{i, r}^4 - \tilde{n}_{i, r}^4 - (n_{i, r+1}^4 - \tilde{n}_{i, r+1}^4))),\,\, r=0,...,s-1.
    \end{equation}
    Therefore, when $n_i^3=1$, for $r=0,...,s-1$, we have:
    \begin{equation}
        \begin{aligned}
             &n_{i, r+1}^5=1, \omega_i^3\in[M_0+h_{i, r}+\epsilon_0, M_0+h_{i, r+1}],\\
             &n_{i, r+1}^5=0, \omega_i^3<M_0+h_{i, r}~ \text{or}~ \omega_i^3>M_0+h_{i, r+1}+\epsilon_0,\\
             &n_{i, r+1}^5\in(0, 1), \text{otherwise};
        \end{aligned}
    \end{equation}
when $n_i^3=1$ we have $n_{i, r+1}^5, r=0,...,s-1$. Furthermore, we can see that for any $n_{i}^3\in[0, 1]$
    \begin{equation}
        0\leq \sum_{r=1}^s n_{i, r}^5\leq 1,
    \end{equation}
    and there exists at most two nonzero $n_{i, s}^5, n_{i+1, s}^5>0$ in $(n_{i, 1}^5,...,n_{i, s}^5)$.

    For any $\epsilon_1>0$, there exist a small $\epsilon_0> 0$ and a large $M_0>0$ such that when $n_i^3=1$:
    \begin{equation}
        0\leq p_{i, r} - \hat{p}_{i, r} \leq \frac{\epsilon_1}{s},\,\, \hat{p}_{i, r}\coloneqq p(n_{i, r}^5=1),\,\, i=1,...,K.
    \end{equation}

    The sixth hidden layer contains $sKd'$ neurons, each of whose input and output are both 
    \begin{equation}
    n_{i, r, k}^6 = n_{i, r}^5(w_{i, r, k} + (A_{i, r}^{-1}\bm{b}_{i, r})_k), i=1,...,K, r=1,...,s, k=1,...,d'.
    \end{equation}
    Here, $w_{i, r, k}\sim\mathcal{N}(0, 1)$ are independent random variables. $\bm{b}_{i, s}$ and $A_{i, s}$ are the mean vector and covariance matrix such that
    \begin{equation}
        f_{\bm{x}_i}(\bm{y}) = \sum_{r=1}^sp_r(\bm{x}_i)\mathcal{N}(\bm{b}_{i, r}, A_{i, r}^TA_{i, r}).
    \end{equation}
    The seventh layer contains $d'$ neurons whose output is:
    \begin{equation}
        \sum_{i=1}^K\sum_{r=1}^s n_{i, r}^5(A_{i, r}\bm{w}_{i, r} + \bm{b}_{i, r}),
    \end{equation}
where $\bm{w}_{i, r}\coloneqq (w_{i, r, 1},...,w_{i, r, d'})$.
    When $\bm{x}\in D(\epsilon)\coloneqq \{x\in D: \exists! i, n_i^3=1\}$, if $n_i^3=1$, then the probability density function of the output of the seventh layer can be written as:
    \begin{equation}
        \sum_{r=1}^s \hat{p}_{i, r}\mathcal{N}(\bm{b}_{i, r}, A_{i, r}^TA_{i, r}) + p_i(\bm{y}), \,\, \int_{\mathbb{R}^{d'}} p_i(\bm{y})\d\bm{y}\coloneqq 1-\sum_{r=1}^s \hat{p}_{i, r}\leq\epsilon_1, \,\, p_i(\bm{y})\geq 0.
    \end{equation}
    Furthermore, since there exist at most two nonzero consecutive $n_{i, s}^5, n_{i+1, s}^5>0$, we have:
    \begin{equation}
    \begin{aligned}
                &\int_{\mathbb{R}^{d'}} \|\bm{y}\|^2p_i(\bm{y})\d\bm{y}\leq \int_{\mathbb{R}^{d'}}p_i(\bm{y})\d\bm{y}\cdot \Big(2\max_{i, r}\E_{\bm{y}\sim\mathcal{N}(\bm{b}_{i, r}, A_{i, r}^TA_{i, r}) }[\|\bm{y}\|^2] \\
&\hspace{4cm}+ 2\E_{\bm{y}\sim\mathcal{N}(\bm{b}_{i, r}, A_{i, r}^TA_{i, r}) }[\|\bm{y}\|^2]\Big)\\
                &\hspace{4.5cm}=
                (1 - \sum_{r=1}^s\hat{p}_{i, r})4\max_{r}(\|\bm{b}_{i, r}\|^2+\|A_{i, r}^TA_{i, r}\|_F^2)\\
        &\hspace{5cm}\leq 4\epsilon_1 \max_{r}(\|\bm{b}_{i, r}\|^2+\|A_{i, r}^TA_{i, r}\|_F^2).
    \end{aligned}
    \end{equation}

Applying Lemma~\ref{mixed_gauss}, we have
\begin{equation}
    W_2^2(\hat{f}_{\bm{x}}, f_{\bm{x}_i})\leq 2 \max_r\frac{\epsilon_1}{s}\cdot s(\|A_{i, r}^TA_{i, r}\|_F^2 +\|\bm{b}_{i, r}\|^2)  + 4\epsilon_1\max_r(\|\bm{b}_{i, r}\|^2+\|A_{i, r}^TA_{i, r}\|_F^2), 
    \label{6epsilon}
\end{equation}
when $\bm{x}\in D^i(\epsilon)\coloneqq \{\bm{x}\in D(\epsilon)|\|\bm{x}_i-\bm{x}\|\leq\|\bm{x}_j-\bm{x}\|, \forall j\neq i\}$. Finally, we have:
\begin{equation}
\begin{aligned}
        \int_D W_2^2(f_{\bm{x}}, \hat{f}_{\bm{x}})\gamma(\d\bm{x})&\leq \int_{D(\epsilon)}W_2^2(f_{\bm{x}}, \hat{f}_{\bm{x}})\gamma(\d\bm{x})\\
        &\quad+2(1 - \gamma(D(\epsilon))\Big(\E[\|\bm{y}\|^2] + 4\sup_{i, r}\big(\|A_{i, r}^TA_{i, r}\|_F^2 + \|\bm{b}_{i, r}\|^2]\big)\Big)\\
        &\leq 2\sum_i\int_{ D^i(\epsilon)}W_2^2(\hat{f}_{\bm{x}}, f_{\bm{x}_i})\gamma(\d\bm{x}) +2\sum_i\int_{D^i(\epsilon)}c\gamma(\d\bm{x})  \\
        &\quad+ 2(1-\gamma(D(\epsilon))\Big(\E[\|\bm{y}\|^2] + 4\sup_{i, r}(\|A_{i, r}^TA_{i, r}\|_F^2 + \|\bm{b}_{i, r}\|^2])\Big)\\
    &\leq 2\sup_{i, r}\Big(6\epsilon_1\big(\|A_{i, r}^TA_{i, r}\|_F^2 + \|\bm{b}_{i, r}\|^2\big)\Big) + 
    2c
    \\
    & + 2(1-\gamma(D(\epsilon))\Big(\E[\|\bm{y}\|^2] + 4\sup_{i, r}(\|A_{i, r}^TA_{i, r}\|_F^2 + \|\bm{b}_{i, r}\|^2])\Big).
\end{aligned}
\label{col3_result}
\end{equation} 
 Note that $\sup_{i, r}(\|A_{i, r}^TA_{i, r}\|_F^2+ \|\bm{b}_{i, r}\|^2)$ is uniformly bounded from the assumption Eq.~\eqref{bounded_condition}. Letting $c, \epsilon, \epsilon_1\rightarrow 0^+$ in Eq.~\eqref{col3_result}, we have proved the inequality~\eqref{mixed_result}.

    \section{Proof of Theorem~\ref{thm3_2}}
    \label{proof_thm3_2}
    First, we define an auxiliary function with an additional parameter $\sigma$:
    \begin{equation}
        f_{\sigma^2}(\bm{y})\coloneqq\int_{\mathbb{R}^{d'}}f(\bm{y}')\mathcal{N}(\bm{y}-\bm{y}';\sigma^2 I_{d'\times d'})\d\bm{y}'.
    \end{equation}
    Since $f$ is uniformly continuous and uniformly bounded, $\forall\epsilon_0>0$, there exists a small $\delta>0, \sigma_0>0$ such that for any $\sigma<\sigma_0$, \added{we have i) $|f(\bm{x})-f(\tilde{\bm{x}})|<\epsilon_0, \forall |\tilde{\bm{x}}-\bm{x}|<\delta$ and ii)}
    \begin{equation}
        \int_{B(0, \delta)} \mathcal{N}(\bm{y}; \sigma^2I_{d'\times d'})\d\bm{y}>1-\epsilon_0.
    \end{equation}
    Therefore, we conclude that
        $\lim_{\sigma\rightarrow0}f_{\sigma^2}(\bm{x})=f(\bm{x})$ uniformly on $\mathbb{R}^{d'}$. 

Letting $\{\bm{y}_j\}_{j=1}^{(n_0+1)^{d'}}$ and $\{w_j\}_{j=1}^{(n_0+1)^{d'}}$ be the multidimensional Hermite collocation points and weights on $\mathbb{R}^{d'}$ described in \cite{shen2011spectral}, we have
    \begin{equation}
    \int \mathcal{I}_{n_0}f(\bm{y}')\cdot \mathcal{I}_{n_0}\mathcal{N}(\bm{y}-\bm{y}';\sigma^2I_{d'\times d'})\d\bm{y}' = \sum_{j=1}^{(n_0+1)^{d'}}f(\bm{y}_i)\mathcal{N}(\bm{y}-\bm{y}_i;\sigma^2I_{d'\times d'})w_j.
    \label{inter_polate}
\end{equation}
Here, $\mathcal{I}_{n_0}$ is the interpolation operator such that
\begin{equation}
    f(\bm{y}_j) = \mathcal{I}_{n_0}f(\bm{y_j})\in P_{n_0},\,\, j=1,...,(n_0+1)^{d'},
\end{equation}
where $P_{n_0}$ is the space spanned by the generalized Hermite functions $\hat{\mathcal{H}}_{\bm{n}}(\bm{y})$ such that ${|\bm{n}|_{\infty}\leq n_0}$ and 
\begin{equation}
    \hat{\mathcal{H}}_{\bm{n}}(\bm{y})\coloneqq \prod_{i=1}^{d'}\hat{\mathcal{H}}_{n_i}(y_i),\,\, \bm{n}=(n_1,...,n_{d'}),\,\, \bm{y}=(y_1,...,y_{d'})
\end{equation}
is the multidimensional generalized Hermite function defined in \cite{shen2011spectral} ($\hat{\mathcal{H}}_{n_i}$ is the 1D generalized Hermite function of order $n_i$).
We denote
\begin{equation}
\begin{aligned}
    f_{\sigma^2, n_0}(\bm{y})&\coloneqq \sum_{j=1}^{(n_0+1)^{d'}}f(\bm{y}_i)\cdot\mathcal{N}(\bm{y}-\bm{y}_i;\sigma^2I_{d'\times d'})w_j\\
    &\quad=\int  \mathcal{I}_{n_0}f(\bm{y}')\cdot \mathcal{I}_{n_0}\mathcal{N}(\bm{y}-\bm{y}';\sigma^2I_{d'\times d'})\text{d}\bm{y}',
    \end{aligned}
\end{equation}
where $f_{\sigma^2, n_0}$ is nonnegative because the collocation weights $w_j>0$.
Furthermore,
\begin{equation}
    \begin{aligned}
            &\Big|\int f(\bm{y}')\mathcal{N}(\bm{y}-\bm{y}';\sigma^2I_{d'\times d'}) - \mathcal{I}_{n_0}f(\bm{y}')\mathcal{I}_{n_0}\mathcal{N}(\bm{y}-\bm{y}';\sigma^2I_{d'\times d'})\text{d}\bm{y}'\Big|\\
    &\quad\leq \|f-\mathcal{I}_{n_0}f\|_{L^2}\cdot\|\mathcal{N}(\bm{y}-\bm{y}';\sigma^2I_{d'\times d'})\|_{L^2} \\
    &\hspace{1cm}+ \|\mathcal{I}_{n_0}f\|_{L^2}\cdot \|\mathcal{N}(\bm{y}-\bm{y}';\sigma^2I_{d'\times d'})-\mathcal{I}_{n_0}\mathcal{N}(\bm{y}-\bm{y}';\sigma^2I_{d'\times d'})\|_{L^2}. 
    \end{aligned}
    \label{error_inter}
\end{equation}
Using \cite[Theorem 7.18, Theorem 8.6]{shen2011spectral}, we have
\begin{equation}
\begin{aligned}
        \|f - \mathcal{I}_{n_0}f\|&\leq \| f - \mathcal{I}_{n_0}^1f \|_{L^2} + \|\mathcal{I}_{n_0}^2\circ ...\circ \mathcal{I}_{n_0}^{d'}f - f\|_{L^2} \\
        &\quad\quad+ \|(\mathcal{I}_{n_0}^1-\mathbb{I}) \circ( \mathcal{I}_{n_0}^2\circ ...\circ \mathcal{I}_{n_0}^{d'}f - f)\|_{L^2},\\
        &\leq Cn_0^{-\frac{1}{3}}\|\partial_{y_1}f\|_{L^2} + \|\mathcal{I}_{n_0}^2\circ ...\circ \mathcal{I}_{n_0}^{d'}f - f\|_{L^2} \\
        &\quad\quad+ \|  \mathcal{I}_{n_0}^2\circ ...\circ (\mathcal{I}_{n_0}^1-\mathbb{I})(\mathcal{I}_{n_0}^{d'}f - f)\|_{L^2},\\
    & \leq Cn_0^{-\frac{1}{3}}\|\partial_{y_1}f\|_{L^2} + \|\mathcal{I}_{n_0}^2\circ ...\circ \mathcal{I}_{n_0}^{d'}f - f\|_{L^2} \\
    &\quad\quad+ Cn_0^{-\frac{1}{3}}\| (\mathcal{I}_{n_0}^2\circ ...\circ \mathcal{I}_{n_0}^{d'}\partial_{y_1}{f} -\partial_{y_1}{f}\|_{L^2}\\
    &\leq ...\\
    &\leq C{n_0}^{-\frac{1}{3}}| {f}|_{mix}           
\end{aligned}
\label{mixf}
\end{equation}
Here, $C$ is a constant and $n_0$ is taken to be large enough such that $C{n_0}^{-\frac{1}{3}}<1$. $\mathbb{I}$ is the identity operator and $\mathcal{I}_{n_0}^i, i=1,..,d'$ is the projection operator in the $i^{\text{th}}$ direction, \textit{i.e.}, if we denote $X_{n_0}\coloneqq \{y_i\}_{i=0}^{n_0}$ to be the 1D Hermite collocation points, then
\begin{equation}
   \mathcal{I}_{n_0}^if(\bm{y}) = f(\bm{y}), \,\,\forall \bm{y}=(y_1,...,y_{d'})~\text{if}~ y_i \in X_{n_0}.
\end{equation}
Similarly, for any fixed $\bm{y}\in\mathbb{R}^{d'}$:
\begin{equation}
    \|\mathcal{N}(\bm{y}-\bm{y}';\sigma^2I_{d'\times d' }) - \mathcal{I}_{n_0}\mathcal{N}(\bm{y}-\bm{y}';\sigma^2I_{d'\times d' })\|_{L^2}                                    \leq \sum_{|\bm{n}|_0\leq n_0}Cn_0^{-\frac{1}{3}}\|\partial_{\bm{n}} \mathcal{N}\|_{L^2}. 
    \label{mixn}
\end{equation}
Combining Eqs.~\eqref{inter_polate},~\eqref{error_inter},~\eqref{mixf}, and~\eqref{mixn}, for every $\bm{y}$ and every $\sigma>0$, as $n_0\rightarrow \infty$,
\begin{equation}
\begin{aligned}
        &\Big|\sum_{j} f(\bm{y}_j)\mathcal{N}(\bm{y}-\bm{y}_j;\sigma^2I_{d'\times d'})- \int_{\mathbb{R}^{d'}} f(\bm{y}')\mathcal{N}(\bm{y}-\bm{y}';\sigma^2I_{d'\times d'})\d\bm{y}'\Big|\\
    &\hspace{0.2cm}\leq Cn_0^{-\frac{1}{3}}\Big(|f|_{\text{mix}}\|\mathcal{N}(\bm{y}, \sigma^2I_{d'\times d'})\|_{L^2} + \big(\|f\|_{L^2} + Cn_0^{-\frac{1}{3}}|f|_{\text{mix}})|\big)\cdot |\mathcal{N}(\bm{y}, \sigma^2I_{d' \times d'})|_{\text{mix}}\Big),
\end{aligned}
\end{equation}
which implies that for any fixed $\sigma>0$, 
\begin{equation}
    f_{\sigma^2, n_0}(\bm{y}) = \sum_{j} f(\bm{y}_j)\mathcal{N}(\bm{y}-\bm{y}_j;\sigma^2I_{d'\times d'})\rightarrow \int_{\mathbb{R}^{d'}} f(\bm{y}')\mathcal{N}(\bm{y}-\bm{y}';\sigma^2I_{d'\times d'})\d\bm{y}' 
    \label{convergence_n}
\end{equation}
uniformly in $\bm{y}\in\mathbb{R}^{d'}$ as $n_0\rightarrow\infty$.

Additionally, we have:
\begin{equation}
    \begin{aligned}
        &\int_{\mathbb{R}^{d'}}\|\bm{y}\|^2 \sum_{j=1}^{(n_0+1)^{d'}}f(\bm{y}_j)w_j\mathcal{N}(\bm{y}-\bm{y}_j; \sigma^2I_{d'\times d'})\d\bm{y} = \sum_{j=1}^{(n_0+1)^{d'}} f(\bm{y}_j)(|\bm{y}_j|^2+d'\sigma^2)w_j\\ 
    &= \sum_j f(\bm{y}_j)(|\bm{y}_j|^2+d'\sigma^2)^2\frac{1}{|\bm{y}_j|^2+d'\sigma^2}w_j\\
    &=\int_{\mathbb{R}^{d'}} \mathcal{I}_{n_0}\big(f(\bm{y})(\|\bm{y}\|^2+d'\sigma^2)^2)\cdot \mathcal{I}_{n_0}\big(\frac{1}{(\|\bm{y}\|^2+d'\sigma^2)}\big)\d\bm{y}\\
    &\leq  \int_{\mathbb{R}^{d'}} f(\bm{y})(\|\bm{y}\|^2+d'\sigma^2)^2\cdot \frac{1}{(\|\bm{y}\|^2+d'\sigma^2)}\d\bm{y} \\
    &\quad\quad+ \big\|(\mathcal{I}_{n_0}-\mathbb{I})\big(f(\bm{y})(\|\bm{y}\|^2+d'\sigma^2)^2\big)\big\|_{L^2}\cdot\big\|\frac{1}{\|\bm{y}\|^2+d'\sigma^2}\big\|_{L^2} \\
    &\hspace{1cm}+ \|\mathcal{I}_{n_0}\big(f(\bm{y})(\|\bm{y}\|^2+d'\sigma^2)^2\big)\|_{L^2}\cdot \|(\mathbb{I}-\mathcal{I}_{n_0})\big(\frac{1}{\|\bm{y}\|^2+d'\sigma^2}\big)\|_{L^2}\\
    &\leq (\sum_{i, j=1}^{d'} 2 C n_0^{-\frac{1}{3}}|f(\bm{y})y_i^2y_j^2|_{\text{mix}}+2C n_0^{-\frac{1}{3}}d'\sigma^2\sum_{i=1}^{d'}|f(\bm{y})y_i^2|_{\text{mix}} \\
    &\hspace{1.5cm}+ C n_0^{-\frac{1}{3}}(d')^2\sigma^4 |f|_{\text{mix}})\cdot\sigma^{-1}C_1(d') \\
    &\hspace{1cm}+ \|\mathcal{I}_{n_0}\big(f(\bm{y})(\bm{y}+d'\sigma^2)^2\big)\|_{L^2} \cdot Cn_0^{-\frac{1}{3}}\big|\frac{1}{\|\bm{y}\|^2+d'\sigma^2}\big|_{\text{mix}} + \E[\|y\|^2] + d'\sigma^2,
    \end{aligned}
    \label{expectation_bound}
\end{equation}
where $C_1(d')$ is another constant depending on the dimensionality $d'$.
Since
\begin{equation}
\begin{aligned}
    &\|\mathcal{I}_{n_0}\big(f(\bm{y})(\|\bm{y}\|^2+d'\sigma^2)^2\big) - f(\bm{y})(\|\bm{y}\|^2+d'\sigma^2)^2\|_{L^2} \\
    &\hspace{1cm}\leq \sum_{i, j=1}^{d'} 2 C n_0^{-\frac{1}{3}}|f(\bm{y})y_i^2y_j^2|_{\text{mix}}+2C n_0^{-\frac{1}{3}}d'\sigma^2\sum_{i=1}^{d'}|f(\bm{y})y_i^2|_{\text{mix}} \\
    &\hspace{2cm}+ C n_0^{-\frac{1}{3}}(d')^2\sigma^4 |f|_{\text{mix}}\leq\infty,
    \end{aligned}
    \label{bounded}
\end{equation}
we conclude that
\begin{equation}
    \int_{\mathbb{R}^{d'}}\|\bm{y}\|^2 \sum_{j=1}^{(n_0+1)^{d'}}f(\bm{y}_j)w_j\mathcal{N}(\bm{y}-\bm{y}_j; \sigma^2)\d\bm{y}<\infty.
\end{equation}
Furthermore, from~\eqref{bounded}, $\|\mathcal{I}_{n_0}\big(f(\bm{y})(\|\bm{y}\|^2+d'\sigma^2)^2\big)\|_{L^2}\rightarrow\|f(\bm{y})(\|\bm{y}\|^2+d'\sigma^2)^2\|_{L^2}$ as $n_0\rightarrow\infty$.
We denote
\begin{equation}
    \tilde{f}_{\sigma^2, n_0}(\bm{y}) \coloneqq \frac{1}{\sum_{j=1}^{(n_0+1)^{d'}}f(\bm{y}_j)w_j}f_{\sigma^2, n_0}(\bm{y}).
    \label{f_tilde_def}
\end{equation}
Therefore, $\int_{\mathbb{R}^{d'}}\tilde{f}_{\sigma^2, n_0}(\bm{y})\d\bm{y} =1$. Notice that:
\begin{equation}
\begin{aligned}
        \big|\int_{\mathbb{R}^{d'}}f(\bm{y})\d\bm{y} - \sum_{j=1}^{n_0^{d'}}f(\bm{y}_j)w_j\big| &= \int f - \mathcal{I}_{n_0}\sqrt{f}\cdot\mathcal{I}_{n_0}\sqrt{f}\d\bm{y}\\
        &\hspace{-1cm}\leq \|\sqrt{f}\|_{L^2}Cn_0^{-\frac{1}{3}}|\sqrt{f}|_{mix} + \|\mathcal{I}_{n_0}\sqrt{f}\|_{L^2}Cn_0^{-\frac{1}{3}}|\sqrt{f}|_{mix},\\
        &\leq Cn_0^{-\frac{1}{3}}|\sqrt{f}|_{mix}(2\|\sqrt{f}\|_{L^2} + Cn_0^{-\frac{1}{3}}|\sqrt{f}|_{\text{mix}}).
\end{aligned}
\end{equation}
Therefore, we can write 
\begin{equation}
    \sum_{j=1}^{n_0^{d'}}f(\bm{y}_j)w_j \coloneqq 1 + n_0^{-\frac{1}{3}}c(\sqrt{f})\leq\infty,
\end{equation}
where $c(\sqrt{f})$ is a constant depending on $|\sqrt{f}|_{\text{mix}}$ ($\|\sqrt{f}\|_{L^2}=1$). We assume $n_0$ is large enough such that $n_0^{-\frac{1}{3}}c(\sqrt{f})<\frac{1}{2}$. Thus, $\E_{\bm{y}\sim f_{\sigma^2, n_0}}[\|{\bm{y}}\|^2]$ is uniformly bounded as $n_0\rightarrow\infty$. From the inequality~\eqref{expectation_bound} and the definition of $\tilde{f}_{\sigma^2, n_0}$ in Eq.~\eqref{f_tilde_def}, we can bound $\E_{\bm{y}\sim \tilde{f}_{\sigma^2, n_0}}[\|\bm{y}\|^2]$ by
\begin{equation}
    \begin{aligned}
        \E_{\bm{y}\sim \tilde{f}_{\sigma^2, n_0}}[\|\bm{y}\|^2]&\leq \frac{1}{1 - |c(f)|n_0^{-\frac{1}{3}}}\Big[ \E_{\bm{y}\sim f}[\|\bm{y}\|^2]+ d'\sigma^2 \\
        &\hspace{1cm}+ n_0^{-\frac{1}{3}}(C_2(\sigma; f)+\E_{\bm{y}\sim f}[\|\bm{y}\|^2]+1) \Big]\\
        &\hspace{-2cm}\leq \Big[ \E_{\bm{y}\sim f}[\|\bm{y}\|^2] + n_0^{-\frac{1}{3}}(C_2(\sigma; f)+\E_{\bm{y}\sim f}[\|y\|^2]+1) + d'\sigma^2\Big] \\
    &\hspace{-1.5cm}+ 2|c(\sqrt{f})|n_0^{-\frac{1}{3}}\Big[ \E_{\bm{y}\sim f}[\|\bm{y}\|^2] + n_0^{-\frac{1}{3}}(C_2(\sigma; f)+\E_{\bm{y}\sim f}[\|\bm{y}\|^2]+1) + d'\sigma^2\Big]\\
    &\hspace{-1cm}= \E[\|\bm{y}\|^2] +d'\sigma^2+n_0^{-\frac{1}{3}}C_3(\sigma; f),
    \end{aligned}
    \label{n_0_bound}
\end{equation}
where $C_2(\sigma; f)$ is a constant that depends on $f$ and $\sigma$, and $C_3(\sigma; f)$ is another constant depending on $f, \sqrt{f}$ and $\sigma$.

Consider the special coupling measure
\begin{equation}
\begin{aligned}
        &\pi(f, \tilde{f}_{\sigma^2, n_0})(\bm{y}, \hat{\bm{y}}) \coloneqq \min\big(f(\bm{y}), \tilde{f}_{\sigma^2, n_0}(\bm{y})\big) \mathbb{\delta}({\bm{y} - \hat{\bm{y}}}) \\
        &\hspace{2cm}+ \frac{1}{A}\big(f(\bm{y}) - \min (f, \tilde{f}_{\sigma^2, n_0})(\bm{y})\big)\cdot \big(\tilde{f}_{\sigma^2, n_0}(\hat{\bm{y}}) - \min (f, \tilde{f}_{\sigma^2, n_0})(\hat{\bm{y}})\big),\\
        &\hspace{3cm}\,\,\text{if} \int_{\mathbb{R}^{d'}} \min (f(\bm{y}), \tilde{f}_{\sigma^2, n_0}(\bm{y}))\d\bm{y}<1,\\
    &\pi(f, \tilde{f}_{\sigma^2, n_0})(\bm{y}, \hat{\bm{y}}) \coloneqq f(\bm{y})\delta({\bm{y} - \hat{\bm{y}}}), \,\, \text{if} \int_{\mathbb{R}^{d'}} \min (f(\bm{y}), \tilde{f}_{\sigma^2, n_0}(\bm{y}))\d\bm{y}=1,
\end{aligned}
\end{equation}
where $A \coloneqq \int_{\mathbb{R}^{d'}} \min\big(f(\bm{y}), \tilde{f}_{\sigma^2, n_0}(\bm{y})\big)\d\bm{y}$ and $\delta$ is the Dirac delta measure. The marginal probability densities of $\pi(f, \tilde{f}_{\sigma^2, n_0})$ are $f(\bm{y})$ and $\tilde{f}_{\sigma^2, n_0}(\bm{y})$, respectively. Furthermore, we have:
\begin{equation}
\begin{aligned}
        \E_{(\bm{y}, \hat{\bm{y}})\sim \pi(f, \tilde{f}_{\sigma^2, n_0})}\big[\|\bm{y} - \hat{\bm{y}}\|^2\big]&\leq 2 \int_{\mathbb{R}^{d'}} \|\bm{y}\|^2 (f(\bm{y}) - \min(f(\bm{y}), \tilde{f}_{\sigma^2, n_0}(\bm{y})))\d\bm{y} \\
        &\quad+ 2 \int_{\mathbb{R}^{d'}} \|\hat{\bm{y}}\|^2 (\tilde{f}_{\sigma^2, n_0}(\bm{y}) - \min(f(\bm{y}), \tilde{f}_{\sigma^2, n_0}(\bm{y})))\d\bm{y}\\
        &\leq 4\int_{\mathbb{R}^{d'}}\|\bm{y}\|^2|f(\bm{y}) - \tilde{f}_{\sigma^2, n_0}(\bm{y})|\d\bm{y}.
\end{aligned}
\end{equation}


Fixing $\sigma>0$, from Eq.~\eqref{convergence_n}, $\sigma$, $f_{\sigma^2, n_0}\rightarrow f_{\sigma^2}$ uniformly as $n_0\rightarrow \infty$; furthermore, from the definition of $\tilde{f}_{\sigma^2, n_0}$ in Eq.~\eqref{f_tilde_def}, $\tilde{f}_{\sigma^2, n_0}\rightarrow f_{\sigma^2, n_0}$ uniformly as $n_0\rightarrow \infty$.
Finally, since $\lim_{\sigma\rightarrow0}f_{\sigma^2}(\bm{y})=f(\bm{y})$ uniformly as $\sigma\rightarrow 0$ for any $\bm{y}\in\mathbb{R}^{d'}$,
we conclude that
\begin{equation}
    \tilde{f}_{\sigma^2, n_0(\sigma)}\rightarrow f
\end{equation}
uniformly as $\sigma\rightarrow 0$ and $n_0(\sigma)\rightarrow\infty$ in $\mathbb{R}^{d'}$.

Since $\E_{\bm{y}\sim f}\|\bm{y}\|^2]<\infty$ and $\E_{\bm{y}\sim \tilde{f}_{\sigma^2, n_0(\sigma)}}[\|\bm{y}\|^2]\rightarrow \E_{\bm{y}\sim f}[\|\bm{y}\|^2] + d'\sigma^2$ as $n_0(\sigma)\rightarrow\infty$ from Eq.~\eqref{n_0_bound}, for any $\epsilon>0$, there exists a measurable set $A\subseteq \mathbb{R}^{d'}$ such that: 
\begin{enumerate}
    \item $|\int_{A}\|\bm{y}\|^2f(\bm{y})\d\bm{y} - \E_{\bm{y}\sim f}[\|\bm{y}\|^2]|<\epsilon$
    \item we can find a sufficiently small $\sigma$ and a sufficiently large $n_0(\sigma)$ such that 
    \begin{equation}
        \int_A \|\bm{y}\|^2\cdot|f(\bm{y}) - \tilde{f}_{\sigma^2, n_0(\sigma)}(\bm{y})|\d\bm{y}\leq \epsilon.
    \end{equation}
\end{enumerate}



Using the inequality~\eqref{n_0_bound}, we have  
\begin{equation}
\begin{aligned}
        W_2^2(f, \tilde{f}_{\sigma^2, n_0(\sigma)})&\leq\E_{(\bm{y}, \hat{\bm{y}})\sim \pi(f, \tilde{f}_{\sigma^2, n_0(\sigma)})}[\|\bm{y} - \hat{\bm{y}}\|^2]
    \leq4\int_A \|\bm{y}\|^2|f(\bm{y}) - \tilde{f}_{\sigma^2, n_0(\sigma)}(\bm{y})|\d\bm{y} \\
    &\hspace{-2.25cm}+ 4\int_{\mathbb{R}^{d'}-A}\|\bm{y}\|^2f(\bm{y})\d\bm{y} + 4\int_{\mathbb{R}^{d'}-A}\|\bm{y}\|^2\tilde{f}_{\sigma^2, n_0(\sigma)}(\bm{y})\d\bm{y}\\
    &\hspace{-1.75cm}\leq 4\epsilon + 4\epsilon + 4 \Big(\E_{\bm{y}\sim f}[\|\bm{y}\|^2] + d'\sigma^2 + n_0(\sigma)^{-\frac{1}{3}}C_3(\sigma; f) -(\int_A\|\bm{y}\|^2f(\bm{y})\d\bm{y}-\epsilon)\Big)\\
    &\hspace{-1.25cm}\leq 16\epsilon + 4d'\sigma^2 + 4n_0(\sigma)^{-\frac{1}{3}}C_3(\sigma; f)\leq24\epsilon,
\end{aligned}
\end{equation}
if we take an $n_0(\sigma)$ large enough such that $
    n_0(\sigma)^{-\frac{1}{3}}C_3(\sigma; f)\leq\epsilon$. From Eq.~\eqref{f_tilde_def}, we have:
    \begin{equation}
        \tilde{f}_{\sigma^2, n_0(\sigma)}(\bm{y}) = \sum_{j=1}^{(n_0(\sigma)+1)^{d'}}\frac{f(\bm{y}_i)w_j}{\sum_{j=1}^{(n_0(\sigma)+1)^{d'}}f(\bm{y}_j)w_j}\cdot\mathcal{N}(\bm{y}-\bm{y}_i;\sigma^2I_{d'\times d'}),
    \end{equation}
    which is indeed the probability density function of a Gaussian mixture model, and this completes the proof of Theorem~\ref{thm3_2}.

\section{The approximation ability of the SNN model in Fig.~\ref{fig:nn_model}}
\label{appendix_universal}
In this subsection, we analyze the capability of the SNN model to approximate a family of probability density functions for a random variable $\bm{y}_{\bm{x}}\sim f_{\bm{x}}, \bm{y}\in\mathbb{R}^{d'}$ characterized by $\bm{x}\in D\subseteq \mathbb{R}^d, \bm{x}\sim\gamma(\cdot)$.
We assume the following conditions hold:
\begin{enumerate}
    \item $f_{\bm{x}}$ is uniformly continuous such that for any $\epsilon$, there exists $\Delta x$ satisfying $
    W_2^2(f_{\bm{x}}, f_{\tilde{\bm{x}}})\leq\epsilon, \,\,\forall \bm{x}, \tilde{\bm{x}}\in D, \|\bm{x}-\tilde{\bm{x}}\|\leq\Delta x$.
\item For any $\bm{x}$, $f_{\bm{x}}$ satisfies the conditions in Theorem~\ref{thm3_2}.
\end{enumerate}
Let $\epsilon>0$ be a small positive number. We can find a $\Delta x>0$ such that  $W_2^2(f_{\bm{x}}, f_{\tilde{\bm{x}}})<\frac{\epsilon}{4}$ if $\|\bm{x}-\tilde{\bm{x}}\|\leq \sqrt{d}\Delta x, \forall \bm{x}\, \tilde{\bm{x}}\in D$.
We can find an equidistance grid point set $X\coloneqq\{\bm{x}_i\}_{i=1}^K$ such that the distance between two adjacent points is $\Delta \bm{x}$, $D\subseteq \cup_{i=1}^K\otimes_{j=1}^d[x_i^j, x_i^j+\Delta \bm{x})$, and $\otimes_{j=1}^d[x_{i_1}^j, x_{i_1}^j+\Delta \bm{x})\cap \otimes_{j=1}^d[x_{i_2}^j, x_{i_2}^j+\Delta \bm{x})=\emptyset$ if $i_1\neq i_2$. 
Thus,
for any $\bm{x}\in D$, there exists $\bm{x}_i\in X$ satisfying $W_2^2(f_{\bm{x}}, f_{\bm{x}_i})<\frac{\epsilon}{4}$. Furthermore, for any $\bm{x}_i=(x_i^1,...,x_i^d)\in X$, from Theorem~\ref{thm3_2}, there exists a probability density function of a Gaussian mixture model:
\begin{equation}
    \tilde{f}_{n_{0, i}, \sigma_i^2}(\bm{y}_{\bm{x}_i}) = \sum_{r=1}^{n_{0, i}}p_{i, r}\mathcal{N}(\bm{y}-\bm{b}_{i, r}, A_{i, r}^TA_{i, r}), \,\,\sum_{r=1}^{n_{0, i}}p_{i, r}=1,
\end{equation} 
such that $W_2^2(f_{\bm{x}_i}, \tilde{f}_{n_{0, i}, \sigma_i^2})< \frac{\epsilon}{4}, i=1,...,K$. We denote $n_0\coloneqq \max_{1\leq i\leq K}n_{0, i}$ and 
\begin{equation}
\begin{aligned}
        &f_{n_{0}, \sigma_i^2}(\bm{y}_{\bm{x}_i}) = \sum_{r=1}^{n_{0, i}-1}p_{i, r}\mathcal{N}(\bm{y} - \bm{b}_{i, r}, A_{i, r}^TA_{i, r}) \\
        &\hspace{2cm}+ \sum_{r=n_{0, i}}^{n_0}\tfrac{p_{i, n_{0, i}}}{n_0-n_{0, i}+1}\mathcal{N}(\bm{y} - \bm{b}_{i, n_{0, i}}, A_{i, n_{0, i}}^TA_{i, n_{0, i}}).
\end{aligned}
\label{f_def1}
\end{equation}
$W_2^2(f_{\bm{x}_i}, f_{n_{0, i}, \sigma_i^2})< \frac{\epsilon}{4}, i=1,...,K$ for $f_{n_{0, i}, \sigma_i^2}$ defined in Eq.~\eqref{f_def1} because $W_2^2(f_{\bm{x}_i}, f_{n_{0, i}, \sigma_i^2}) = W_2^2(\tilde{f}_{\bm{x}_i}, f_{n_{0, i}, \sigma_i^2})$.

We define a new continuous random variable $\tilde{\bm{y}}_{\bm{x}}$ with a probability distribution $\tilde{f}_{\bm{x}}, \bm{x}\in D $  such that:
\begin{equation}
    \tilde{f}_{\bm{x}} = f_{n_{0}, \sigma^2_i}, \,\, \text{if}\,\,\bm{x}\in D\cap \otimes_{j=1}^d[x_i^j, x_i^j+\Delta \bm{x}).
\end{equation}

Therefore, we have
\begin{equation}
    \int_D W_2^2(\tilde{f}_{\bm{x}}, f_{\bm{x}})\gamma(\d\bm{x})< \sum_{i=1}^K\int_{D\cap \otimes[x_i^j, x_i^j+\Delta x]} 2 \big(W_2^2(\tilde{f}_{\bm{x}}, f_{\bm{x}_i}) +  W_2^2(f_{\bm{x}}, f_{\bm{x}_i})\big)\gamma(\d\bm{x})=\epsilon.
\end{equation}
We denote 
\begin{equation}
    Y(\epsilon, X)\coloneqq \sup_{i=1,...,K , s=1,...,n_{0}} (\|\bm{b}_{i, s}\|^2 + \|A^T_{i, s}A_{i, s}\|_F^2),
    \label{Y_epsilon}
\end{equation}
where $\bm{b}_{i, s}$ and $A^T_{i, s}A_{i, s}$ are the mean vectors and covariance matrices in Eq.~\eqref{f_def1}.
Similar to the proof of Corollary~\ref{col3_1}  in~\ref{AppendixF}, for any $\epsilon_1>0$, there exists 
\begin{equation}
    D(\epsilon_1) \coloneqq D\cap \big(\cup_{i=1}^K\otimes_{j=1}^d [x_{i}^{j}, x_{i}^{j}+\Delta x-\epsilon_1]\big)
\end{equation}
such that $\gamma(D - D(\epsilon_1))\coloneqq \int_{D - D(\epsilon_1)}1\gamma(\d\bm{x}) \leq\epsilon$. 
Additionally, similar to the estimates Eqs.~\eqref{6epsilon} and~\eqref{col3_result}, we can find an SNN whose output obeys a distribution $\hat{f}_{\bm{x}}$ satisfying:
\begin{equation}
\begin{aligned}
    &W_2^2(\hat{f}_{\bm{x}}, \tilde{f}_{\bm{x}_i})\leq 6\epsilon Y,\,\, \bm{x}\in D^i(\epsilon_1)\coloneqq \{\bm{x}\in D(\epsilon_1)|\|\bm{x}_i-\bm{x}\|\leq \|\bm{x}_j-\bm{x}\|\}\\
    &W_2^2(\hat{f}_{\bm{x}}, f_{\bm{x}_i}) \leq 2\max_i\E_{\bm{y}_{\bm{x}}\sim f_{n_{0}, \sigma_i^2}}[\|\bm{y}_{\bm{x}}\|^2] + 2 \E_{\bm{y}_{\bm{x}}\sim \hat{f}_{\bm{x}}}[\|\bm{y}_{\bm{x}}\|^2]\\
    &\hspace{2.2cm}\leq 10\max_i\E_{\bm{y}_{\bm{x}}\sim f_{n_{0}, \sigma_i^2}}[\|\bm{y}_{\bm{x}}\|^2]=10Y(\epsilon, X),\,\, x\notin D(\epsilon_1).
    \end{aligned}
\end{equation}
    

Therefore, we have
\begin{equation}
\begin{aligned}
    \int_D W_2^2(\hat{f}_{\bm{x}}, \tilde{f}_{\bm{x}})\gamma(\d\bm{x})&\leq \int_{D(\epsilon_1)}W_2^2(\hat{f}_{\bm{x}}, \tilde{f}_{\bm{x}})\gamma(\d\bm{x}) + \int_{D-D(\epsilon_1)}W_2^2(\hat{f}_{\bm{x}}, \tilde{f}_{\bm{x}}) \gamma(\d\bm{x})\\
    &=6\epsilon Y(\epsilon, X) + 10\epsilon Y(\epsilon, X)=16\epsilon Y(\epsilon, X),
    \end{aligned}
\end{equation}
and 
\begin{equation}
    \int_D W_2^2(\hat{f}_{\bm{x}}, f_{\bm{x}})\gamma(\d\bm{x})\leq 2\int_D \big(W_2^2(\hat{f}_{\bm{x}}, \tilde{f}_{\bm{x}})+ W_2^2(\tilde{f}_{\bm{x}}, f_{\bm{x}})\big)\gamma(\d\bm{x})
    \leq 32Y(\epsilon, X)\epsilon +2\epsilon.
    \label{Y_epsilon1}
\end{equation}
In Eq.~\eqref{Y_epsilon1}, $Y(\epsilon, X)$ also depends on $\epsilon$ and the choice of the grid point set $X=\{\bm{x}_i\}_{i=1}^K$. 
Specifically, as $\epsilon\rightarrow0$, if we can design a grid point set $\{\bm{x}\}_{i=1}^{K(\epsilon)}$ such that the quantity
$\epsilon \cdot Y(\epsilon, X)\rightarrow 0$ \added{(\textit{e.g.}, if $Y(\epsilon, X)$ is uniformly bounded and its upper bound is independent of the choice of the grid point set)}, then there exists an SNN model in Fig.~\ref{fig:nn_model} such that the distribution of its output can approximate $f_{\bm{x}}$, the probability distribution of $\bm{y}$ given the input $\bm{x}$, in the squared $W_2$ sense.

\section{Default training settings}
\label{appendix_training}
We list the hyperparameters and settings for training the SNN model in Fig.~\ref{fig:nn_model} of each example in Table~\ref{tab:setting} below.
\begin{table}[h!]
\scriptsize
\centering
\caption{Training settings for each example.} 
\begin{tabular}{lllll}
\toprule {Loss} & Example \ref{example1} & Example \ref{example2} &
Example \ref{example3} &
Example \ref{example4} \\
\midrule
Gradient descent method & AdamW & AdamW & AdamW & AdamW \\
Learning rate & 0.001 & 0.0005 & 0.0005 & 0.001\\
Weight decay & 0.01 & 0 & 0.02 & 0.02\\
No. of epochs & 500 & 400 & 2000 & 400\\
No. of training trajectories & 200 & 400 & 200 & 300 \\
Hidden layers & 1 & 3 & 3 & 2\\
Activation function & ReLU & ReLU & ReLU & ReLU\\
Neurons in each layer  & 40 & 50 &
10 & 50 \\
time step $\Delta t$ & 0.1 & 0.1 & 0.05 & 0.1 \\
Number of timesteps $N_T$ &81 & 21 & 41 & 21 \\
Initialization for biases &  2 &  0.01
&  $\mathcal{N}(0, 0.03^2)$ &  0.01 \\ 
Initialization for weights & $\mathcal{N}(0, 0.01^2)$ & $\mathcal{N}(0, 0.01^2)$ & $\mathcal{N}(0, 0.03^2)$ &  $\mathcal{N}(0, 0.01^2)$
  \\
%
%
\bottomrule
\end{tabular}
\label{tab:setting}
\end{table}

\section{Definitions of different loss metrics}
\label{appendix_loss}
Below, we provide descriptions and definitions for different loss functions used for comparison in this study. In the following, $N$ denotes the number of samples, and $t_i=i\frac{T}{n_T}, i=0,...,n_T$ denotes a uniform mesh in $[0, T]$.
\begin{compactenum}
\item A scaled local time-decoupled squared $W_2$ distance:
\begin{equation}
         \frac{1}{n_T}\sum_{i=1}^{n_T} W_{2, \delta}^{2, \text{e}}(\bm{X}(t_i), \hat{\bm{X}}(t_i))=\frac{1}{n_TN}\sum_{i=1}^{n_T}\sum_{j=1}^NW_2^2\big(\nu_{\bm{X}_{0, j}, \delta}^{\text{e}}(t_i),\hat{\nu}_{\bm{X}_{0, j}, \delta}^{\text{e}}(t_i)\big),
    \label{loss_define_empirical}
\end{equation}
where $W_{2, \delta}^{2, \text{e}}(\bm{X}(t_i), \hat{\bm{X}}(t_i))$ is the local squared $W_2$ distance defined in Eq.~\eqref{local_squared}. $\nu^{\text{e}}_{\bm{X}_{0, j}, \delta}(t)$ and $\hat{\nu}^{\text{e}}_{\bm{X}_{0, j}, \delta}(t)$ are the empirical conditional probability distributions of $\bm{X}$ and $\hat{\bm{X}}$ at time $t$ conditioned on $|\bm{X}(0)-\bm{X}_j(0)|\leq\delta$ and $|\hat{\bm{X}}(0)-\bm{X}_j(0)|\leq\delta$, respectively ($\bm{X}_j(0)$ denotes the initial state of the $j^{\text{th}}$ trajectory in the training set). $W_2^2(\nu_{\bm{X}_{0, j}, \delta}^{\text{e}},\hat{\nu}_{\bm{X}_{0, j}, \delta}^{\text{e}})$ is estimated by 
\begin{equation}
W_2^2(\nu_{\bm{X}_{0, j}, \delta}^{\text{e}},\hat{\nu}_{\bm{X}_{0, j}, \delta}^{\text{e}})\approx\texttt{ot.emd2}\Big(\frac{1}{N_j}\bm{I}_{N_j},
\frac{1}{N_j}\bm{I}_{N_j}, \bm{C}_j(t_i)\Big),
\label{time_coupling0}
\end{equation}
where $\texttt{ot.emd2}$ is the function for solving the earth movers
distance problem in the $\texttt{PoT}$ package of Python in \cite{flamary2021pot}. $N_j$ is the
number of elements in the set $X_j\coloneqq \Big\{\{\bm{X}_{i}\}_{t=0}^{T}\big|\|\bm{X}_i(0) - \bm{X}_j(0)\|\leq\delta, i=1,...,N\Big\}$, $\bm{I}_{N_j}$ is
an $N_j$-dimensional vector whose elements are all 1, and
$\bm{C}_j(t_i)\in\mathbb{R}^{N_j\times N_j}$ is a matrix with entries
$(\bm{C}_j(t_i))_{sr} = \|\bm{X}_s(t_i)-\hat{\bm{X}}_r(t_i)\|^2$. $\bm{X}_s(t_i)$ and $\hat{\bm{X}}_r(t_i)$ are the states of the $s^{\text{th}}$ ground truth trajectory at time $t_i$ in the set $X_j$ and the states of the $r^{\text{th}}$ predicted trajectory at time $t_i$ in the set $\hat{X}_j\coloneqq \Big\{\{\hat{\bm{X}}_{i}\}_{t=0}^{T}\big|\|\hat{\bm{X}}_i(0) - \bm{X}_j(0)\|\leq\delta, i=1,...,N\Big\}$, respectively.
\item A scaled time-decoupled squared $W_2$ distance
$$\tilde{W}_2^2(\bm{X}, \hat{\bm{X}})\approx \frac{1}{n_T}\sum_{i=1}^{n_T} W_2^2(\nu^{\text{e}}(t_i), \hat{\nu}^{\text{e}}(t_i)),$$ where
$\nu^{\text{e}}(t_i)$ and $\hat{\nu}^{\text{e}}(t_i)$ are the empirical
distributions of $\bm{X}(t_i)$ and
$\hat{\bm{X}}(t_i)$, respectively. It is
estimated by
\begin{equation}
W_2^2(\nu_N^{\text{e}}(t_i),
\hat{\nu}_N^{\text{e}}(t_i))\approx\texttt{ot.emd2}\Big(\frac{1}{N}\bm{I}_{N},
\frac{1}{N}\bm{I}_{N}, \bm{C}(t_i)\Big),
\label{time_coupling}
\end{equation}
where $\texttt{ot.emd2}$ is the function for solving the earth movers
distance problem in the $\texttt{ot}$ package of Python in \cite{flamary2021pot}. $N$ is the
number of ground truth and predicted samples, $\bm{I}_{N}$ is
an $N$-dimensional vector whose elements are all 1, and
$\bm{C}(t_i)\in\mathbb{R}^{N\times N}$ is a matrix with entries
$(\bm{C}(t_i))_{sj} = \|\bm{X}_s(t_i)-\hat{\bm{X}}_j(t_i)\|^2$. $\bm{X}_s(t_i)$ and $\hat{\bm{X}}_j(t_i)$ are the states of the $s^{\text{th}}$ ground truth trajectory at time $t_i$ and the states of the $j^{\text{th}}$ predicted trajectory at time $t_i$, respectively.
\item MMD (maximum mean discrepancy) 
  \cite{li2015generative}:
\begin{equation}
\begin{aligned}
&    \text{MMD}(\{\bm{X}\}, \{\hat{\bm{X}}\}) = \frac{1}{n_T}\sum_{i=1}^{n_T}\E[K(\{\bm{X}(t_i)\}, \{\bm{X}(t_i)\})]\\
&\quad\quad- 2\E[K(\{\bm{X}(t_i)\}, \{\hat{\bm{X}}(t_i)\})] + \E[K(\{\hat{\bm{X}}(t_i)\}, \{\hat{\bm{X}}(t_i)\})],
\end{aligned}
\label{MMD_def}
\end{equation}
where $K$ is the standard radial basis function (or Gaussian kernel)
with the multiplier equal to $2$ and the number of kernels equal to $5$. $\{\bm{X}(t_i)\}$ and $\{\hat{\bm{X}}(t_i)\}$
denote the sets of ground truth observations and predicted trajectories at time $t_i$, respectively.
\item Mean squared error (MSE):
\begin{equation}
    \operatorname{MSE}(\bm{X}, \hat{\bm{X}}) = \frac{1}{n_TN}\sum_{i=1}^{n_T}\sum_{s=1}^N
\|\bm{X}_s(t_i)-\hat{\bm{X}}_s(t_i)\|^2.
\end{equation}
$\bm{X}_s(t_i)$ and $\hat{\bm{X}}_s(t_i)$ are the states of the $s^{\text{th}}$ ground truth trajectory at time $t_i$ and the states of the $s^{\text{th}}$ predicted trajectory at time $t_i$, respectively.
\item Mean$^2$+Var loss function:
\begin{equation}
\begin{aligned}
    (\operatorname{Mean}^2+\operatorname{Var})(\bm{X}, \hat { \bm{X}}) =&
   \frac{1}{n_T}\sum_{i=1}^{n_T} \big(\frac{1}{N}\sum_{s=1}^N
\|\bm{X}_s(t_i)-\hat{\bm{X}}_s(t_i)\|^2 \\&\quad+ |\text{Var}(\bm{X}(t_i)) - \text{Var}(\hat{\bm{X}}(t_i))|\big)
\end{aligned}
    \end{equation}
where
\begin{equation}
    \text{Var}(\bm{X}(t_i))  = \sum_{s=1}^N \big\|\bm{X}_s(t_i) - \sum_{i=1}^N \frac{1}{N}\bm{X}_i(t_i)\big\|^2.
\end{equation}
$\bm{X}_s(t_i)$ and $\hat{\bm{X}}_s(t_i)$ are the states of the $s^{\text{th}}$ ground truth trajectory at time $t_i$ and the states of the $s^{\text{th}}$ predicted trajectory at time $t_i$, respectively.
\end{compactenum}


\section{Using a neural network to approximate the diffusion function in Example~\ref{example4}}
\label{example4_appendix}
Here, we apply a feedforward neural network with deterministic weights and biases to approximate the diffusion function in a jump-diffusion process when the form of the diffusion function is unknown. We consider the following jump-diffusion process:
\begin{equation}
\begin{aligned}
  &\d X_t = 0.05 \d t + \sigma_0\sqrt{|X_t|} \d B_t
  + \int_{U} \xi X_t\d \tilde{N}(\nu(\d\xi)\d{t}),\,\,\,  t\in[0, 2],\\
  &\quad\quad \xi\sim\mathcal{N}(\beta_0, \sigma_1^2),\,\, X_0\sim \mathcal{N}(2, \sigma_2^2).
  \end{aligned}
 \label{example4_model_appendix}
\end{equation}
In Eq.~\eqref{example4_model_appendix}, $\sigma_0$ is a positive constant, and $B_t$ and $\tilde{N}_t$ are an independent Wiener process and a compensated Poisson process, respectively. We use the following approximate jump-diffusion process to approximate Eq.~\eqref{example4_model_appendix}:
\begin{equation}
\begin{aligned}
  \d \hat{X}_t = 0.05 \d t + \hat{\sigma}(\hat{X}_t) \d \hat{B}_t
  + \int_{U} \hat{\xi} \hat{X}_t\d \hat{N}(\nu(\d\xi)\d{t}),\,\,\,  t\in[0, 2], \,\, \hat{X}_0=X_0.
  \end{aligned}
 \label{example4_model_appendix_approx}
\end{equation}
In Eq.~\eqref{example4_model_appendix_approx}, $\hat{\sigma}(\hat{X}_t)$ is a deterministic feedforward parameterized neural network that takes the state $\hat{X}_t$ as the input, and we aim to approximate the ground truth diffusion function $|\sigma(\bm{X}_t)|\coloneqq|\sigma_0|\sqrt{|X_t|}$ in Eq.~\eqref{example4_model_appendix} using the approximate $|\hat{\sigma}(X_t)|$ in Eq.~\eqref{example4_model_appendix_approx}. $\hat{\xi}$ is the output of an SNN when the input is 1, which aims at approximating the distribution of $\xi$, and $\hat{B}_t$ is another Wiener process independent of the Wiener process $B_t$ in Eq.~\eqref{example4_model_appendix} while $\hat{N}$ is another compensated Poisson process independent of $B_t, \hat{B}_t$, and $\tilde{N}_t$. We train both the deterministic neural network $\hat{\sigma}(\hat{X}_t)$ and the SNN that approximates the distribution of $\xi$ simultaneously by minimizing our local time-decoupled squared $W_2$ loss function Eq.~\eqref{local_define}. The parameterized neural network $\hat{\sigma}(\hat{X}_t)$ consists of two hidden layers with fifty neurons in each layer. All other hyperparameters are the same as those used in Example~\ref{example4}, described in Table~\ref{tab:setting}. We vary the variance of $\xi$ as well as the value of $\sigma_0$ when reconstructing the diffusion function and the distribution of $\xi$ in the jump function. 


\begin{figure}[h]
\centering
\includegraphics[width=\textwidth]{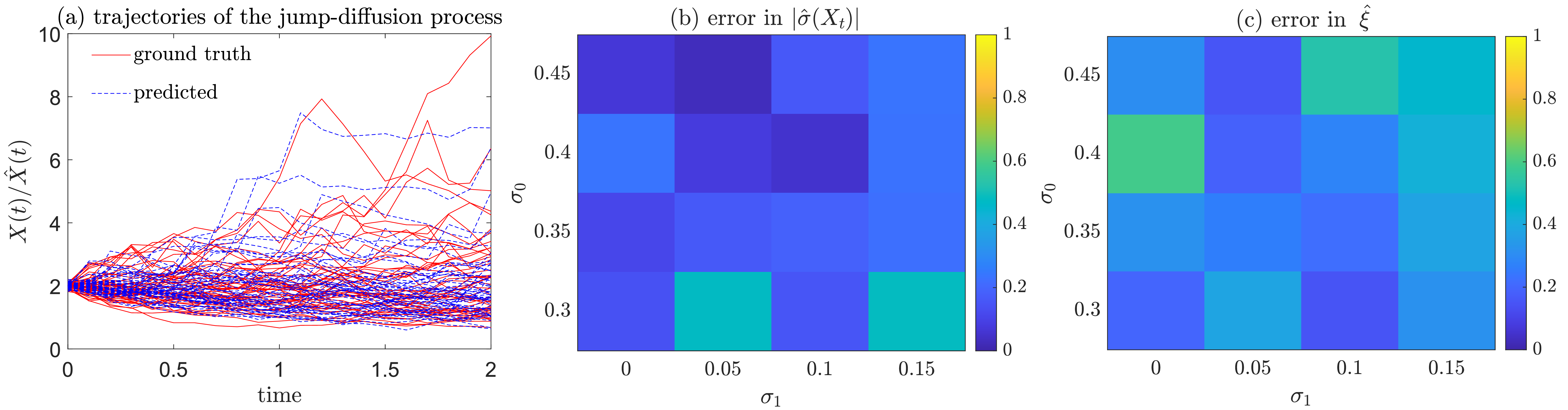}
\caption{\footnotesize(a) Ground truth trajectories versus the reconstructed trajectories when $\sigma_0=0.3,\beta_0=0.3, \sigma_1=0.15, \sigma_2=0.1$, and $\delta=0.1$ in the loss function Eq.~\eqref{local_define}. For clarity, 50 ground truth trajectories and 50 reconstructed trajectories are plotted. (b) The error in the reconstructed diffusion function $\hat{\sigma}$. Here, the error denotes the relative squared $L^2$ error: $\tfrac{\int_0^T (|\hat{\sigma}(X_t)|-\sigma_0\sqrt{|X_t|})^2\d t}{\int_0^T(\sigma_0 \sqrt{|X_t|})^2\d t}$. (c) The error in the reconstructed distribution of $\hat{\xi}$ (Eq.~\eqref{relative_error}). In (b) and (c), the values of other parameters are: $\beta_0=0.3, \sigma_2=0.1$, and the size of neighborhood $\delta=0.1$ in Eq.~\eqref{loss_define_empirical}. Errors are the average error over three repeated experiments.}
\label{fig:example4_appendix}
\end{figure}

\added{The trajectories generated by the reconstructed jump-diffusion process align well with ground truth trajectories (shown in Fig.~\ref{fig:example4_appendix} (a)). Furthermore, the error in the reconstructed diffusion function, as well as the error in the reconstructed distribution of $\hat{\xi}$, is small. Thus, when the form of the diffusion function is unknown while the diffusion function itself is deterministic, we could consider using a deterministic feedforward neural network to reconstruct it while using an SNN (Fig.~\ref{fig:nn_model}) to simultaneously reconstruct the distribution of $\xi$ in the jump magnitude function in Eq.~\eqref{example4_model_appendix} by minimizing the loss function Eq.~\eqref{local_define}.}


\bibliographystyle{plain}
\bibliography{bibliography}
\end{document}